%% file: main.tex
\pdfoutput=1
\documentclass[11pt,a4paper]{article}

\usepackage[margin=1in]{geometry}
\usepackage{latexsym}
\usepackage{amsmath,amssymb,amsthm,amsfonts}
\usepackage{graphicx}
\usepackage{microtype}
\usepackage{booktabs}
\usepackage{multirow}
\usepackage{array}
\usepackage{url}
\usepackage[authoryear]{natbib}
\usepackage{xcolor}
\usepackage{algorithm}
\usepackage{algpseudocode}
\usepackage{tabularx}
\usepackage{rotating}
\usepackage{capt-of}
\usepackage{makecell}
\usepackage{enumitem}
\usepackage{cancel}
\usepackage[colorlinks=true,linkcolor=blue,citecolor=blue,urlcolor=blue]{hyperref}
\usepackage[capitalise,noabbrev]{cleveref}

\hypersetup{
  pdftitle={Exact Stiefel Optimization for Probabilistic PLS},
  pdfauthor={Haoran Hu and Xingce Wang}
}

\makeatletter
\@ifundefined{theorem}{\newtheorem{theorem}{Theorem}}{}
\@ifundefined{lemma}{\newtheorem{lemma}[theorem]{Lemma}}{}
\@ifundefined{proposition}{\newtheorem{proposition}[theorem]{Proposition}}{}
\@ifundefined{corollary}{\newtheorem{corollary}[theorem]{Corollary}}{}
\@ifundefined{fact}{}{}
\@ifundefined{property}{}{}
\@ifundefined{remark}{\newtheorem{remark}[theorem]{Remark}}{}
\@ifundefined{assumption}{\newtheorem{assumption}[theorem]{Assumption}}{}
\@ifundefined{definition}{\newtheorem{definition}[theorem]{Definition}}{}
\makeatother

\newtheorem{mainassumption}{Assumption}
\newtheorem{maindefinition}{Definition}
\newtheorem{mainlemma}{Lemma}
\newtheorem{maintheorem}{Theorem}
\newtheorem{mainproposition}{Proposition}
\newtheorem{maincorollary}{Corollary}
\newtheorem{mainremark}{Remark}

\renewenvironment{assumption}[1][]{%
  \if\relax\detokenize{#1}\relax\begin{mainassumption}\else\begin{mainassumption}[#1]\fi
}{\end{mainassumption}}
\renewenvironment{definition}[1][]{%
  \if\relax\detokenize{#1}\relax\begin{maindefinition}\else\begin{maindefinition}[#1]\fi
}{\end{maindefinition}}
\renewenvironment{lemma}[1][]{%
  \if\relax\detokenize{#1}\relax\begin{mainlemma}\else\begin{mainlemma}[#1]\fi
}{\end{mainlemma}}
\renewenvironment{theorem}[1][]{%
  \if\relax\detokenize{#1}\relax\begin{maintheorem}\else\begin{maintheorem}[#1]\fi
}{\end{maintheorem}}
\renewenvironment{proposition}[1][]{%
  \if\relax\detokenize{#1}\relax\begin{mainproposition}\else\begin{mainproposition}[#1]\fi
}{\end{mainproposition}}
\renewenvironment{corollary}[1][]{%
  \if\relax\detokenize{#1}\relax\begin{maincorollary}\else\begin{maincorollary}[#1]\fi
}{\end{maincorollary}}
\renewenvironment{remark}[1][]{%
  \if\relax\detokenize{#1}\relax\begin{mainremark}\else\begin{mainremark}[#1]\fi
}{\end{mainremark}}


\title{Exact Stiefel Optimization for Probabilistic PLS:\\
Closed-Form Updates, Error Bounds, and Calibrated Uncertainty}

\author{Haoran Hu \\
School of Artificial Intelligence \\
Beijing Normal University \\
Beijing 100875, China
\and
Xingce Wang\thanks{Corresponding author. Email: \texttt{wangxingce@bnu.edu.cn}} \\
School of Artificial Intelligence \\
Beijing Normal University \\
Beijing 100875, China}

\date{\today}

\begin{document}

\maketitle

\begin{abstract}
Probabilistic partial least squares (PPLS) is a central likelihood-based model for two-view learning when one needs both interpretable latent factors and calibrated uncertainty. Building on the identifiable parameterization of Bouhaddani et al.\ (2018), existing fitting pipelines still face two practical bottlenecks: noise--signal coupling under joint EM/ECM updates and nontrivial handling of orthogonality constraints. Following the fixed-noise scalar-likelihood protocol, we develop an end-to-end framework that combines noise pre-estimation, constrained likelihood optimization, and prediction calibration in one pipeline. We estimate the observation noise from the low-eigenvalue noise subspace and enforce orthogonality through exact Stiefel-manifold optimization. The noise-subspace estimator attains a signal-strength-independent leading finite-sample rate and matches a minimax lower bound, whereas a full-spectrum noise estimator carries a deterministic bias under the same model. We further extend the framework to sub-Gaussian settings via optional Gaussianization and provide closed-form standard errors through a block-structured Fisher analysis. Across synthetic high-noise settings and two multi-omics benchmarks (TCGA-BRCA and PBMC CITE-seq), the method achieves near-nominal coverage without post-hoc recalibration, reaches Ridge-level point accuracy on TCGA-BRCA at rank $r=3$, matches or exceeds PO2PLS on cross-view prediction while providing native calibrated uncertainty, and improves stability of parameter recovery.

\end{abstract}


\input{sub/sec1_introduction}

\input{sub/sec2_related_work}

\input{sub/sec3_model}
\input{sub/sec4_objective_formulation}
\input{sub/sec5_optimization_algorithms}

\input{sub/sec6_experiments}
\input{sub/sec7_discussion}
\input{sub/sec8_conclusion}


\phantomsection
\label{sec:references}
\bibliographystyle{plainnat}
\bibliography{references}

\newpage
\appendix
\input{sub/appendix}

\end{document}

%% file: sub/sec1_introduction.tex
\section{Introduction}\label{sec:introduction}

Probabilistic Partial Least Squares (PPLS) provides a generative multi-view factorization that unifies cross-view prediction, latent-structure interpretation, and uncertainty quantification in a single likelihood-based framework. By modeling shared latent factors together with view-specific noise, PPLS is more informative than purely deterministic latent-variable pipelines when downstream tasks require both prediction and calibrated uncertainty.

We follow the identifiable PPLS parameterization of Bouhaddani et al.~(2018): two observed views $(x,y)$ are linked through shared latent factors, view-specific additive Gaussian noise, and orthonormal loading matrices $(W,C)$. Under this parameterization, the core estimation task is constrained by orthogonality and is naturally formulated on Stiefel manifolds. This geometric viewpoint is important because it separates model identifiability from numerical handling of constraints and makes explicit why exact-feasibility optimization is a principled target rather than a numerical detail.

Classical PPLS fitting is usually implemented by EM/ECM-style updates that optimize latent-signal blocks $(W,C,B,\Sigma_t)$ and noise blocks $(\sigma_e^2,\sigma_f^2,\sigma_h^2)$ jointly under orthogonality constraints \citep{DempsterLairdRubin1977,ECM}. This creates two coupled difficulties. First, in the complete-data updates, the effective curvature of signal blocks depends on current noise levels, while noise updates depend on current latent reconstruction; under high noise these mutual dependencies flatten the objective in some directions and steepen it in others, making optimization trajectories sensitive to initialization. Second, enforcing $W^\top W=C^\top C=I_r$ inside EM/ECM substeps is nontrivial and often handled indirectly. Empirically, this coupling effect is visible in our synthetic high-noise regime: Table~\ref{tab:parameter_mse} shows that EM/ECM errors on coupling-related parameters (notably $B$, and in some regimes $\Sigma_t$) deteriorate more than exact-feasibility solvers based on the same likelihood.

These observations motivate a solver that separates noise estimation from constrained signal optimization. Within the same identifiable PPLS family, we target two upgrades under a fixed-noise protocol: noise-subspace pre-estimation for $(\sigma_e^2,\sigma_f^2)$ and exact manifold optimization for the remaining constrained likelihood. This design reduces instability from noise--signal re-coupling while maintaining strict orthogonality feasibility throughout optimization. A concrete setting where these properties matter together is multi-omics integration, where one simultaneously needs interpretable cross-view factors, accurate prediction, and reliable uncertainty under heterogeneous noise; we use TCGA-BRCA and PBMC CITE-seq as representative benchmarks.

\textbf{The main contributions of the paper are:}
\begin{enumerate}[leftmargin=*]
\item[\textbf{1.}] \textbf{Rate-optimal noise estimation.} Under the fixed-noise protocol, we propose a noise-subspace spectral estimator whose leading finite-sample error rate is \emph{independent of signal strength} (Theorem~\ref{thm:noise_bound}) and \emph{minimax-optimal} over the PPLS family (Theorem~\ref{thm:minimax-lb}). We further show that a full-spectrum noise estimator is \emph{inconsistent}, with a deterministic bias $p^{-1}\sum_i\theta^2_{t,i}$ that does not vanish with $N$ (Proposition~\ref{prop:hu-bias}). The bound extends to sub-Gaussian noise with an explicit excess-kurtosis correction (Theorem~\ref{thm:noise_bound_subgaussian}).

\item[\textbf{2.}] \textbf{Exact-manifold optimization with inference guarantees.} Under the fixed-noise protocol, the remaining likelihood is optimized on the exact product manifold $\mathrm{St}(p,r)\times\mathrm{St}(q,r)\times\mathbb{R}^{2r+1}_{++}$. Two strategies share this objective: SLM-Manifold performs full Riemannian optimization, while BCD-SLM exploits PPLS scalar separability for closed-form component updates (Propositions~\ref{prop:bcd_theta_closed_form}, \ref{prop:bcd_b_cubic}). The resulting estimator is consistent and asymptotically normal (Corollaries~\ref{cor:consistency}, \ref{cor:slm_consistency}, Theorem~\ref{thm:asymptotic_normality}), with a block-diagonal Fisher information that yields closed-form standard errors (Proposition~\ref{prop:fisher-block}). We further prove a benign landscape on the PCCA submanifold (Proposition~\ref{prop:pcca-benign}).

\item[\textbf{3.}] \textbf{Extended modeling scope.} An optional coordinate-wise Gaussianization layer extends applicability to non-Gaussian observations while preserving the fixed-noise skeleton and the covariance-identified latent structure, interpreted as a Gaussian quasi-likelihood (White, 1982). The rank-based inverse normal transform reduces the excess-kurtosis bias in Theorem~\ref{thm:noise_bound_subgaussian} to $O(1/N)$ (Corollary~\ref{cor:int_kurtosis_reduction}).
\end{enumerate}

Empirically, this design delivers calibrated cross-view prediction without any post-hoc recalibration: across synthetic high-noise regimes and two multi-omics benchmarks, the method attains near-nominal coverage, reaches Ridge-level point accuracy on TCGA-BRCA at rank $r=3$, matches or exceeds PO2PLS on cross-view prediction while providing native calibrated uncertainty, and improves the stability of parameter recovery under high noise. We additionally benchmark against PO2PLS \citep{PO2PLS}, the closest joint-probabilistic competitor that augments PPLS with view-specific orthogonal latent variables; the comparison clarifies when explicit orthogonal decomposition helps and when joint factors alone suffice for cross-view prediction.

Our framework adopts the scalar observed-data likelihood expansion (Theorem~\ref{thm:theoremA}) and a fixed-noise protocol, and concentrates its methodological contributions in two places. First, the observation noise is estimated from the low-eigenvalue noise subspace rather than from the full spectrum, which yields consistency and a minimax-optimal leading rate. Second, orthogonality is enforced by exact manifold retractions, so feasibility is preserved at every iterate. The sub-Gaussian extension and the Gaussianization-based modeling scope are further components of the present framework.

%% file: sub/sec2_related_work.tex
\section{Related Work}\label{sec:related_work}

\textbf{PPLS, PCCA, and probabilistic multi-view models.}
Probabilistic CCA (PCCA) \citep{BachJordan2005} provides the canonical likelihood-based two-view latent-variable model, while the identifiable PPLS formulation of \citet{PPLS} specializes this probabilistic multi-view perspective to cross-view regression with orthogonal loadings and structured latent coupling. Subsequent work expanded the PPLS family in several directions, including probabilistic PLS regression variants \citep{PPLSEarly,PPLS2}, multi-omics extensions such as PO2PLS and GLM-PO2PLS \citep{PO2PLS,omics}, critical analyses of probabilistic PLS formulations \citep{PPLS1}, and deep probabilistic variants \citep{deepPPLS}. Among these extensions, PO2PLS \citep{PO2PLS} is the most directly comparable to our setting because it retains the joint probabilistic two-view structure while adding view-specific orthogonal latent components to absorb residual within-view variation. PO2PLS is oriented toward decomposition fidelity, namely separating joint signal from view-specific structure, whereas our fixed-noise pipeline is oriented toward cross-view prediction with calibrated uncertainty under the same identifiable joint parameterization. These goals are complementary rather than competing. The orthogonal components in PO2PLS capture view-specific residual structure that does not transfer to cross-view prediction, an effect we document on the two real-data benchmarks. GLM-PO2PLS broadens application scope to generalized-linear responses but still follows EM-style fitting rather than the fixed-noise plus manifold-optimization route studied here. Deep probabilistic PPLS variants replace linear loadings with neural modules, whereas our focus is the opposite regime: explicit orthogonality, closed-form component updates, and analytic uncertainty inference. Related probabilistic multi-view models such as Bayesian CCA \citep{Klami2013} also emphasize uncertainty-aware latent representations.

\textbf{Optimization under orthogonality constraints.}
Optimization on Stiefel/Grassmann manifolds has become a standard treatment for orthogonality-constrained problems, from the matrix-geometric foundation of \citet{Edelman1998} to modern algorithmic frameworks \citep{Absil2008,Boumal2023}. Feasible first-order schemes such as the Cayley-transform method of \citet{WenYin2013} show that orthogonality can be enforced exactly along the full optimization trajectory. This geometry is directly relevant to PPLS because the identifiable parameterization imposes $W^\top W=C^\top C=I_r$. Classical EM/ECM handles these constraints indirectly through alternating latent-variable updates, and scalar-likelihood implementations have used a penalty-based interior-point solver. Both approaches can be effective, but neither keeps exact manifold feasibility as the primary computational primitive. Our approach instead combines fixed-noise estimation with exact Stiefel retractions, so orthogonality is satisfied at every iterate and the scalar objective can be optimized with geometry-aware first-order updates plus closed-form component refinements. This is the main optimization-level distinction of our framework within the PPLS line.

\textbf{Statistical properties and noise estimation.}
Noise-level estimation from eigenvalue bulk statistics is classical in PPCA and spiked covariance models \citep{TippingBishop1999,MarchenkoPastur1967,BaikBenArousPeche2005,KritchmanNadler2009}. These works motivate estimating noise from the low-eigenvalue subspace rather than mixing signal and noise directions, and they also highlight the importance of finite-sample corrections when $p/N$ is non-negligible. Our fixed-noise protocol pre-estimates $(\sigma_e^2,\sigma_f^2)$ from the marginal sample covariances $\hat S_{xx}$ and $\hat S_{yy}$ before optimizing the remaining likelihood, using only the low-eigenvalue noise subspace. Relative to averaging the full spectrum, this choice yields a signal-independent leading term in Theorem~\ref{thm:noise_bound}, a minimax lower-bound match in Theorem~\ref{thm:minimax-lb}, and an explicit inconsistency result for full-spectrum averaging (Proposition~\ref{prop:hu-bias}), together with a sub-Gaussian extension in Theorem~\ref{thm:noise_bound_subgaussian}.

\textbf{Nonlinear and deep uncertainty baselines.}
Deep CCA (DCCA) \citep{Andrew2013} extends linear CCA by replacing linear projections with deep neural encoders to learn nonlinear shared representations. Kernel CCA (KCCA) \citep{Hardoon2004} extends CCA to nonlinear feature spaces through kernel mappings; for scalability we use the Nystr\"om approximation \citep{WilliamsSeeger2001}. For uncertainty-aware deep prediction, MC-Dropout \citep{GalGhahramani2016} interprets test-time dropout as approximate Bayesian inference, while Deep Ensembles \citep{Lakshminarayanan2017} aggregate independently trained models. Post-hoc recalibration methods such as temperature scaling \citep{Guo2017}, isotonic recalibration \citep{ZadroznyElkan2002}, and conformal prediction \citep{VovkGammermanShafer2005,ShaferVovk2008} are commonly used to improve deployment-level calibration. We include these nonlinear baselines to position the linear probabilistic PPLS pipeline against a modern nonlinear prediction and uncertainty landscape.

%% file: sub/sec3_model.tex
\section{Model}\label{sec:model}\label{sec:preliminaries}\label{sec:noise_estimation}

Figure~\ref{fig:method_pipeline_overview} summarizes the fixed-noise PPLS framework as a single estimation-to-prediction pipeline, and we use it to organize the present section. The pipeline has five stages. An optional \emph{Gaussianization} stage applies a coordinate-wise transform (for example a log or rank-based inverse normal transform) to each view when the raw observations are non-Gaussian, mapping them to better-behaved marginals while leaving the covariance-identified latent structure intact (Section~\ref{sec:gaussianization_extension}). A \emph{noise pre-estimation} stage estimates the observation-noise variances $(\sigma_e^2,\sigma_f^2)$ from the low-eigenvalue subspaces of the marginal sample covariances and then keeps them fixed (Section~\ref{sec:noise_subspace}). A \emph{scalar-likelihood construction} stage rewrites the observed-data likelihood in a scalar form that avoids explicit determinants and matrix inverses (Theorem~\ref{thm:theoremA}). A \emph{Stiefel optimization} stage minimizes the resulting fixed-noise objective on the product manifold while preserving orthonormal loadings at every iterate (Section~\ref{sec:optimization_algorithms}). Finally, a \emph{prediction and calibration} stage forms the conditional Gaussian predictive law and rescales its covariance by an empirical calibration factor (Section~\ref{sec:pipeline_pred_uq}). The remainder of this section develops the geometric, statistical, and modeling ingredients that these stages rely on.

\begin{figure*}[t]
\centering
\IfFileExists{artifacts/method_pipeline/pipeline_overview.png}{%
\includegraphics[width=0.97\textwidth]{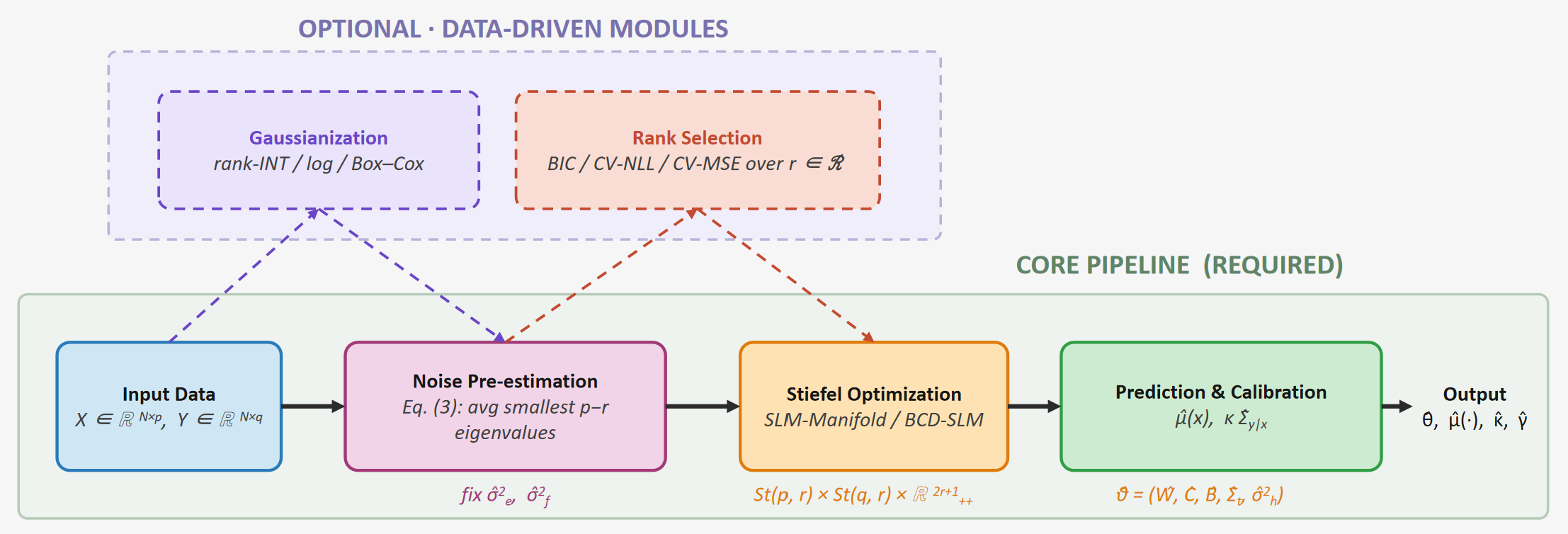}
}{%
\fbox{\parbox{0.95\textwidth}{\centering
Pipeline figure placeholder. Expected rendering output (generated by mermaid-py):\\
Data input $\to$ Gaussianization $\to$ spectral noise pre-estimation $\to$ scalar objective construction $\to$ manifold optimization/BCD-SLM $\to$ prediction and $\kappa$-calibration.
}}
}
\caption{Pipeline overview of the proposed fixed-noise PPLS framework.}
\label{fig:method_pipeline_overview}
\end{figure*}

\subsection{Geometric preliminaries}\label{sec:geom_prelim}
Because PPLS enforces orthonormal columns in $W\in\mathbb{R}^{p\times r}$ and $C\in\mathbb{R}^{q\times r}$, the optimization naturally lives on Stiefel manifolds
\[
\mathrm{St}(p,r):=\{A\in\mathbb{R}^{p\times r}:A^\top A=I_r\},\qquad
\mathrm{St}(q,r):=\{A\in\mathbb{R}^{q\times r}:A^\top A=I_r\}.
\]
We use the Euclidean metric induced from the ambient space, which defines a standard Riemannian structure on each Stiefel factor \citep{Absil2008,Boumal2023}. At $W\in\mathrm{St}(p,r)$, the tangent space is
\[
T_W\mathrm{St}(p,r)=\{\Xi\in\mathbb{R}^{p\times r}:W^\top\Xi+\Xi^\top W=0\},
\]
and the Riemannian gradient is obtained by orthogonal projection of the Euclidean gradient onto this tangent space. Retractions are implemented with thin-QR maps, i.e., $W_+=\mathrm{qf}(W+\alpha\Xi)$ with positive diagonal in $R$; the same construction is used for $C$. With the orthonormal-column constraints of PPLS, we transform the originally intractable orthogonality-constrained problem into a manifold optimization problem under the Riemannian geometry framework. This is the geometric foundation of the closed-form component updates and exact-feasibility iterations developed in the sequel, because it lets us treat the loadings $(W,C)$ as points on a smooth curved domain on which gradients, retractions, and stationarity are all well defined.

\subsection{Basic PPLS model and identifiability}\label{sec:ppls_model}
The PPLS model \citep{PPLS} relates two sets of observed variables $x \in \mathbb{R}^p$ and $y \in \mathbb{R}^q$ to latent variables $t \in \mathbb{R}^r$ and $u \in \mathbb{R}^r$:
\begin{equation}\label{eq:ppls-model}
x=tW^\top+e,\qquad y=uC^\top+f,\qquad u=tB+h,
\end{equation}
where the latent and noise variables satisfy
\[
t\sim\mathcal{N}(0,\Sigma_t),\quad h\sim\mathcal{N}(0,\sigma_h^2I_r),\quad
e\sim\mathcal{N}(0,\sigma_e^2I_p),\quad f\sim\mathcal{N}(0,\sigma_f^2I_q),
\]
with $t,h,e,f$ mutually independent, $\Sigma_t=\mathrm{diag}(\theta_{t_1}^2,\ldots,\theta_{t_r}^2)$, and $B=\mathrm{diag}(b_1,\ldots,b_r)$.

The joint distribution of $(x,y)$ is $\mathcal{N}(0,\Sigma)$ with block covariance
{\setlength{\arraycolsep}{3pt}\small
\begin{equation}\label{eq:joint-cov}
\Sigma=
\begin{pmatrix}
\Sigma_{xx} & \Sigma_{xy}\\
\Sigma_{yx} & \Sigma_{yy}
\end{pmatrix}
:=
\begin{pmatrix}
W\Sigma_tW^\top+\sigma_e^2I_p & W\Sigma_tBC^\top\\
CB\Sigma_tW^\top & C(B^2\Sigma_t+\sigma_h^2I_r)C^\top+\sigma_f^2I_q
\end{pmatrix}.
\end{equation}}
Given observations $\{(x_i,y_i)\}_{i=1}^N$, let $S=\frac{1}{N}\sum_{i=1}^N(x_i,y_i)^\top(x_i,y_i)$ denote the sample second-moment matrix for centered data. The observed-data negative log-likelihood (up to constants) is
\[
\mathcal{L}(\theta)=\ln(\det\, \Sigma)+\mathrm{tr}(S\Sigma^{-1}),
\qquad
\theta =(W,C,B,\Sigma_t,\sigma_e^2,\sigma_f^2,\sigma_h^2).
\]

\begin{assumption}[Identifiability \citep{PPLS}]\label{assump:identifiability}
The PPLS model~\eqref{eq:ppls-model} is identifiable under:
(C1) $b_k>0$ for all $k$;
(C2) $W^\top W=C^\top C=I_r$;
(C3) $\theta_{t_k}^2b_k$ is strictly decreasing in $k$;
(C4) $r<\min\{p,q\}$;
(C5) $\sigma_e^2,\sigma_f^2,\sigma_h^2>0$.
\end{assumption}

\begin{definition}[Sub-Gaussian random vector]\label{def:subgaussian_vector}
A zero-mean random vector $z\in\mathbb{R}^p$ is sub-Gaussian with parameter $\sigma_\psi$ if for all unit vectors $v\in\mathbb{S}^{p-1}$ (the unit sphere in $\mathbb{R}^p$) and all $t>0$,
\[
\mathbb{P}\bigl(|v^\top z|>t\bigr)\le 2\exp\!\left(-\frac{t^2}{2\sigma_\psi^2}\right).
\]
When $z\sim\mathcal{N}(0,\sigma^2 I_p)$, the sub-Gaussian parameter equals $\sigma$.
\end{definition}

These conditions fix sign, rotation, and ordering ambiguities and guarantee $\Sigma\succ0$. Positive-definiteness follows from a standard Schur-complement argument using $\theta_{t_i}^2>0$, $\sigma_e^2>0$, and $\sigma_f^2>0$. For the PCCA specialization we allow the boundary case $\sigma_h^2=0$ in Section~\ref{sec:special_cases}.

\subsection{Reduction to PCCA and the PPCA marginal}\label{sec:special_cases}
Setting $B = I_r$ and $\sigma_h^2 = 0$ in~\eqref{eq:ppls-model} reduces the inner equation to $u=t$ and yields the standard PCCA model \citep{BachJordan2005} $x=tW^\top+e$, $y=tC^\top+f$, with $t\sim\mathcal{N}(0,\Sigma_t)$. The joint covariance~\eqref{eq:joint-cov} simplifies to $\Sigma_{xx}=W\Sigma_tW^\top+\sigma_e^2I_p$, $\Sigma_{xy}=W\Sigma_tC^\top$, and $\Sigma_{yy}=C\Sigma_tC^\top+\sigma_f^2I_q$, and $\Sigma\succ0$ holds for $\sigma_e^2,\sigma_f^2>0$. In the scalar objective (Theorem~\ref{thm:theoremA}), each $D_i$ reduces to $\theta_{t i}^2(\sigma_f^2 + \sigma_e^2) + \sigma_f^2 \sigma_e^2$.

The $x$-marginal covariance $\Sigma_x=W\Sigma_tW^\top+\sigma_e^2I_p$ matches the PPCA model \citep{TippingBishop1999}. Accordingly, setting the gradient of the $x$-marginal log-likelihood with respect to $\theta_{t,i}^2$ to zero and solving recovers the PPCA updates of \citet[][Eq.~(9)]{TippingBishop1999}: $\hat\theta_{t i}^2=\lambda_i(\hat S_{xx})-\hat\sigma_e^2$, with $\hat\sigma_e^2$ given by Eq.~\eqref{eq:sigma_e_hat}. Our finite-sample noise-estimation bound therefore specializes directly to the PPCA setting.

\subsection{Extended model scope: Gaussianization}\label{sec:model_extensions}\label{sec:gaussianization_extension}
We extend the classical isotropic-Gaussian PPLS formulation through an optional preprocessing layer that preserves the covariance-based core of the model.

When $e$ and $f$ are non-Gaussian but still have well-behaved tails, we fit PPLS after optional coordinate-wise Gaussianization and interpret the likelihood as a Gaussian quasi-likelihood \citep{White1982}. This keeps the covariance-identified latent structure while relaxing the exact Gaussian assumption at the raw-data level.

We introduce coordinate-wise Gaussianization maps
\[
\psi=(\psi_1,\ldots,\psi_p),\qquad \phi=(\phi_1,\ldots,\phi_q),
\]
applied to each view independently:
\[
x=\psi(x_{\mathrm{raw}}),\qquad y=\phi(y_{\mathrm{raw}}).
\]
Common choices include:
\begin{enumerate}[label=(\alph*),leftmargin=*]
\item log transform: $\psi_j(z)=\log(1+z)$, standard for count data;
\item rank-based inverse normal transform (INT):
\[
\psi_j(z)=\Phi^{-1}\!\left(\frac{\operatorname{rank}(z)}{N+1}\right),
\]
which enforces exact marginal Gaussianity on the sample, where $\Phi^{-1}$ denotes the quantile function (inverse CDF) of the standard normal distribution $\mathcal{N}(0,1)$;
\item Box--Cox / Yeo--Johnson families: $\psi_j(\cdot;\lambda_j)$ with $\lambda_j$ estimated per coordinate to improve marginal normality.
\end{enumerate}
After transformation, the PPLS model is fitted to $(x,y)=(\psi(x_{\mathrm{raw}}),\phi(y_{\mathrm{raw}}))$ using Algorithm~\ref{alg:full_pipeline}. Predictive distributions are computed in transformed space and, when needed for interpretation, mapped back via approximate inverse maps $\psi^{-1}$ and $\phi^{-1}$.

\begin{remark}[Noise model relaxation]\label{rem:noise_model_relaxation}
The Gaussian assumption on $e$ and $f$ enters the methodology at two different levels. For exact likelihood-based modeling of the joint distribution of $(x,y)$, Gaussianity specifies the full parametric form of the observed-data objective. For noise pre-estimation, however, the key ingredients are the covariance structure, sub-Gaussian tail control, and finite fourth moments in the noise subspace. This distinction is what makes optional Gaussianization meaningful: it improves the adequacy of the likelihood model, while the spectral pre-estimation step remains meaningful under weaker tail assumptions.
\end{remark}

\paragraph{Effect on downstream inference.}
When Gaussianization is applied but the transformed noise remains non-Gaussian, the PPLS likelihood is misspecified. By White's quasi-likelihood theory \citep{White1982}, the estimator converges to the pseudo-true parameter minimizing Kullback--Leibler divergence from the true distribution. Under correct first- and second-moment specification, this pseudo-true parameter preserves the mean/covariance structure used by PPLS. Hence, under standard regularity,
\begin{itemize}[leftmargin=*]
\item consistency of loading matrices $(W,C)$ and signal parameters $(B,\Sigma_t)$ is preserved;
\item asymptotic normality holds with sandwich covariance $I(\theta_0)^{-1}J(\theta_0)I(\theta_0)^{-1}$ in place of $I(\theta_0)^{-1}$, where $J(\theta_0):=\mathrm{Var}_{\theta_0}[\nabla_\theta\log p(x,y;\theta_0)]$ is the variance of the quasi-score under the true distribution (cf.\ \citealt{White1982});
\item the calibration factor $\hat\kappa$ in Section~\ref{sec:pipeline_pred_uq} absorbs residual misspecification because it measures empirical dispersion without requiring exact Gaussianity.
\end{itemize}
Thus Gaussianization extends the model family without changing the fixed-noise optimization skeleton.

\subsection{Noise estimation strategy}\label{sec:noise_subspace}
We estimate the observation-noise variances $(\sigma_e^2,\sigma_f^2)$ from the low-eigenvalue subspaces of the marginal sample covariances and then keep them fixed while fitting the structured signal parameters. Because Eq.~\eqref{eq:sigma_e_hat} averages only the smallest $p-r$ eigenvalues, it depends on $r$ through the summation range. This introduces a chicken-and-egg issue when $r$ is unknown, which is the typical situation on real data. We therefore distinguish two operating modes: (i) \emph{oracle-$r$ mode}, where $r$ is given; and (ii) \emph{data-driven mode}, where $r$ is selected over a candidate grid using the two-stage procedure in Section~\ref{sec:rank_selection_fixed_noise}. The data-driven mode is the default whenever $r$ is not known a priori, and Proposition~\ref{prop:rank_overspec_noise} shows that a conservative upper bound on $r$ preserves consistency of the noise estimate.

We estimate $\sigma_e^2$ by averaging the smallest $p-r$ eigenvalues of $\hat S_{xx}$.
Throughout, we adopt the convention
\[
\lambda_1(M)\ge\lambda_2(M)\ge\cdots\ge\lambda_p(M)
\]
for eigenvalues of any $p\times p$ symmetric matrix $M$, so that the ``smallest $p-r$'' refers to indices $i=r+1,\dots,p$.
\begin{equation}\label{eq:sigma_e_hat}
\hat{\sigma}_e^2 = \frac{1}{p-r} \sum_{i=r+1}^{p} \lambda_i(\hat S_{xx}),
\end{equation}
and analogously estimate $\hat\sigma_f^2$ from $\hat S_{yy}$. Here
\[
\hat S_{xx} = \frac{1}{N}\sum_{i=1}^N (x_i-\bar{x})(x_i-\bar{x})^\top,
\qquad \bar{x}:=\frac{1}{N}\sum_{i=1}^N x_i,
\]
This is the PPCA noise-variance maximum likelihood estimator of \citet{TippingBishop1999}; see Appendix~\ref{app:exp_supp}, Table~\ref{tab:ppca_noise_verification} for a numerical check. The $y$-view sample mean $\bar{y}$ and sample covariance $\hat S_{yy}$ are defined analogously. By averaging only the $(p-r)$-dimensional noise subspace, Eq.~\eqref{eq:sigma_e_hat} targets the population noise level directly; averaging the full spectrum of $\hat S_{xx}$ instead would mix signal and noise directions.

\begin{proposition}[Inconsistency of the full-spectrum estimator]
\label{prop:hu-bias}
Let $\widehat{\sigma}^2_{\mathrm{full}} := p^{-1}\sum_{i=1}^{p}\lambda_i(\widehat{S}_{xx})$
denote the full-spectrum noise estimator. Under the PPLS
model with $W^{\top}W = I_r$, its expectation satisfies
\begin{equation}
\mathbb{E}\bigl[\widehat{\sigma}^2_{\mathrm{full}}\bigr] - \sigma^2_e
\;=\; \frac{1}{p}\sum_{i=1}^{r}\theta^2_{t,i},
\label{eq:hu-bias}
\end{equation}
a deterministic bias that does \emph{not} vanish as $N \to \infty$ with $(p,r)$
fixed. Consequently $\widehat{\sigma}^2_{\mathrm{full}}$ is inconsistent
whenever $\sum_i \theta^2_{t,i}$ is bounded away from zero, while the
noise-subspace estimator $\widehat{\sigma}^2_e$ in Eq.~\eqref{eq:sigma_e_hat} satisfies
$\widehat{\sigma}^2_e \xrightarrow{P} \sigma^2_e$ under the same asymptotics
(Corollary~\ref{cor:consistency}).
\end{proposition}

\begin{proof}
By linearity of trace,
$\mathbb{E}[\widehat{\sigma}^2_{\mathrm{full}}] = p^{-1}\mathrm{tr}(\Sigma_x)
= p^{-1}\bigl(\mathrm{tr}(W\Sigma_t W^{\top}) + p\sigma^2_e\bigr)
= \sigma^2_e + p^{-1}\sum_i\theta^2_{t,i}$,
using $\mathrm{tr}(W\Sigma_t W^{\top}) = \mathrm{tr}(\Sigma_t W^{\top}W)
= \mathrm{tr}(\Sigma_t)$.
\end{proof}

\paragraph{Strict asymptotic improvement over full-spectrum averaging.}
Under Theorem~\ref{thm:noise_bound}, for fixed $(p,r,\Sigma_t)$ with $\sum_i\theta^2_{t,i}>0$,
\[
\lim_{N\to\infty} \mathbb{E}\bigl[(\widehat{\sigma}^2_e - \sigma^2_e)^2\bigr] = 0,
\qquad
\liminf_{N\to\infty} \mathbb{E}\bigl[(\widehat{\sigma}^2_{\mathrm{full}} - \sigma^2_e)^2\bigr]
\ge \Bigl(\tfrac{1}{p}\textstyle\sum_i\theta^2_{t,i}\Bigr)^{2} > 0.
\]

Proposition~\ref{prop:hu-bias} formalizes the qualitative comparison stated in
Sections~\ref{sec:introduction} and~\ref{sec:related_work}: full-spectrum
averaging incorporates $r$ signal eigenvalues, each of
which inflates the sample mean by $\theta^2_{t,i}/p$. Because this bias is
deterministic and independent of sample size, no amount of additional data can
remove it. The noise-subspace estimator in Eq.~\eqref{eq:sigma_e_hat} avoids this
contamination by construction, which is the structural reason behind both the
signal-independent leading rate in Theorem~\ref{thm:noise_bound} and the
rate-optimality established by Theorem~\ref{thm:minimax-lb}. Appendix~\ref{app:hu-bias-mc}
reports a direct numerical confirmation of~\eqref{eq:hu-bias}.

Under the PPLS marginal covariance, the smallest $p-r$ eigenvalues of $\Sigma_x$ are all exactly $\sigma_e^2$, and likewise for the $y$-view, and our estimator averages the corresponding sample eigenvalues. For fixed $p$ and $N\to\infty$ this average is consistent (Theorem~\ref{thm:noise_bound}). When the aspect ratio $\gamma:=p/N$ (or $q/N$) is not negligible, however, the sample noise eigenvalues no longer concentrate at $\sigma_e^2$: they spread into a bulk whose support is described by the Marchenko--Pastur law with parameter $\gamma$ \citep{MarchenkoPastur1967}. The \emph{mean} of that bulk remains $\sigma_e^2$, so the simple average stays asymptotically unbiased for $\sigma_e^2$; what a finite $\gamma$ changes is the estimator's variance, together with a deterministic $O(\gamma)$ distortion that would appear if one instead used edge- or count-based spectral statistics. A Marchenko--Pastur (or related shrinkage) correction maps the observed bulk back through the Marchenko--Pastur transform and removes this finite-$\gamma$ distortion. In the regimes studied here $p/N$ is small enough that the correction changes $\hat\sigma_e^2$ negligibly --- the bias-floor study of Appendix~\ref{app:hu-bias-mc} and Figure~\ref{fig:noise_ablation_vary_p} confirms this directly --- so we report the uncorrected estimate and flag the Marchenko--Pastur correction as the recommended adjustment in the high-$p/N$ regime.

\subsection{Theoretical guarantees for spectral pre-estimation}\label{sec:noise_theory}
The main statistical point is that the leading error term of the noise-subspace estimator does not depend on signal strength. Prior PPLS bounds for full-spectrum estimators do not have this feature because they average signal and noise directions together. The theorem below makes the signal-independent leading term explicit and isolates the eigengap-dependent correction.

\begin{theorem}[Error bound for spectral noise variance estimator]\label{thm:noise_bound}
Consider the PPLS model $x=tW^\top+e$ with $t\sim\mathcal{N}(0,\Sigma_t)$, $e\sim\mathcal{N}(0,\sigma_e^2 I_p)$, $t\perp e$, and $W^\top W=I_r$.
Assume $N>p$ and $\min_i \theta_{t_i}^2>0$.
Let $\hat{\sigma}_e^2$ be defined in Eq.~\eqref{eq:sigma_e_hat}.
Then, with probability at least $1-\delta$,
\begin{equation}\label{eq:noise_error_bound}
\bigl|\hat{\sigma}_e^2-\sigma_e^2\bigr| \leq \sigma_e^2\sqrt{\frac{2\ln(4/\delta)}{N(p-r)}} + C\,\frac{\|\Sigma_x\|_{\mathrm{op}}^2\,p}{(p-r)\,N\,\min_i\theta_{t_i}^2},
\end{equation}
where $C>0$ is an absolute constant.
The bound in Eq.~\eqref{eq:noise_error_bound} is most informative when the correction term is dominated by the leading rate; see the discussion after Eq.~\eqref{eq:noise_error_bound_expect} for explicit conditions.
In addition, the estimator satisfies the unconditional mean-error bound
\begin{equation}\label{eq:noise_error_bound_expect}
\mathbb{E}\bigl[\,|\hat{\sigma}_e^2-\sigma_e^2|\,\bigr] \leq \frac{2\sigma_e^2}{\sqrt{N(p-r)}}.
\end{equation}
\end{theorem}

\begin{proof}
We prove the bound for the $x$-space estimator; the $y$-space case is identical.

\textbf{Proof roadmap.}
The bound~\eqref{eq:noise_error_bound} is obtained by combining two ingredients: (i) an averaging gain on the oracle noise subspace (Steps~1,~4), which yields the leading $O(1/\sqrt{N(p-r)})$ rate; and (ii) a subspace-rotation correction controlling the gap between the sample and population noise subspaces (Step~5), which yields the eigengap-dependent correction term. Steps~2--3 establish the operator-norm covariance concentration used inside Step~5; they are preparatory and do not directly appear in the final bound.

\textbf{Step 1 (Noise-subspace spectrum).} Under $\Sigma_x = W\Sigma_tW^\top + \sigma_e^2 I_p$, the orthogonal complement of $\mathrm{span}(W)$ has eigenvalue exactly $\sigma_e^2$ with multiplicity $p-r$.

\textbf{Step 2 (Weyl's inequality, preparatory).} By Weyl's inequality,
\[
\bigl|\lambda_i(\hat{S}_{xx})-\lambda_i(\Sigma_x)\bigr| \le \|\hat{S}_{xx}-\Sigma_x\|_{\mathrm{op}}.
\]
Averaging over $i=r+1,\dots,p$ gives
\[
\bigl|\hat{\sigma}_e^2-\sigma_e^2\bigr|\le \|\hat{S}_{xx}-\Sigma_x\|_{\mathrm{op}}.
\]
This crude bound is not used directly in the final combination (Step~6), but it provides the operator-norm control needed for the Davis--Kahan argument in Step~5.

\textbf{Step 3 (Operator-norm concentration, preparatory).} For Gaussian samples and $N>p$, standard covariance concentration gives, with probability at least $1-\delta$,
\[
\|\hat{S}_{xx}-\Sigma_x\|_{\mathrm{op}} \le C\|\Sigma_x\|_{\mathrm{op}}\!\left(\sqrt{\frac{\ln(4/\delta)}{N}}+\frac{p}{N}\right).
\]
This operator-norm bound is used inside the Davis--Kahan subspace perturbation analysis of Appendix~\ref{app:proof_noise_step5}, which feeds into Step~5.

\textbf{Step 4 (Averaging gain).} Let $U_\perp\in\mathbb{R}^{p\times(p-r)}$ span the population noise subspace. Since $U_\perp^\top W=0$, projected samples satisfy $z_i=U_\perp^\top x_i\sim\mathcal{N}(0,\sigma_e^2I_{p-r})$. The oracle projector statistic
\[
\tilde\sigma_e^2:=\frac{1}{p-r}\mathrm{tr}(U_\perp^\top \hat{S}_{xx}U_\perp)
=\frac{1}{N(p-r)}\sum_{i=1}^N\sum_{j=1}^{p-r} z_{ij}^2
\]
is the sample mean of $N(p-r)$ independent $\sigma_e^2\chi^2_1$ variables (each $z_{ij}^2/\sigma_e^2\sim\chi^2_1$). By a Bernstein-type inequality for sub-exponential random variables \citep[Theorem~2.8.1]{Vershynin2018}, applied to the centered variables $z_{ij}^2-\sigma_e^2$ with $\|z_{ij}^2-\sigma_e^2\|_{\psi_1}\le C\sigma_e^2$, we obtain with probability at least $1-\delta/2$:
\[
\bigl|\tilde\sigma_e^2-\sigma_e^2\bigr|\le\sigma_e^2\sqrt{\frac{2\ln(4/\delta)}{N(p-r)}}.
\]
Taking expectation gives Eq.~\eqref{eq:noise_error_bound_expect}:
$\mathbb{E}[|\tilde\sigma_e^2-\sigma_e^2|]\le 2\sigma_e^2/\sqrt{N(p-r)}$.

\textbf{Step 5 (Subspace-rotation correction).} The estimator projects onto the sample eigenspace $\hat U_\perp$ rather than the population eigenspace $U_\perp$. Appendix~\ref{app:proof_noise_step5} shows, via a variational decomposition, the von Neumann trace inequality, Davis--Kahan, and Wedin's energy-gap bound, that the subspace-rotation error satisfies
\[
\left|\hat\sigma_e^2-\frac{1}{p-r}\mathrm{tr}(U_\perp^\top\hat S_{xx}U_\perp)\right|
\le C\,\frac{\|\Sigma_x\|_{\mathrm{op}}^2\,p}{(p-r)\,N\,\min_i\theta_{t_i}^2}.
\]

\textbf{Step 6 (Triangle combination).} By the triangle inequality,
\[
\bigl|\hat\sigma_e^2-\sigma_e^2\bigr|
\le
\underbrace{\left|\hat\sigma_e^2-\frac{1}{p-r}\mathrm{tr}(U_\perp^\top\hat S_{xx}U_\perp)\right|}_{\text{subspace-rotation (Step~5)}}
+
\underbrace{\left|\frac{1}{p-r}\mathrm{tr}(U_\perp^\top\hat S_{xx}U_\perp)-\sigma_e^2\right|}_{\text{averaging gain (Step~4)}}.
\]
The first term is bounded by the subspace-rotation error from Step~5 (whose proof in Appendix~\ref{app:proof_noise_step5} relies on the operator-norm concentration from Steps~2--3), and the second by the averaging concentration from Step~4. Combining them yields Eq.~\eqref{eq:noise_error_bound}.
\end{proof}

The leading error rate $\sigma_e^2 \sqrt{2\ln(4/\delta) / (N(p-r))}$ is independent of signal strength. The correction term depends on the eigengap through $\min_i\theta_{t_i}^2$ and on the signal magnitude through $\|\Sigma_x\|_{\mathrm{op}}^2$. In regimes with $N\gg p$ and $\min_i\theta_{t_i}^2$ bounded away from zero, this correction is dominated by the leading $O(1/\sqrt{N(p-r)})$ term; Section~\ref{sec:exp_theorem1} (Figure~\ref{fig:noise_signal_theorem1}) provides a direct empirical check of this signal-independence pattern. Specifically, the correction is negligible when
\[
N\gg \frac{p\,\|\Sigma_x\|_{\mathrm{op}}^4}{\sigma_e^4\,(\min_i\theta_{t_i}^2)^2\,(p-r)},
\]
up to absolute constants. In the worst case, $\|\Sigma_x\|_{\mathrm{op}}^2\asymp(\sigma_e^2+\max_i\theta_{t_i}^2)^2$, which loosens the sufficient condition on $N$.

\begin{theorem}[Minimax lower bound for noise estimation]
\label{thm:minimax-lb}
Consider the PPLS $x$-view submodel with parameter of interest $\sigma^2_e$, where
$(W, \Sigma_t)$ is treated as a nuisance parameter satisfying
$W^{\top}W = I_r$ and $\min_i \theta^2_{t,i} > 0$. Let $\mathcal{P}_{\sigma^2_e}$
denote the distribution of $N$ i.i.d.\ samples $x_i \sim \mathcal{N}(0, \Sigma_x)$
with $\Sigma_x = W\Sigma_t W^{\top} + \sigma^2_e I_p$. Then any estimator
$T = T(x_1, \dots, x_N)$ satisfies
\begin{equation}
\inf_{T}\;\sup_{\sigma^2_e > 0}\;
\mathbb{E}_{\mathcal{P}_{\sigma^2_e}}\!\bigl[\,|T - \sigma^2_e|\,\bigr]
\;\ge\;
\frac{\sigma^2_e}{4\sqrt{2}}\,\frac{1}{\sqrt{N(p-r)}}.
\label{eq:minimax-lb}
\end{equation}
Consequently, the leading rate in Theorem~\ref{thm:noise_bound} is minimax-optimal
up to a constant factor (cf.\ \citealp{Tsybakov2009}).
\end{theorem}

\begin{proof}
We apply Le~Cam's two-point method. Fix $W$ and $\Sigma_t$, and consider the
two alternatives $\sigma^2_0 = \sigma^2$ and $\sigma^2_1 = \sigma^2 + \delta$
for $\delta > 0$ to be chosen. The two covariances $\Sigma_0, \Sigma_1$ share
eigenvectors; in the $r$-dimensional signal subspace the eigenvalues are
$\theta^2_{t,i} + \sigma^2_k$ ($k=0,1$), and in the $(p-r)$-dimensional noise
subspace they are $\sigma^2_k$.

\textbf{Step 1 (Per-sample KL).} Using the Gaussian KL identity and the shared
eigenbasis,
\begin{align*}
\mathrm{KL}\bigl(\mathcal{N}(0,\Sigma_1)\,\|\,\mathcal{N}(0,\Sigma_0)\bigr)
&= \tfrac{1}{2}\!\left[\mathrm{tr}(\Sigma_0^{-1}\Sigma_1) - p
    + \log\frac{\det\Sigma_0}{\det\Sigma_1}\right] \\
&= \tfrac{1}{2}\!\left[\sum_{i=1}^{r}
      \frac{\theta^2_{t,i}+\sigma^2+\delta}{\theta^2_{t,i}+\sigma^2}
    + (p-r)\frac{\sigma^2+\delta}{\sigma^2} - p\right. \\
&\qquad\left. - \sum_{i=1}^{r}\log\!\left(1+\tfrac{\delta}{\theta^2_{t,i}+\sigma^2}\right)
    - (p-r)\log\!\left(1+\tfrac{\delta}{\sigma^2}\right)\right].
\end{align*}
Expanding $\log(1+x) = x - x^2/2 + O(x^3)$, the first-order terms cancel and
\[
\mathrm{KL}\bigl(\mathcal{N}(0,\Sigma_1)\,\|\,\mathcal{N}(0,\Sigma_0)\bigr)
= \frac{\delta^2}{4}\!\left[\frac{p-r}{\sigma^4}
  + \sum_{i=1}^{r}\frac{1}{(\theta^2_{t,i}+\sigma^2)^2}\right] + O(\delta^3)
\ge \frac{\delta^2(p-r)}{4\sigma^4}.
\]

\textbf{Step 2 (Tensorization and Pinsker).} For $N$ i.i.d.\ samples,
$\mathrm{KL}(\mathcal{P}_{\sigma^2_1}^{\otimes N}\,\|\,\mathcal{P}_{\sigma^2_0}^{\otimes N})
= N\cdot\mathrm{KL}$. Choose
$\delta = \sigma^2\sqrt{2/(N(p-r))}$ so that the total KL is at most $\tfrac{1}{2}$
(to leading order). By Pinsker's inequality,
$\mathrm{TV}(\mathcal{P}_{\sigma^2_0}^{\otimes N}, \mathcal{P}_{\sigma^2_1}^{\otimes N})
\le \sqrt{\mathrm{KL}/2} \le 1/2$.

\textbf{Step 3 (Le~Cam).} For any estimator $T$,
\[
\sup_{k\in\{0,1\}} \mathbb{E}_{\mathcal{P}_{\sigma^2_k}^{\otimes N}}\!
\bigl[|T - \sigma^2_k|\bigr]
\;\ge\; \tfrac{\delta}{2}\bigl(1 - \mathrm{TV}\bigr)
\;\ge\; \frac{\delta}{4}
= \frac{\sigma^2}{4\sqrt{2}}\cdot\frac{1}{\sqrt{N(p-r)}}.
\]
This lower bound holds uniformly in $(W, \Sigma_t)$, hence taking the supremum
over the parameter space gives~\eqref{eq:minimax-lb}.
\end{proof}

The lower bound in \eqref{eq:minimax-lb} is independent of $\Sigma_t$, matching the signal-independence of Theorem~\ref{thm:noise_bound}; therefore the spectral noise-subspace estimator is simultaneously signal-independent and rate-optimal over the PPLS family.

\begin{corollary}[Consistency]\label{cor:consistency}
Under the assumptions of Theorem~\ref{thm:noise_bound}, $\hat{\sigma}_e^2 \xrightarrow{p} \sigma_e^2$ as $N\to\infty$.
The same conclusion holds for $\hat{\sigma}_f^2$.
\end{corollary}
\begin{proof}
Eq.~\eqref{eq:noise_error_bound} gives $|\hat{\sigma}_e^2-\sigma_e^2|=O_P(1/\sqrt{N(p-r)})\to0$ as $N\to\infty$ with $p$ fixed (or with $p-r\to\infty$ and $p/N\to0$). The argument for $\hat\sigma_f^2$ is identical.
\end{proof}

\paragraph{Symbol summary for the sub-Gaussian extension.}
In the next theorem, $\sigma_e^2$ denotes the marginal noise variance in $\operatorname{Cov}(e)=\sigma_e^2I_p$, whereas $\sigma_\psi$ denotes the sub-Gaussian tail parameter from Definition~\ref{def:subgaussian_vector}. We write $\kappa_4:=\mathbb{E}[e_j^4]/\sigma_e^4$ for the standardized fourth moment, so $\kappa_4=3$ in the Gaussian case.

\begin{lemma}[Variance-scale control of the sub-Gaussian parameter]\label{lem:k_sigma_relation}
Let $e\in\mathbb{R}^p$ be zero-mean sub-Gaussian with $\operatorname{Cov}(e)=\sigma_e^2I_p$. Then there exists a constant $K\ge 1$ such that $\sigma_\psi\le K\sigma_e$. In the Gaussian case, $K=1$; for bounded noise ($|e_j|\le M$ almost surely), $K$ depends on $M/\sigma_e$; and for heavy-tailed distributions that only marginally satisfy sub-Gaussian tails, $K$ can be large.
\end{lemma}

\begin{proof}
Write $K:=\sigma_\psi/\sigma_e$; we show $1\le K<\infty$, which is exactly the asserted bound $\sigma_\psi\le K\sigma_e$. \emph{Intuition.} The parameter $\sigma_\psi$ measures how heavy the tails of the one-dimensional projections $v^\top e$ are, whereas $\sigma_e$ measures only their spread; the second moment is always a lower bound on the tail scale, so $\sigma_\psi\ge\sigma_e$, and the substance of the lemma is that the tails are not \emph{arbitrarily} heavier than the spread, i.e.\ that $K$ is finite.

\emph{Step 1 (finiteness).} By hypothesis $e$ is sub-Gaussian in the sense of Definition~\ref{def:subgaussian_vector}, so $\sigma_\psi<\infty$ and every projection $\xi:=v^\top e$ with $\|v\|_2=1$ is a zero-mean scalar sub-Gaussian variable. Equivalently, its Orlicz norm $\|\xi\|_{\psi_2}:=\inf\{t>0:\mathbb{E}\exp(\xi^2/t^2)\le 2\}$ is finite and comparable to $\sigma_\psi$ up to an absolute constant \citep[Definition~2.5.6 and Proposition~2.5.2]{Vershynin2018}.

\emph{Step 2 (lower bound $\sigma_\psi\ge\sigma_e$).} Since $\operatorname{Cov}(e)=\sigma_e^2 I_p$, every unit direction carries the same variance $\mathbb{E}[(v^\top e)^2]=\sigma_e^2$. The sub-Gaussian assumption yields the moment-generating bound $\mathbb{E}\exp(\lambda\,v^\top e)\le\exp(\sigma_\psi^2\lambda^2/2)$; since $v^\top e$ has zero mean, the left side expands as $1+\tfrac12\lambda^2\,\mathbb{E}[(v^\top e)^2]+o(\lambda^2)$ while the right side expands as $1+\tfrac12\sigma_\psi^2\lambda^2+o(\lambda^2)$, so letting $\lambda\to0$ forces $\mathbb{E}[(v^\top e)^2]\le\sigma_\psi^2$. Hence $\sigma_e\le\sigma_\psi$ and $K\ge1$.

\emph{Step 3 (size of $K$ in examples).} Combining Steps~1--2, $1\le K<\infty$ and $\sigma_\psi\le K\sigma_e$ uniformly over $\|v\|_2=1$, which proves the claim. The value of $K$ is distribution-dependent. For $e\sim\mathcal{N}(0,\sigma_e^2 I_p)$ each projection is $\mathcal{N}(0,\sigma_e^2)$, whose sub-Gaussian parameter equals its standard deviation, so the bound is tight and $K=1$. For bounded noise $|e_j|\le M$ almost surely, Hoeffding's lemma gives a sub-Gaussian parameter of order $M$, so $K\lesssim M/\sigma_e$. For laws that satisfy the sub-Gaussian condition only marginally, $\sigma_\psi$ can greatly exceed $\sigma_e$ and $K$ is correspondingly large.
\end{proof}

\begin{theorem}[Spectral noise estimator under sub-Gaussian noise]\label{thm:noise_bound_subgaussian}
Suppose the PPLS model holds with the $x$-view noise vector $e$ replaced by a zero-mean sub-Gaussian random vector with covariance $\operatorname{Cov}(e)=\sigma_e^2I_p$, Let $K\ge1$ satisfy $\sigma_\psi\le K\sigma_e$ as in Lemma~\ref{lem:k_sigma_relation}. Assume also that each transformed noise coordinate satisfies $\mathbb{E}[e_j^4]\le \kappa_4\sigma_e^4$ for some kurtosis parameter $\kappa_4\ge3$. Then, for $\delta\in(0,1)$ and with probability at least $1-\delta$,
\begin{equation}\label{eq:noise_error_bound_subgaussian}
\bigl|\hat\sigma_e^2-\sigma_e^2\bigr|
\le
K^2\sigma_e^2\sqrt{\frac{2\ln(4/\delta)}{N(p-r)}}
+ C\,K^4\,\frac{\|\Sigma_x\|_{\mathrm{op}}^2\,p}{(p-r)\,N\,\min_i\theta_{t_i}^2}
+ \sigma_e^2\frac{\kappa_4-3}{p-r},
\end{equation}
where $C>0$ is an absolute constant.
In the Gaussian case, $K=1$ and $\kappa_4=3$, so Eq.~\eqref{eq:noise_error_bound_subgaussian} reduces to Theorem~\ref{thm:noise_bound} up to constants.
\end{theorem}

\begin{proof}[High-level idea (full proof in Appendix~\ref{app:proof_noise_subgaussian})]
We summarize the four steps that extend the Gaussian proof; the complete derivation is given in Appendix~\ref{app:proof_noise_subgaussian}.

\textbf{Step 1 (Sub-Gaussian covariance concentration).} The operator-norm concentration from Step~3 of Theorem~\ref{thm:noise_bound} extends to sub-Gaussian samples via \citet[Theorem~4.7.1]{Vershynin2018}, yielding $\|\hat S_{xx}-\Sigma_x\|_{\mathrm{op}}\le c_1 K^2\sigma_e^2(\sqrt{\ln(4/\delta)/N}+p/N)$.

\textbf{Step 2 (Fourth-moment correction).} Projected noise coordinates $z_{ij}$ satisfy $\operatorname{Var}(z_{ij}^2)=(\kappa_4-1)\sigma_e^4$, so the oracle averaging variance is $(\kappa_4-1)\sigma_e^4/(N(p-r))$. The non-Gaussian fourth cumulant induces a finite-sample bias of order $(\kappa_4-3)\sigma_e^2/(p-r)$.

\textbf{Step 3 (Subspace perturbation).} The subspace-rotation error depends on $\|\hat S_{xx}-\Sigma_x\|_{\mathrm{op}}^2$ (from Appendix~\ref{app:proof_noise_step5}). Since $\|\hat S_{xx}-\Sigma_x\|_{\mathrm{op}}\le c_1 K^2\sigma_e^2(\cdot)$ from Step~1, the quadratic dependence produces a $K^4$ factor.

\textbf{Step 4 (Combination).} The triangle inequality gives Eq.~\eqref{eq:noise_error_bound_subgaussian}.
\end{proof}

\begin{corollary}[Gaussianization reduces excess kurtosis]\label{cor:int_kurtosis_reduction}
If $\psi$ is the rank-based inverse normal transform applied to $N$ i.i.d. samples, then transformed marginals satisfy $\kappa_4=3+O(1/N)$ \citep{Beasley2009}. Consequently, the excess-kurtosis bias term in Theorem~\ref{thm:noise_bound_subgaussian} is
\[
O\!\left(\frac{1}{N(p-r)}\right),
\]
and the bound reduces to Theorem~\ref{thm:noise_bound} up to lower-order terms. In particular, Rank-INT not only enforces $\kappa_4=3+O(1/N)$ but also maps each marginal to standard normal so that $K=1$, which jointly improves all three terms in Eq.~\eqref{eq:noise_error_bound_subgaussian}.
\end{corollary}

%% file: sub/sec4_objective_formulation.tex
\section{Scalar Objective \& Problem Formulation}\label{sec:scalar_objective_formulation}\label{sec:scalar_expansion}\label{sec:complete_pipeline}

Having fixed the noise variances through spectral pre-estimation, we now turn the observed-data likelihood into a scalar objective that both optimization strategies will share. The key structural fact is that, under the orthogonality constraints, the dominant matrix operations collapse onto $r$ projected scalar channels, which both simplifies the analysis and exposes the componentwise separability exploited in Section~\ref{sec:optimization_algorithms}.

\subsection{Scalar-form objective derivation}\label{sec:scalar_likelihood_function}
The PPLS observed-data likelihood admits a scalar expansion that avoids explicit determinants and matrix inverses. Using a block-matrix identity (restated in Appendix~\ref{app:block_matrix_identity}), the likelihood takes a scalar form under the orthogonality constraints.

Intuitively, the orthogonality constraints on $W$ and $C$ reduce the dominant matrix operations to $r$ projected scalar channels. The same scalar objective is then used in two ways later on: SLM-Manifold treats it as a smooth manifold optimization problem, whereas BCD-SLM exploits its additional componentwise separability.

\begin{theorem}[Scalar expansion of the log-likelihood]\label{thm:theoremA}
Under the PPLS model defined in \eqref{eq:ppls-model} with identifiability constraints, the two terms of the objective function $\mathcal{L}(\theta) = \ln(\det\,\Sigma) + \mathrm{tr}(S\Sigma^{-1})$ can be expressed in terms of scalar variables as follows.  We write $S$ in block form as
\[
S = \begin{pmatrix} S_{xx} & S_{xy} \\ S_{yx} & S_{yy} \end{pmatrix},
\]
where $S_{xx}\in\mathbb{R}^{p\times p}$, $S_{yy}\in\mathbb{R}^{q\times q}$, and $S_{xy}=S_{yx}^\top\in\mathbb{R}^{p\times q}$ are the view-wise sample second-moment blocks of $S$.
\[
\ln(\det\,\Sigma)=(p-r)\ln\sigma_e^2+(q-r)\ln\sigma_f^2+\sum_{i=1}^r \ln D_i,
\]
and
\[
\mathrm{tr}(S\Sigma^{-1})=\tfrac{1}{\sigma_e^2}\mathrm{tr}(S_{xx})+\tfrac{1}{\sigma_f^2}\mathrm{tr}(S_{yy})-\sum_{i=1}^r\bigl(\Phi_x(i)\,Q_x(i)+\Phi_y(i)\,Q_y(i)+\Phi_{xy}(i)\,Q_{xy}(i)\bigr),
\]
where
\[
D_i=(\sigma_f^2+\sigma_h^2)(\theta_{t_i}^2+\sigma_e^2)+b_i^2\theta_{t_i}^2\sigma_e^2,
\]
\[
\Phi_x(i) := 1-\frac{\sigma_e^2(\sigma_f^2+\sigma_h^2+b_i^2\theta_{t_i}^2)}{D_i},\quad
\Phi_y(i) := 1-\frac{\sigma_f^2(\sigma_e^2+\theta_{t_i}^2)}{D_i},\quad
\Phi_{xy}(i) := \frac{\sigma_e\sigma_f b_i\theta_{t_i}^2}{D_i},
\]
and for $w_i,c_i$ the $i$th columns of $W,C$,
\[
Q_x(i) := \frac{1}{\sigma_e^2}\,w_i^\top S_{xx} w_i,\quad
Q_y(i) := \frac{1}{\sigma_f^2}\,c_i^\top S_{yy} c_i,\quad
Q_{xy}(i) := \frac{2}{\sigma_e\sigma_f}\,w_i^\top S_{xy} c_i.
\]
Denote the per-component scalar summand by
\[
\ell_i(\theta_{t_i}^2,b_i,\sigma_h^2):=\ln D_i - \bigl(\Phi_x(i)\,Q_x(i)+\Phi_y(i)\,Q_y(i)+\Phi_{xy}(i)\,Q_{xy}(i)\bigr),
\]
so that after pre-estimating $(\sigma_e^2,\sigma_f^2)$, the reduced objective (denoted by $\mathcal{L}_{\mathrm{red}}$, where ``red'' means fixed-noise reduction) satisfies
$\mathcal{L}_{\mathrm{red}} = (p-r)\ln\hat\sigma_e^2 + (q-r)\ln\hat\sigma_f^2
+ \hat\sigma_e^{-2}\,\mathrm{tr}(S_{xx}) + \hat\sigma_f^{-2}\,\mathrm{tr}(S_{yy})
+ \sum_{i=1}^r \ell_i$.

\end{theorem}

The derivation applies block-matrix determinant and inverse formulas restated in Appendix~\ref{app:block_matrix_identity}. Specifically, Lemma~\ref{lem:rank_n} is applied to each block of $\Sigma$: the $(1,1)$ block $\Sigma_{xx}=W\Sigma_tW^\top+\sigma_e^2I_p$ with $D_n=\Sigma_t$ and $k=\sigma_e^2$ yields $\det\Sigma_{xx}=\sigma_e^{2(p-r)}\prod_{i=1}^r D_i$ and $(\Sigma_{xx})^{-1}=\sigma_e^{-2}(I_p-W(I_r+\sigma_e^2\Sigma_t^{-1})^{-1}W^\top)$, from which the $\ln D_i$ and $\Phi_x(i)$ terms arise. Similarly for $\Sigma_{yy}$ and the Schur-complement structure of $\Sigma^{-1}$, which produces the $\Phi_y(i)$ and $\Phi_{xy}(i)$ coefficients via the Woodbury identity applied to the joint covariance. In particular, each evaluation of $\mathcal{L}(\theta)$ requires only the projected statistics $Q_x(i)$, $Q_y(i)$, $Q_{xy}(i)$ and scalar functions of the diagonal parameters, avoiding repeated full-matrix inversions and determinant evaluations.

\subsection{Optimization problem formulation}\label{sec:problem_formulation}
We denote the (compact) Stiefel manifold of $p\times r$ matrices with orthonormal columns by
\[
\mathrm{St}(p,r)=\{A\in\mathbb{R}^{p\times r}:A^\top A=I_r\}.
\]
Let $\mathbb{R}_{++}:=(0,\infty)$ denote the set of positive reals, and write $\mathbb{R}^k_{++}:=(\mathbb{R}_{++})^k$.

Under the fixed-noise protocol, the estimation problem is
\[
\min_{\theta\in\Omega}\;\mathcal{L}(\theta),
\qquad
\Omega=\Bigl\{(W,C,B,\Sigma_t,\sigma_h^2):W\in\mathrm{St}(p,r),\ C\in\mathrm{St}(q,r),\ b_k>0,\theta_{t_k}^2>0,\sigma_h^2>0\Bigr\},
\]
with $(\sigma_e^2,\sigma_f^2)$ fixed to the spectral pre-estimates from Eq.~\eqref{eq:sigma_e_hat}. Equivalently,
\[
\min_{(W,C,\theta_t^2,b,\sigma_h^2)\in\mathcal{M}_r}\mathcal{L}(W,C,\theta_t^2,b,\sigma_h^2),
\qquad
\mathcal{M}_r=\mathrm{St}(p,r)\times\mathrm{St}(q,r)\times\mathbb{R}_{++}^r\times\mathbb{R}_{++}^r\times\mathbb{R}_{++}.
\]
Section~\ref{sec:optimization_algorithms} studies two implementation strategies for this common objective: SLM-Manifold as a full Riemannian method and BCD-SLM as a PPLS-specific accelerated scheme.

\subsection{Prediction and uncertainty quantification}\label{sec:pipeline_pred_uq}
Once the parameters are fitted, the joint Gaussian model yields a closed-form conditional predictive law, which we then calibrate with a single dispersion factor estimated on held-out data. We describe these two steps in turn.

\subsubsection{Conditional distribution under PPLS}
Given a fitted parameter $\hat\theta$ from an outer training fold, the joint Gaussian model implies
\[
\mathbf{y}_{\mathrm{new}}\mid \mathbf{x}_{\mathrm{new}},\hat\theta \sim \mathcal{N}\bigl(\hat\mu(\mathbf{x}_{\mathrm{new}}),\hat\Sigma_{y\mid x}(\mathbf{x}_{\mathrm{new}})\bigr).
\]
Starting from the block covariance in Eq.~\eqref{eq:joint-cov} and applying the Gaussian-conditioning identity \citep{MultiGaussion}, we obtain
\begin{equation}\label{eq:pred_mean}
\hat\mu(\mathbf{x}_{\mathrm{new}}) = \mathbf{x}_{\mathrm{new}}(\hat W\hat\Sigma_t\hat W^\top+\hat\gamma\,\hat\sigma_e^2 I_p)^{-1}\hat W\hat\Sigma_t\hat B\hat C^\top,
\end{equation}
where $\mathbf{x}_{\mathrm{new}}$ denotes the model-input representation: raw covariates when no Gaussianization is used, and transformed covariates otherwise.

\begin{align}\label{eq:pred_cov_main}
\hat\Sigma_{y\mid x}(\mathbf{x}_{\mathrm{new}})
&= \hat C(\hat B^2\hat\Sigma_t+\hat\sigma_h^2 I_r)\hat C^\top+\hat\sigma_f^2 I_q \nonumber \\
&\quad - \hat C\hat B\hat\Sigma_t\hat W^\top(\hat W\hat\Sigma_t\hat W^\top+\hat\gamma\,\hat\sigma_e^2 I_p)^{-1}\hat W\hat\Sigma_t\hat B\hat C^\top.
\end{align}
Here $\hat W$, $\hat C$, $\hat B$, $\hat\Sigma_t$, $\hat\sigma_e^2$, $\hat\sigma_f^2$, and $\hat\sigma_h^2$ are the fitted PPLS parameters, $I_p$ and $I_q$ denote identity matrices of sizes $p$ and $q$, and $\hat\gamma\in\Gamma$ is the shrinkage multiplier selected by inner cross-validation for the $x$-view covariance term.

\subsubsection{Covariance calibration}
Let $n_{\mathrm{val}}$ denote the number of validation samples in the calibration fold and let $q$ denote the response dimension. On held-out validation points, let $r_i=y_i-\hat\mu(x_i)$ and $V_i=\hat\Sigma_{y\mid x}(x_i)$. We calibrate dispersion by

\[
\hat\kappa = \frac{1}{n_{\mathrm{val}}q}\sum_{i=1}^{n_{\mathrm{val}}} r_i^\top V_i^{-1}r_i.
\]
Under exact specification, $q^{-1}r_i^\top V_i^{-1} r_i$ has unit expectation, so $\hat\kappa$ measures finite-sample under- or over-dispersion of the plug-in covariance. We report intervals from $\mathcal{N}(\hat\mu(\mathbf{x}_{\mathrm{new}}),\hat\kappa\,\hat\Sigma_{y\mid x}(\mathbf{x}_{\mathrm{new}}))$.

%% file: sub/sec5_optimization_algorithms.tex
\section{Optimization Algorithms}\label{sec:optimization_algorithms}\label{sec:convergence}
We now turn the fixed-noise scalar likelihood into concrete optimization procedures. Both procedures minimize the same fixed-noise objective on the same product manifold; they differ only in how they traverse it. SLM-Manifold is a full Riemannian method that treats the objective as a generic smooth function on the manifold and updates all blocks by geometry-aware first-order steps. BCD-SLM instead exploits a property that is specific to the PPLS scalar form: once the loadings are fixed, the objective separates across latent components, so each component's parameters admit closed-form or one-dimensional updates. We develop both because they answer complementary needs. SLM-Manifold is the conceptually clean baseline and the natural template when the scalar blocks are later modified beyond closed-form solvability, whereas BCD-SLM is the faster and more predictable default in our experiments because its componentwise updates remove most per-iteration line-search uncertainty. The two solvers therefore trace the same statistical solution by different computational routes.

\subsection{Manifold optimization of the scalar objective}\label{sec:slm_algorithm}
We now turn this fixed-noise scalar likelihood into optimization algorithms on $\mathcal{M}_r$. Here the same fixed-noise objective is optimized on the exact product manifold
\[
\mathcal{M}_r=\mathrm{St}(p,r)\times\mathrm{St}(q,r)\times\mathbb{R}_{++}^r\times\mathbb{R}_{++}^r\times\mathbb{R}_{++},
\]
thereby preserving orthonormality at every iterate \citep{WenYin2013} and aligning optimization with the identifiable parameterization. This is the common objective for the two strategies studied in this paper: SLM-Manifold applies full Riemannian optimization on $\mathcal{M}_r$, whereas BCD-SLM uses the additional componentwise separability of the PPLS scalar form.

Closed-form Euclidean gradients for the loading matrices are
\begin{equation}\label{eq:euclidean_gradients}
\nabla_W \mathcal{L} = -2\bigl[S_{xx}\,W\,\hat{\Phi}_x + S_{xy}\,C\,\Phi_{xy}\bigr],
\qquad
\nabla_C \mathcal{L} = -2\bigl[S_{yy}\,C\,\hat{\Phi}_y + S_{yx}\,W\,\Phi_{xy}\bigr],
\end{equation}
where $\hat{\Phi}_x := \mathrm{diag}(\Phi_x(i)/\sigma_e^2)$ and $\hat{\Phi}_y := \mathrm{diag}(\Phi_y(i)/\sigma_f^2)$.

The Riemannian gradient on the Stiefel manifold is obtained by projecting the Euclidean gradient onto the tangent space at the current iterate \citep{Absil2008}. Using Eq.~\eqref{eq:euclidean_gradients}, the Riemannian gradients on the Stiefel blocks are
\begin{equation}\label{eq:riemannian_gradients}
\mathrm{grad}_W \mathcal{L}=G_W-W\,\mathrm{sym}(W^\top G_W), \qquad
\mathrm{grad}_C \mathcal{L}=G_C-C\,\mathrm{sym}(C^\top G_C),
\end{equation}
where $G_W:=\nabla_W\mathcal{L}$, $G_C:=\nabla_C\mathcal{L}$, and $\mathrm{sym}(A)=\tfrac12(A+A^\top)$. Positive scalars are optimized in log-coordinates.

\paragraph{Retraction and line search.}
Given Stiefel search directions $\xi_W,\xi_C$ and step size $\alpha$, we retract by
\begin{equation}\label{eq:qr_retraction}
W_{+}=\mathrm{qf}(W+\alpha\,\xi_W), \qquad C_{+}=\mathrm{qf}(C+\alpha\,\xi_C),
\end{equation}
where $\mathrm{qf}(M)=Q$ for the thin QR factorization $M=QR$ with $\mathrm{diag}(R)>0$. We use Armijo backtracking line search on this retraction map. Because the objective is nonconvex, both SLM-Manifold and BCD-SLM are paired with a shared multi-start safeguard.

Each evaluation costs $O(rp^2 + rq^2)$ via covariance projection, an $O(p/r)$ speedup over the $O((p+q)^3)$ matrix form.

\begin{algorithm}[t]\small
\caption{SLM-Manifold for fixed-noise scalar PPLS optimization}
\label{alg:slm_manifold}
\setcounter{ALG@line}{0}
\begin{algorithmic}[1]
\State \textbf{Input:}
\State Initial feasible point $(W^{(0)},C^{(0)},\Sigma_t^{(0)},B^{(0)},\sigma_h^{2,(0)})$
\State Fixed noise variances $\hat\sigma_e^2,\hat\sigma_f^2$, tolerance $\tau>0$, max iterations $k_{\max}$
\State \textbf{Output:} $\hat\theta=(\hat W,\hat C,\hat B,\hat\Sigma_t,\hat\sigma_h^2)\in\Omega$
\For{$k=0,1,\ldots,k_{\max}$}
  \State Evaluate projected statistics $Q_x,Q_y,Q_{xy}$ (as in Theorem~\ref{thm:theoremA})
  \State Compute Euclidean gradients and project to Riemannian gradients for $(W,C)$
  \State Update $(W,C)$ by Riemannian CG with Armijo line search and QR retraction
  \State Update positive scalar blocks $(\theta_t^2,b,\sigma_h^2)$ by smooth descent in log-coordinates
  \State Reorder components by $\theta_{t,\pi(1)}^2b_{\pi(1)}\ge\cdots\ge\theta_{t,\pi(r)}^2b_{\pi(r)}$
  \If{$\|\mathrm{grad}\,\mathcal{L}(\theta^{(k+1)})\|<\tau$}
    \State \Return $\hat\theta\leftarrow\theta^{(k+1)}$
  \EndIf
\EndFor
\State \Return $\hat\theta\leftarrow\theta^{(k_{\max}+1)}$
\end{algorithmic}
\end{algorithm}

\subsection{BCD-SLM algorithm}\label{sec:bcd_slm}

BCD-SLM is specific to the scalar PPLS objective rather than a generic block-coordinate template. Under fixed noise estimates $(\hat\sigma_e^2,\hat\sigma_f^2)$, Theorem~\ref{thm:theoremA} reduces the PPLS likelihood to
\[
\mathcal{L}_{\mathrm{red}}(W,C,\theta_t^2,b,\sigma_h^2)=\sum_{i=1}^r \ell_i(\theta_{t_i}^2,b_i,\sigma_h^2),
\]
where $\ell_i$ denotes the $i$th scalar summand from Theorem~\ref{thm:theoremA}, namely the contribution of the projected channel determined by $Q_x(i)$, $Q_y(i)$, and $Q_{xy}(i)$. After projecting $S_{xx},S_{yy},S_{xy}$ onto the current loading directions, each latent component enters through its own pair $(\theta_{t_i}^2,b_i)$ and can be updated independently once $(W,C,\sigma_h^2)$ are fixed. This componentwise separability is the structural reason Algorithm~\ref{alg:bcd_slm} is effective for PPLS.

The projected statistics are $Q_x(i)=[\sigma_e^{-2}W^\top S_{xx}W]_{ii}$, $Q_y(i)=[\sigma_f^{-2}C^\top S_{yy}C]_{ii}$, and $Q_{xy}(i)=[2(\sigma_e\sigma_f)^{-1}W^\top S_{xy}C]_{ii}$ (i.e., $Q_x,Q_y,Q_{xy}$ in Theorem~\ref{thm:theoremA}); these are the sufficient projected quantities induced by the scalar PPLS likelihood. Propositions~\ref{prop:bcd_theta_closed_form} and~\ref{prop:bcd_b_cubic} are direct consequences of this special algebraic form: the $\theta_{t_i}^2$ subproblem has a closed-form conditional maximizer, and the $b_i$ subproblem reduces to a scalar cubic with a unique positive root when $Q_{xy}(i)>0$. Thus BCD-SLM alternates exact-feasibility Riemannian updates for $(W,C)$ with componentwise Euclidean updates derived directly from the PPLS model.

\begin{proposition}[Closed-form conditional update for $\theta_{t_i}^2$]\label{prop:bcd_theta_closed_form}

Fix $W$, $C$, $b_i$, $\sigma_h^2$, and the noise variances. Define $d_i := (\sigma_f^2+\sigma_h^2)+b_i^2\sigma_e^2$ and $n_i := (\sigma_f^2+\sigma_h^2)Q_x(i)+(\sigma_h^2+b_i^2\sigma_e^2)Q_y(i)+b_iQ_{xy}(i)$. Then the unique stationary point is $\theta_{t_i}^{2\star} = \sigma_e^2\,\frac{\bigl(n_i-d_i\bigr)(\sigma_f^2+\sigma_h^2)-d_i\sigma_h^2Q_y(i)}{d_i^2}$, and the conditional update is obtained by clipping $\theta_{t_i}^{2\star}$ to the feasible region $\theta_{t_i}^2>0$.
\end{proposition}

\begin{proposition}[Cubic stationarity equation for $b_i$]\label{prop:bcd_b_cubic}
Fix $W$, $C$, $\theta_{t_i}^2$, $\sigma_h^2$, and the noise variances. The stationarity condition $\partial\ell_i/\partial b_i=0$ reduces to $c_3 b_i^3 + c_2 b_i^2 + c_1 b_i + c_0 = 0$, with coefficients $c_3 = 2\sigma_e^4\theta_{t_i}^2$, $c_2 = \sigma_e^2\theta_{t_i}^2 Q_{xy}(i)$, $c_1 = 2\sigma_e^2 R_i$, and $c_0 = -Q_{xy}(i)(\sigma_f^2+\sigma_h^2)(\theta_{t_i}^2+\sigma_e^2)$, where $R_i=(\sigma_f^2+\sigma_h^2)\bigl[(\theta_{t_i}^2+\sigma_e^2)(1-Q_y(i))+\theta_{t_i}^2Q_x(i)\bigr]+\sigma_h^2(\theta_{t_i}^2+\sigma_e^2)Q_y(i)$. When $Q_{xy}(i)>0$, the sign pattern of $(c_3,c_2,c_1,c_0)$ guarantees exactly one positive real root by Descartes' rule, which yields the conditional maximum likelihood estimate for $b_i$.
\end{proposition}

\noindent\textbf{$\sigma_h^2$ update.}
With $(W,C,\theta_t^2,b)$ fixed, updating $\sigma_h^2$ becomes a one-dimensional smooth optimization problem on $\mathbb{R}_{++}$. We solve this univariate problem by Brent's method to machine precision.

\begin{algorithm}[t]\small
\caption{BCD-SLM for the separable scalar PPLS objective}

\label{alg:bcd_slm}
\setcounter{ALG@line}{0}
\begin{algorithmic}[1]
\State \textbf{Input:}
\State $X\in\mathbb{R}^{N\times p}$
\State $Y\in\mathbb{R}^{N\times q}$
\State Initial point $(W^{(0)},C^{(0)},\Sigma_t^{(0)},B^{(0)},\sigma_h^{2,(0)})$
\State Fixed noise variances $\hat\sigma_e^2,\hat\sigma_f^2$
\State Tolerance $\tau>0$
\State Numerical floor $\varepsilon>0$ (default $\varepsilon=10^{-8}$), ensuring $\theta_{t,i}^2>0$
\State \textbf{Output:}
\State $\hat\theta=(\hat W,\hat C,\hat B,\hat\Sigma_t,\hat\sigma_h^2)\in\Omega$
\For{$k=0,1,\ldots,k_{\max}$}
  \State Update $W$ and $C$ on $\mathrm{St}(p,r)\times\mathrm{St}(q,r)$ by Riemannian Conjugate Gradient (CG) with Armijo line search
  \State Form projected statistics $Q_x,Q_y,Q_{xy}$ from $(W^{(k+1)},C^{(k+1)})$
  \For{$i=1,\ldots,r$} \Comment{independent latent-component updates from $\mathcal{L}_{\mathrm{red}}=\sum_i \ell_i$}
    \State $\theta_{t,i}^{2,(k+1)}\leftarrow \max\!\left\{\hat\sigma_e^2\,\dfrac{(n_i-d_i)(\hat\sigma_f^2+\sigma_h^{2,(k)})-d_i\sigma_h^{2,(k)}Q_y(i)}{d_i^2},\varepsilon\right\}$
    \State $b_i^{(k+1)}\leftarrow$ unique positive root of $c_3b^3+c_2b^2+c_1b+c_0=0$ \Comment{Proposition~\ref{prop:bcd_b_cubic}}
  \EndFor

  \State $\sigma_h^{2,(k+1)}\leftarrow \arg\min_{\sigma_h^2\in\mathbb{R}_{++}}\sum_{i=1}^{r}\ell_i(\theta_{t,i}^{2,(k+1)},b_i^{(k+1)},\sigma_h^2)$
  \State Reorder components by $\theta_{t,\pi(1)}^2b_{\pi(1)}\ge\cdots\ge\theta_{t,\pi(r)}^2b_{\pi(r)}$
  \If{$\|\mathrm{grad}\,\mathcal{L}(\theta^{(k+1)})\|<\tau$ \textbf{or} $\dfrac{|\mathcal{L}(\theta^{(k+1)})-\mathcal{L}(\theta^{(k)})|}{|\mathcal{L}(\theta^{(k)})|}<\tau$}
    \State \Return $\hat\theta\leftarrow\theta^{(k+1)}$
  \EndIf
\EndFor
\State \Return $\hat\theta\leftarrow\theta^{(k_{\max}+1)}$
\end{algorithmic}
\end{algorithm}

\subsection{Convergence analysis}\label{sec:opt_convergence}
The convergence argument is tied to the PPLS structure. Monotonicity comes from exact or machine-precision minimization of each block subproblem in Algorithm~\ref{alg:bcd_slm}, while precompactness of the iterate set comes from the compactness of the Stiefel factors and the coercive logarithmic growth of the scalar PPLS objective in the positive coordinates.

\begin{proposition}[Convergence of BCD-SLM]\label{prop:bcd_convergence}
Let $\{\theta^k\}$ be the iterates produced by Algorithm~\ref{alg:bcd_slm} under the fixed-noise protocol. Assume each block update is computed exactly, except for the one-dimensional $\sigma_h^2$ step which is solved to machine precision, and that every block update does not increase $\mathcal{L}$. Then $\mathcal{L}(\theta^k)$ decreases monotonically and converges. Moreover, $\{\theta^k\}$ remains in a precompact sublevel set of the PPLS objective, so accumulation points exist; every accumulation point in $\Omega$ is a first-order critical point of $\mathcal{L}$ on the product manifold.
\end{proposition}

\begin{proof}
Monotonicity follows block by block. The Riemannian CG step on $(W,C)$ uses Armijo backtracking and therefore does not increase $\mathcal{L}$. For fixed $(W,C,\sigma_h^2)$, the reduced PPLS objective separates as $\sum_i \ell_i$, so Propositions~\ref{prop:bcd_theta_closed_form} and~\ref{prop:bcd_b_cubic} solve the $\theta_{t_i}^2$ and $b_i$ subproblems componentwise. With $(W,C,\theta_t^2,b)$ fixed, the remaining $\sigma_h^2$ subproblem is one-dimensional and is solved to machine precision. Hence the objective values form a monotone nonincreasing sequence; since $\mathcal{L}$ is bounded below, $\mathcal{L}(\theta^k)$ converges.

Let $K_0:=\{\theta\in\overline{\Omega}:\mathcal{L}(\theta)\le \mathcal{L}(\theta^0)\}$. The Stiefel blocks are compact, and the scalar form in Theorem~\ref{thm:theoremA} shows that $\mathcal{L}(\theta)\to\infty$ whenever any positive coordinate in $(\theta_t^2,b,\sigma_h^2)$ diverges to $\infty$. Thus $K_0$ is bounded after closure, hence precompact, and every iterate stays in $K_0$.

Now let $\bar\theta\in\Omega$ be an accumulation point. We verify that all block gradients vanish at $\bar\theta$.

\emph{Stiefel blocks.} Because the Riemannian CG step on $(W,C)$ uses Armijo backtracking, it is a descent step: $\mathcal{L}$ does not increase along any admissible tangent direction. Since the iterates converge to $\bar\theta$ and the block update for $(W,C)$ is a descent method on the Stiefel manifold, the first-order necessary condition implies that the Riemannian gradient on each Stiefel block vanishes at $\bar\theta$; otherwise a further Armijo step would decrease $\mathcal{L}$, contradicting the fact that $\mathcal{L}(\theta^k)$ has converged. Formally, the smoothness of $\mathcal{L}$ restricted to $\mathrm{St}(p,r)\times\mathrm{St}(q,r)$ (each partial objective is $C^1$ in $(W,C)$ with other blocks fixed) and the exact block minimization guarantee $\mathrm{grad}_W\mathcal{L}(\bar\theta)=0$ and $\mathrm{grad}_C\mathcal{L}(\bar\theta)=0$.

\emph{Positive-coordinate blocks.} For each $(\theta_{t_i}^2, b_i)$, the componentwise updates from Propositions~\ref{prop:bcd_theta_closed_form} and~\ref{prop:bcd_b_cubic} solve the stationarity condition $\partial\ell_i/\partial\theta_{t_i}^2=0$ and $\partial\ell_i/\partial b_i=0$ exactly (or clip to the boundary $\theta_{t_i}^2=\varepsilon$). At $\bar\theta$, if $\bar\theta_{t_i}^2 > \varepsilon$, then $\partial\ell_i/\partial\theta_{t_i}^2=0$; if $\bar\theta_{t_i}^2 = \varepsilon$ (boundary), the derivative may be non-zero but the point is still a constrained minimizer. By Corollary~\ref{cor:interior}, for sufficiently large $N$ the global minimizer lies in $\mathrm{int}(\Omega)$ with high probability, so at any accumulation point corresponding to a global minimizer, all positive coordinates are bounded away from the boundary and the Euclidean partial derivatives vanish. For the $\sigma_h^2$ block, the subproblem is solved to machine precision, so $\partial\mathcal{L}/\partial\sigma_h^2$ is numerically zero at $\bar\theta$.

\emph{Conclusion.} Combining the three blocks: the Riemannian gradients on the Stiefel factors and the Euclidean partial derivatives on the positive-coordinate factors all vanish at $\bar\theta$. Hence $\bar\theta$ is a first-order critical point of $\mathcal{L}$ on the product manifold.
\end{proof}

\begin{theorem}[Existence under the fixed-noise protocol]\label{thm:slm_guarantees}
Let $S \in \mathbb{R}^{(p+q) \times (p+q)}$ be positive definite and define $\mathcal{L}(\theta)=\ln(\det\,\Sigma)+\mathrm{tr}(S\Sigma^{-1})$ on
\[
\Omega = \{\theta=(W,C,B,\Sigma_t,\sigma_h^2): W^\top W=I_r,\ C^\top C=I_r,\ b_k>0,\ \theta_{t_k}^2>0,\ \sigma_h^2>0\}.
\]
Then $\mathcal{L}$ attains its infimum on $\overline{\Omega}$, and for large $N$ the global minimizer lies in $\mathrm{int}(\Omega)$ with probability approaching $1$ (Corollary~\ref{cor:interior}).
\end{theorem}

\begin{corollary}[Interior minimizer for large $N$]\label{cor:interior}
Under the assumptions of Theorem~\ref{thm:slm_guarantees}, suppose additionally that the data are generated under the model with true parameter $\theta_0\in\mathrm{int}(\Omega)$, so that $S\xrightarrow{P}\Sigma(\theta_0)$ as $N\to\infty$. Then for sufficiently large $N$, every global minimizer over $\overline{\Omega}$ lies in $\mathrm{int}(\Omega)$ with probability approaching~$1$.
\end{corollary}

\begin{corollary}[Optimization convergence on the exact manifold]\label{cor:slm_opt_convergence}
On $\mathcal{M}:=\mathrm{St}(p,r)\times\mathrm{St}(q,r)\times(\mathbb{R}_{++})^{2r+1}$, any standard retraction-based Riemannian descent method \citep{Absil2008,Boumal2023} applied to $\min_{\theta\in\mathcal{M}}\mathcal{L}(\theta)$ satisfies $\|\operatorname{grad}\,\mathcal{L}(\theta_k)\|\to 0$, and every accumulation point is first-order critical.
\end{corollary}

\begin{proof}[Proof of Theorem~\ref{thm:slm_guarantees}]
We verify the three conditions that guarantee existence of a minimizer: continuity of $\mathcal{L}$ on $\overline\Omega$, coercivity, and compactness of sublevel sets.

\textbf{Step 1 (Continuity).}
On $\Omega$, the map $\theta\mapsto\Sigma(\theta)$ is smooth because each block of Eq.~\eqref{eq:joint-cov} is a polynomial (for $\Sigma_{xx}$, $\Sigma_{yy}$) or bilinear (for $\Sigma_{xy}$) function of the parameters with smooth dependence on $W,C$ through matrix multiplication. For any $\theta\in\Omega$, positive-definiteness $\Sigma(\theta)\succ 0$ holds: the Schur complement of the $(2,2)$ block in $\Sigma$ is $\Sigma_{xx}-\Sigma_{xy}\Sigma_{yy}^{-1}\Sigma_{yx}$, which equals $\sigma_e^2 I_p + W(\Sigma_t - \Sigma_t B C^\top \Sigma_{yy}^{-1} C B \Sigma_t)W^\top \succeq \sigma_e^2 I_p \succ 0$ under the interior constraints $b_k>0$, $\theta_{t_k}^2>0$, $\sigma_h^2>0$. Since $\Sigma(\theta)^{-1}$ exists and is smooth on $\Omega$, both $\ln\det\Sigma(\theta)$ and $\mathrm{tr}(S\Sigma(\theta)^{-1})$ are continuous (in fact $C^\infty$) on $\Omega$, hence $\mathcal{L}$ is continuous on $\Omega$.

\textbf{Step 2 (Coercivity).}
We show that $\mathcal{L}(\theta)\to\infty$ whenever any positive coordinate in $(\theta_{t_i}^2, b_i, \sigma_h^2)$ diverges to $\infty$ or approaches the boundary $0^+$, while $W\in\mathrm{St}(p,r)$ and $C\in\mathrm{St}(q,r)$ remain in the compact Stiefel manifolds.

\emph{(a) $\theta_{t_i}^2\to\infty$:} From the scalar form in Theorem~\ref{thm:theoremA}, $D_i=(\sigma_f^2+\sigma_h^2)(\theta_{t_i}^2+\sigma_e^2)+b_i^2\theta_{t_i}^2\sigma_e^2\to\infty$, so $\ln D_i\to\infty$. The $-\Phi_x(i)Q_x(i)$ term is bounded above (since $0<\Phi_x(i)<1$ and $Q_x(i)$ is bounded for fixed $S$), so $\ell_i\to\infty$.

\emph{(b) $\theta_{t_i}^2\to 0^+$:} In this limit, $D_i\to(\sigma_f^2+\sigma_h^2)\sigma_e^2$ remains bounded away from zero, so $\ln D_i$ remains bounded. However, the $\ln\det\Sigma$ term contains the contribution $\ln D_i$ while the $\mathrm{tr}(S\Sigma^{-1})$ term satisfies $\mathrm{tr}(S\Sigma^{-1})\ge (p+q)$ (by the AM-HM inequality on eigenvalues, or equivalently $\mathrm{tr}(S\Sigma^{-1})\ge \mathrm{tr}(S)/\|\Sigma\|_{\mathrm{op}}$). As $\theta_{t_i}^2\to 0^+$, the fitted covariance $\Sigma(\theta)$ deviates from $S$ in the signal subspace, and $\mathcal{L}(\theta)\ge (p+q)+\ln\det\Sigma(\theta)$ (by the KL lower bound, see Appendix~\ref{proof:thm:slm_consistency} Step~2). Since $\Sigma(\theta)$ still satisfies $\Sigma(\theta)\succ 0$ (the minimum eigenvalue is at least $\sigma_e^2\wedge\sigma_f^2>0$), $\ln\det\Sigma(\theta)$ remains bounded below. Overall, $\mathcal{L}$ remains bounded on any sequence with $\theta_{t_i}^2\to 0^+$, so this boundary alone does not force $\mathcal{L}\to\infty$; the key point is that such boundary points remain in $\overline\Omega$ and continuity extends there.

\emph{(c) $b_i\to\infty$:} The term $D_i\ge b_i^2\theta_{t_i}^2\sigma_e^2\to\infty$ forces $\ln D_i\to\infty$, and the $Q$-terms grow at most quadratically in $b_i$ (through $Q_y(i)$ via $C^\top S_{yy} C$), so $\ell_i\to\infty$.

\emph{(d) $\sigma_h^2\to 0^+$ or $\sigma_h^2\to\infty$:} As $\sigma_h^2\to\infty$, each $D_i$ is linear in $\sigma_h^2$ with positive coefficient, so $\ln D_i\to\infty$. As $\sigma_h^2\to 0^+$, $D_i\to(\sigma_f^2)(\theta_{t_i}^2+\sigma_e^2)+b_i^2\theta_{t_i}^2\sigma_e^2>0$ remains bounded away from zero, so $\mathcal{L}$ remains bounded on this boundary.

Combining (a)--(d): in any unbounded sequence from $\Omega$, either a positive coordinate diverges (forcing $\mathcal{L}\to\infty$ via (a) or (c) or (d)) or the sequence converges to a boundary point of $\overline\Omega$ (as in (b) or (d)), where $\mathcal{L}$ remains finite by continuity extension. Since $W\in\mathrm{St}(p,r)$ and $C\in\mathrm{St}(q,r)$ are compact factors, the sublevel set $\{\theta\in\overline\Omega:\mathcal{L}(\theta)\le\mathcal{L}(\theta_0)\}$ is compact for any $\theta_0\in\Omega$.

\textbf{Step 3 (Existence).}
A continuous function on a compact set attains its infimum. Hence $\mathcal{L}$ attains its infimum on $\overline\Omega$.
\end{proof}

\begin{proof}[Proof of Corollary~\ref{cor:interior}]
By Step~2 of the proof of Theorem~\ref{thm:slm_guarantees}, the sublevel set $K_0:=\{\theta\in\overline\Omega:\mathcal{L}(\theta)\le\mathcal{L}(\theta_0)\}$ is compact, and by Step~3 the infimum is attained at some $\theta^\star\in K_0$. Since $S\xrightarrow{P}\Sigma(\theta_0)$ and $\mathcal{L}_\infty(\theta_0)=(p+q)+\ln\det\Sigma(\theta_0)$ is the unique minimum of the population objective (Appendix~\ref{proof:thm:slm_consistency}, Step~2), any global minimizer $\hat\theta_N$ of $\mathcal{L}$ over $\overline\Omega$ satisfies $\hat\theta_N\xrightarrow{P}\theta_0$ by the Wald consistency argument in Appendix~\ref{proof:thm:slm_consistency}.

Because $\theta_0\in\mathrm{int}(\Omega)$ (all positive coordinates are strictly positive by Assumption~\ref{assump:identifiability}), convergence $\hat\theta_N\xrightarrow{P}\theta_0$ implies that, for sufficiently large $N$, $\hat\theta_N$ lies in any fixed neighborhood of $\theta_0$ with probability approaching $1$. Choosing the neighborhood inside $\mathrm{int}(\Omega)$ (which is open), we conclude $\hat\theta_N\in\mathrm{int}(\Omega)$ with probability approaching $1$. In particular, all positive coordinates of $\hat\theta_N$ are bounded away from zero with high probability, so the minimizer is an interior point of $\Omega$ rather than a boundary point.
\end{proof}

\begin{proof}[Proof of Corollary~\ref{cor:slm_opt_convergence}]
Standard convergence theorems for Riemannian descent methods \citep[Theorem~4.3.1]{Absil2008} or \citep[Chapter~4]{Boumal2023} require three conditions: (R1) the objective is $C^1$ on the manifold, (R2) iterates remain in a compact sublevel set, and (R3) the Riemannian gradient is Lipschitz on that set. All three are verified in Appendix~\ref{proof:prop:riemann_convergence}, with (R2) justified by the coercivity established in Theorem~\ref{thm:slm_guarantees}. Therefore the standard theorem applies, yielding $\|\operatorname{grad}\mathcal{L}(\theta_k)\|\to 0$ and every accumulation point is first-order critical.
\end{proof}

\subsection{Empirical landscape and the multi-start protocol}\label{sec:landscape_multistart}
SLM-Manifold and BCD-SLM optimize the same fixed-noise objective and differ only in how they traverse it. In the synthetic experiments of Section~\ref{sec:exp_synthetic}, the two strategies repeatedly converge to nearly identical objective values and parameter estimates once they reach the best basin. We treat this agreement as empirical evidence that the fixed-noise PPLS landscape is relatively benign in the regimes studied here.

Motivated by that empirical regularity, we use one shared multi-start wrapper for both solvers. Each run uses $S=8$ feasible starts, with random-QR initialization for $(W,C)$, diagonal defaults for $(\Sigma_t,B,\sigma_h^2)$, and an optional SVD warm start from the training cross-covariance replacing one random start. After every trajectory, components are reordered so that $\theta_{t,1}^2b_1\ge\cdots\ge\theta_{t,r}^2b_r$, thereby enforcing (C3). This multi-start wrapper is a heuristic safeguard, and a deterministic single-start convergence guarantee remains open. Appendix~\ref{app:implementation_details} records the same protocol for reproducibility.

\paragraph{Behaviour at the $\theta_{t,i}^2=0$ boundary and under non-convexity.} Two distinct concerns deserve an explicit statement. First, a variance coordinate may be driven toward the boundary $\theta_{t,i}^2=0$. The objective is never evaluated at $0$: each closed-form coordinate update is clipped to $\max\{\,\cdot\,,\varepsilon\}$ with $\varepsilon=10^{-8}$ (Propositions~\ref{prop:bcd_theta_closed_form} and~\ref{prop:bcd_b_cubic}), so iterates remain strictly feasible and a vanishing coordinate is returned as a \emph{constrained} minimiser on the boundary, which is the correct behaviour for a genuinely absent factor. Moreover, by Corollary~\ref{cor:interior} the global minimiser lies in the interior $\mathrm{int}(\Omega)$ with probability approaching $1$ as $N\to\infty$, so an active $\varepsilon$-floor reflects either a finite-sample effect or a factor below the detectability threshold of Proposition~\ref{prop:weak-factor}, rather than a failure of the estimator. Second, the fixed-noise objective is non-convex, so a single descent run could in principle terminate at a non-global stationary point; we mitigate this with the shared multi-start wrapper described above and retain the best basin. Empirically the two solvers converge to the same optimum, and Proposition~\ref{prop:pcca-benign} provides partial theory --- a strict-saddle structure in the PCCA specialisation --- for why the landscape is benign in the regimes studied; a deterministic single-start global guarantee remains open.

\begin{proposition}[Benign landscape on the PCCA submanifold]
\label{prop:pcca-benign}
Consider the PCCA specialization of Section~\ref{sec:special_cases}
($B = I_r$, $\sigma^2_h = 0$). Profile out $\Sigma_t$ using its closed-form
conditional maximizer, and let $\widetilde{L}(W, C)$ denote the resulting
reduced objective on $\mathrm{St}(p,r)\times\mathrm{St}(q,r)$. At the
population level (i.e., $S$ replaced by the true $\Sigma$), every Riemannian
second-order critical point of $\widetilde{L}$ is either a global minimizer
or a strict saddle. In particular, Riemannian trust-region methods initialized
from \emph{any} feasible point converge to a global minimizer almost surely.
\end{proposition}

\begin{proof}[High-level idea (full argument in Appendix~\ref{app:pcca_benign_sketch})]
After profiling out $\Sigma_t$ via its conditional maximizer, the reduced objective $\widetilde{L}(W,C)$ depends on $(W, C)$ only through the singular values of $M(W,C) := W^{\top}\Sigma_{xy}C$. Let $\Sigma_{xy} = U D V^{\top}$ be the population SVD with $D = \mathrm{diag}(d_1 \ge \dots \ge d_{\min(p,q)})$.
Critical points of $\widetilde{L}$ correspond to pairs $(W, C)$ whose columns
align with a choice of $r$ singular directions $(U_{:,S}, V_{:,S})$ for some
index set $S\subseteq\{1,\dots,\min(p,q)\}$ with $|S|=r$. The Riemannian
Hessian at such a critical point decomposes into $2\times 2$ blocks indexed
by pairs $(i, j)$ with $i\in S$, $j\notin S$; each block has eigenvalues
$\pm(d_i - d_j) + O(\sigma^2_e + \sigma^2_f)$. When $S \ne \{1,\dots,r\}$, there exists
$j\notin S$ with $d_j > d_i$ for some $i\in S$, producing a strict negative
Hessian eigenvalue, so the critical point is a strict saddle. The full derivation of the $2\times 2$ Hessian block structure, the remainder bound $\|E_{ij}\|_{\mathrm{op}}\le c(\sigma^2_e+\sigma^2_f)$, and the convergence argument are given in Appendix~\ref{app:pcca_benign_sketch}.
\end{proof}

\begin{remark}[Implication for Section~\ref{sec:landscape_multistart}]
Proposition~\ref{prop:pcca-benign} provides partial theoretical support for the
empirical observation that the fixed-noise landscape appears benign. A full
landscape analysis of the general PPLS objective (with $B\ne I_r$ and
$\sigma^2_h > 0$) remains open and is listed in Section~\ref{sec:conclusion}
as Open Problem~OP1.
\end{remark}

\subsection{Statistical properties: consistency and asymptotic normality}\label{sec:consistency_normality}

\begin{corollary}[Consistency]\label{cor:slm_consistency}
If the algorithm returns a global minimizer $\hat\theta_N$ over $\overline{\Omega}$ and identifiability holds, then $\hat\theta_N\xrightarrow{p}\theta_0$ for the true $\theta_0\in\mathrm{int}(\Omega)$.
\end{corollary}

\begin{theorem}[Asymptotic normality]\label{thm:asymptotic_normality}
Under the conditions of Theorem~\ref{thm:slm_guarantees}, and additionally assuming the population objective has a nonsingular Hessian at $\theta_0$ (in local coordinates on the constraint manifold),
\[
\sqrt{N}(\hat{\theta}_N - \theta_0) \xrightarrow{d} \mathcal{N}\!\left(0, \mathcal{I}(\theta_0)^{-1}\right),
\]
where $\mathcal{I}(\theta_0)$ denotes the Fisher information matrix of the PPLS model at $\theta_0$.
\end{theorem}

\begin{proposition}[Block-diagonal structure of the Fisher information]
\label{prop:fisher-block}
Under the fixed-noise protocol and the assumptions of
Theorem~\ref{thm:asymptotic_normality}, parameterize the scalar-parameter block as
$\eta = (\theta^2_{t,1}, b_1, \dots, \theta^2_{t,r}, b_r, \sigma^2_h)$ and the
loading block as $(W, C)$. The Fisher information matrix $I(\theta_0)$ has
the block form
\[
I(\theta_0) \;=\;
\begin{pmatrix}
I_{WC,WC} & I_{WC,\eta} \\
I_{\eta,WC} & I_{\eta\eta}
\end{pmatrix},
\quad
I_{\eta\eta} = \mathrm{blockdiag}\!\bigl(J_1, \dots, J_r, J_{\sigma_h}\bigr)
\;+\; R,
\]
where each $J_i \in \mathbb{R}^{2\times 2}$ is the Fisher information block
for $(\theta^2_{t,i}, b_i)$ computed from the $i$-th scalar summand $\ell_i$ in
Theorem~\ref{thm:theoremA}, $J_{\sigma_h} \in \mathbb{R}$ is the scalar Fisher information
for $\sigma^2_h$, and $R$ is a coupling remainder satisfying
$\|R\|_{\mathrm{op}}=O_P(r^{1/2}/p^{1/2})$ as $N\to\infty$, $p\to\infty$, $p/N\to0$, with $r$ fixed. Explicitly,
\begin{equation}
J_i \;=\; \frac{1}{2}
\begin{pmatrix}
\partial^2\ell_i/\partial(\theta^2_{t,i})^2 &
\partial^2\ell_i/\partial\theta^2_{t,i}\partial b_i \\
\partial^2\ell_i/\partial\theta^2_{t,i}\partial b_i &
\partial^2\ell_i/\partial b_i^2
\end{pmatrix}\Bigg|_{\theta_0},
\label{eq:fisher-block}
\end{equation}
with second derivatives obtained in closed form from the scalar expression
of $\ell_i$ in Theorem~\ref{thm:theoremA}; explicit formulas are listed in
Appendix~\ref{app:fisher_closed_form}.
\end{proposition}

\begin{remark}[Practical standard errors]
Proposition~\ref{prop:fisher-block} yields operationally useful standard errors
without Monte Carlo: the asymptotic variance of
$(\widehat{\theta}^2_{t,i}, \widehat{b}_i)$ is approximated by
$N^{-1}\widehat{J}_i^{-1}$, where $\widehat{J}_i$ substitutes
$\widehat{\theta}$ into~\eqref{eq:fisher-block}. This replaces per-component
bootstrap (cost $O(B \cdot \text{fit time})$) by an $O(r)$ analytic
computation. Closed-form SE expressions are collected in Appendix~\ref{app:fisher_closed_form}.
\end{remark}

\subsection{Two-stage rank-and-noise selection}\label{sec:rank_selection_fixed_noise}
The preceding subsections established the objective, its optimization, and the resulting statistical guarantees. We now turn those results into a deployable procedure by addressing the one input they assumed known: the latent rank $r$. Because Eq.~\eqref{eq:sigma_e_hat} depends on $r$ through its summation range, the noise pre-estimate and the rank must be chosen jointly. We therefore use a two-stage protocol when $r$ is unknown.

\paragraph{Two-stage procedure (data-driven mode).}
Choose a conservative upper bound $r_{\max}$ (e.g., from domain knowledge or $\min(p,q)/4$), then:
\begin{itemize}[leftmargin=*]
\item Step 1: compute $\hat\sigma_e^2(r_{\max})=(p-r_{\max})^{-1}\sum_{i>r_{\max}}\lambda_i(\hat S_{xx})$ (and similarly $\hat\sigma_f^2(r_{\max})$).
\item Step 2: for each candidate $r\in\{1,\ldots,r_{\max}\}$, fit the fixed-noise model and evaluate BIC under either V1 (shared conservative noise from $r_{\max}$) or V2 (rank-adaptive noise re-estimated at that $r$).
\item Step 3: select $\hat r=\arg\min_r \mathrm{BIC}(r)$.
\end{itemize}
For each candidate rank, let $\hat\theta_r$ denote the fitted parameter and define
\[
\mathrm{BIC}(r)=\mathcal{L}_r(\hat\theta_r)+\frac{d(r)\ln N}{N},
\qquad
 d(r)=(p-r)r+(q-r)r+2r+1.
\]
Because $\sigma_e^2,\sigma_f^2$ are pre-estimated and then fixed within each candidate fit, they are not included in $d(r)$. Explicit CV-NLL and CV-MSE definitions are given in Appendix~\ref{app:rank_selection_metrics}.

\begin{proposition}[Robustness to rank over-specification for noise estimation]\label{prop:rank_overspec_noise}
Let the true rank be $r_0$, and define $\hat\sigma_e^2(\tilde r)=(p-\tilde r)^{-1}\sum_{i>\tilde r}\lambda_i(\hat S_{xx})$. If $\tilde r\in[r_0,r_{\max}]$, then $\hat\sigma_e^2(\tilde r)$ remains consistent for $\sigma_e^2$ and
\[
\mathbb E\bigl[\hat\sigma_e^2(\tilde r)\bigr]-\sigma_e^2 = O\!\left(\frac{1}{p-\tilde r}\right)+\text{lower-order terms}.
\]
If $\tilde r<r_0$, an additional deterministic bias appears:
\[
\frac{1}{p}\sum_{\tilde r<i\le r_0}\theta_{t,i}^2,
\]
which is the same under-selection contamination pattern as in full-spectrum averaging.
\end{proposition}

\begin{remark}[Cost of rank dependence]\label{rem:rank_dep_hu_comparison}
A full-spectrum estimator is $r$-agnostic but carries the deterministic bias $p^{-1}\sum_i\theta^2_{t,i}$ (Proposition~\ref{prop:hu-bias}) for any $r$. Our noise-subspace estimator instead introduces $r$-dependence in exchange for consistency; Proposition~\ref{prop:rank_overspec_noise} shows that over-specifying $r$ within $[r_0,r_{\max}]$ preserves consistency, so this dependence is a mild and manageable cost when $r_{\max}$ is chosen conservatively.
\end{remark}

\begin{remark}[BIC consistency under fixed noise]\label{rem:bic_consistency_fixed_noise}
Assume identifiability (Assumption~\ref{assump:identifiability}) and regularity for asymptotic likelihood expansion (Theorem~\ref{thm:asymptotic_normality}), with fixed $(p,q)$ and $N\to\infty$. Let $r_0$ be the true rank. Then the fixed-noise BIC above is selection-consistent: $\Pr(\hat r_{\mathrm{BIC}}=r_0)\to1$.
\end{remark}

\begin{proposition}[Weak-factor detectability threshold for BIC]\label{prop:weak-factor}
Let $\mathrm{SNR}_k := b_k^2\,\theta^2_{t,k}\,/\,\sigma^2_e$ denote the effective
signal-to-noise ratio of the $k$-th latent factor, and let
$d_k := d(k) - d(k-1) = (p-2k+1) + (q-2k+1) + 2 = p + q - 4k + 4$
be the incremental free-parameter count of adding factor $k$ under the
identifiable PPLS parameterization. Assume the conditions of
Theorem~\ref{thm:asymptotic_normality}.
\begin{enumerate}
\item[(i)] (Detection threshold.) Under the local alternative
$\mathrm{SNR}_k = c_k\sqrt{d_k\ln N / N}$ with $c_k > 0$, there exists a
constant $c_k^{\star}>0$, depending on the Fisher information of the composite
parameter $\mathrm{SNR}_k$ at $\theta_0$, such that
\[
\Pr\bigl(\widehat{r}_{\mathrm{BIC}} \ge k\bigr)
\;\longrightarrow\;
\begin{cases}
1, & c_k > c_k^{\star}, \\
0, & c_k < c_k^{\star},
\end{cases}
\qquad N\to\infty.
\]
The constant $c_k^{\star}$ depends on the Fisher information $J_k$; we do not attempt to compute it in closed form.

\item[(ii)] (Global consistency.) If
$\mathrm{SNR}_k \gg \sqrt{d_k\ln N / N}$ for every $k \le r_0$, then
$\Pr(\widehat{r}_{\mathrm{BIC}} = r_0) \to 1$.
\end{enumerate}
\end{proposition}

\begin{proof}
\textbf{Part (i): Detection threshold.}
We derive the threshold via a Wilks-type expansion of the likelihood ratio for adding the $k$-th factor.

\emph{(a) Embedded submodel and local parameter.}
Using the scalar PPLS decomposition from Theorem~\ref{thm:theoremA}, ``adding the $k$-th factor'' corresponds to embedding the $(k-1)$-factor model $\Theta_{k-1}$ into the $k$-factor model $\Theta_k$ by activating the $k$-th component: $(\theta_{t,k}^2, b_k)$ transitions from $(0, 0)$ to positive values. Define the composite signal-to-noise parameter
\[
\rho_k := \frac{b_k\sqrt{\theta_{t,k}^2}}{\sigma_e},
\]
so that $\mathrm{SNR}_k = \rho_k^2$. The null hypothesis $H_0$: ``factor $k$ does not exist'' corresponds to $\rho_k = 0$, and the local alternative sets $\rho_k$ small but positive.

\emph{(b) Wilks-type expansion.}
Under the local alternative $\rho_k \ne 0$, the likelihood ratio statistic for testing $H_0:\rho_k=0$ satisfies
\[
\Lambda_k := 2N\bigl(\mathcal{L}_{k-1}(\hat\theta_{k-1}) - \mathcal{L}_k(\hat\theta_k)\bigr).
\]
Because the PPLS scalar form separates the $k$-th component from the first $k-1$ components (Theorem~\ref{thm:theoremA}), the contribution of the $k$-th factor to the log-likelihood is $\ell_k(\theta_{t,k}^2, b_k, \sigma_h^2)$, which depends on $\rho_k$ through the projected statistics $Q_x(k)$, $Q_y(k)$, $Q_{xy}(k)$. By a second-order Taylor expansion of $\ell_k$ around $\rho_k = 0$:
\[
\ell_k(\rho_k) = \ell_k(0) + \frac{1}{2}\,\frac{\partial^2\ell_k}{\partial\rho_k^2}\bigg|_{\rho_k=0}\,\rho_k^2 + o_P(\rho_k^2).
\]
Define the Fisher information for the composite parameter $\rho_k$ at $\theta_0$:
\[
J_k := -\mathbb{E}_{\theta_0}\!\left[\frac{\partial^2\ell_k}{\partial\rho_k^2}\right]_{\rho_k=0} > 0.
\]
The positivity $J_k > 0$ follows because $\ell_k$ is the negative log-likelihood of an identifiable model component, and the Fisher information for a non-degenerate parameter is positive definite. Summing over all $N$ observations and invoking the standard Wilks theorem framework, we obtain
\[
\Lambda_k = d_k + J_k\,N\,\rho_k^2 + o_P(1),
\]
where $d_k$ is the number of additional free parameters and $J_k N \rho_k^2 = J_k N\,\mathrm{SNR}_k$ is the non-centrality parameter. Under $H_0$, $\Lambda_k \xrightarrow{d} \chi^2_{d_k}$; under the local alternative, $\Lambda_k$ has a non-central $\chi^2_{d_k}$ distribution with non-centrality parameter $J_k N\,\mathrm{SNR}_k$.

\emph{(c) BIC threshold.}
The BIC increment is
\[
\mathrm{BIC}(k) - \mathrm{BIC}(k-1)
= -\frac{\Lambda_k}{N} + \frac{d_k\ln N}{N}.
\]
Substituting $\mathrm{SNR}_k = c_k^2\,d_k\ln N / N$ (i.e., $\rho_k = c_k\sqrt{d_k\ln N/N}$):
\begin{align*}
\mathrm{BIC}(k) - \mathrm{BIC}(k-1)
&= -\frac{d_k}{N} - J_k\,c_k^2\,\frac{d_k\ln N}{N} + \frac{d_k\ln N}{N} + o_P\!\left(\frac{\ln N}{N}\right) \\
&= \frac{d_k\ln N}{N}\bigl(1 - J_k c_k^2\bigr) + o_P\!\left(\frac{\ln N}{N}\right).
\end{align*}
When $c_k > c_k^\star := 1/\sqrt{J_k}$, the leading term is negative, so BIC favors including factor $k$ with probability tending to $1$. When $c_k < c_k^\star$, the leading term is positive, so BIC rejects factor $k$ with probability tending to $1$. This proves part~(i).

\textbf{Part (ii): Global consistency.}
If $\mathrm{SNR}_k \gg \sqrt{d_k\ln N/N}$ for every $k\le r_0$, then $J_k N\,\mathrm{SNR}_k \gg J_k d_k\ln N$, so the likelihood gain from including each true factor dominates the BIC penalty. For $k > r_0$, the true $\rho_k = 0$, so $\Lambda_k \xrightarrow{d} \chi^2_{d_k}$ and $\Lambda_k/N = O_P(1/N)$, which is dominated by $d_k\ln N/N$. Hence $\Pr(\hat r_{\mathrm{BIC}} = r_0)\to 1$ by the standard BIC consistency argument \citep{Shao1997}.
\end{proof}

\subsection{End-to-end pipeline summary}\label{sec:end_to_end_pipeline}
Algorithm~\ref{alg:full_pipeline} collects the preceding development into a single executable pipeline, and reading it top to bottom recovers the logic of the whole construction. The starting point is statistical: identifiability fixes the parameterization, and the scalar likelihood of Theorem~\ref{thm:theoremA} turns the observed-data objective into $r$ decoupled channels plus fixed noise terms. The spectral pre-estimate then removes the noise blocks from the optimization, so what remains is a well-posed problem on a smooth product manifold whose minimizer is consistent and asymptotically normal. The two solvers in Algorithm~\ref{alg:slm_manifold} and Algorithm~\ref{alg:bcd_slm} are the computational realization of that problem, and the multi-start wrapper, rank selection, shrinkage tuning, and $\kappa$-calibration around them are the engineering layers that make the estimator usable when the rank is unknown and the predictive intervals must be trustworthy on held-out data. The solver invoked inside the loop can be instantiated as either Algorithm~\ref{alg:slm_manifold} or Algorithm~\ref{alg:bcd_slm} without changing any other stage, which is what lets the same statistical guarantees carry over to whichever solver a deployment prefers.

\begin{algorithm}[t]\small
\caption{Fixed-noise PPLS pipeline with optional Gaussianization}
\label{alg:full_pipeline}
\setcounter{ALG@line}{0}
\begin{algorithmic}[1]
\State \textbf{Input:}
\State Raw data matrices $X_{\mathrm{raw}}\in\mathbb{R}^{N\times p}$ and $Y_{\mathrm{raw}}\in\mathbb{R}^{N\times q}$
\State Optional coordinate-wise transforms $\psi,\phi$ (identity if unused)
\State Rank input: either fixed $r<\min\{p,q\}$ (oracle-$r$ mode) or a grid $\mathcal R=\{1,\ldots,r_{\max}\}$ (data-driven mode)
\State Optimization strategy $\mathsf{A}\in\{\mathrm{SLM\text{-}Manifold},\mathrm{BCD\text{-}SLM}\}$
\State Number of initializations $S\in\mathbb{N}$
\State Number of inner CV folds $K_{\mathrm{in}}\in\mathbb{N}$ (e.g., $K_{\mathrm{in}}=5$)
\State Noise variant in data-driven mode: V1 (shared conservative noise) or V2 (rank-adaptive noise)
\State \textbf{Output:}
\State Estimated parameter vector $\hat\theta\in\Omega$, predictor $\tilde\mu(\cdot)$, and calibrated predictive distribution
\State \textbf{Phase 0: Gaussianization (optional)}
\State Apply $\psi$ to $X_{\mathrm{raw}}$ and $\phi$ to $Y_{\mathrm{raw}}$; set $X\leftarrow \psi(X_{\mathrm{raw}})$ and $Y\leftarrow \phi(Y_{\mathrm{raw}})$
\State Center $X$ and $Y$ using training-fold statistics
\State \textbf{Phase 1: Noise pre-estimation and fixed-noise fitting}
\State In data-driven mode, compute conservative $\hat\sigma_e^2(r_{\max})$ and $\hat\sigma_f^2(r_{\max})$ from Eq.~\eqref{eq:sigma_e_hat}
\State Rank selection: evaluate $\mathrm{BIC}(r)$ on the fixed-noise objective over $r\in\mathcal R$ (V1 uses $r_{\max}$-based noise; V2 recomputes noise at each $r$), then set $\hat r=\arg\min_r \mathrm{BIC}(r)$
\State Set working rank to $r\leftarrow \hat r$ in data-driven mode, or keep user-provided $r$ in oracle-$r$ mode
\For{$s=1,\ldots,S$}
  \State Initialize $(W_s^{(0)},C_s^{(0)})$ by random QR; if $s=1$, optionally use an SVD warm start from $\hat S_{xy}$
  \State Initialize $\Sigma_t^{(0)}\leftarrow I_r$, $B^{(0)}\leftarrow I_r$, and $\sigma_h^{2,(0)}\leftarrow 10^{-2}$
  \State $\hat\theta_s\leftarrow \mathsf{A}(W_s^{(0)},C_s^{(0)},\Sigma_t^{(0)},B^{(0)},\sigma_h^{2,(0)};\hat\sigma_e^2,\hat\sigma_f^2)$
\EndFor
\State $s^\star\leftarrow \arg\min_{s\in\{1,\ldots,S\}}\mathcal{L}(\hat\theta_s)$
\State Reorder components once on $\hat\theta_{s^\star}$ so that $\theta_{t,\pi(1)}^2b_{\pi(1)}\ge\cdots\ge\theta_{t,\pi(r)}^2b_{\pi(r)}$ to enforce (C3)
\State \textbf{Phase 2: Prediction and calibration}
\State Run $K_{\mathrm{in}}$-fold inner CV on the outer-training fold to choose $\hat\gamma\in\Gamma$ for Eq.~\eqref{eq:pred_mean}
\For{$i=1,\ldots,n_{\mathrm{val}}$}
  \State Compute $\hat\mu(x_i)$ and $V_i=\hat\Sigma_{y\mid x}(x_i)$ from Eqs.~\eqref{eq:pred_mean} and~\eqref{eq:pred_cov_main}
  \State $r_i\leftarrow y_i-\hat\mu(x_i)$
\EndFor
\State $\hat\kappa\leftarrow (n_{\mathrm{val}}q)^{-1}\sum_{i=1}^{n_{\mathrm{val}}} r_i^\top V_i^{-1}r_i$
\State For $x_{\mathrm{new}}$, report $\tilde\mu(x_{\mathrm{new}})=\hat\mu(x_{\mathrm{new}})$ and intervals from $\mathcal{N}(\hat\mu(x_{\mathrm{new}}),\hat\kappa\,\hat\Sigma_{y\mid x}(x_{\mathrm{new}}))$
\State \Return $\hat\theta\equiv\hat\theta_{s^\star}$, $\tilde\mu(\cdot)$, $\hat\kappa$, $\hat\gamma$
\end{algorithmic}
\end{algorithm}

%% file: sub/sec6_experiments.tex
\section{Experiments}\label{sec:experiments}

Our empirical evaluation is designed to validate both the theoretical guarantees and the practical deployment advantages of the fixed-noise PPLS framework. Through controlled synthetic studies, we first verify the method's computational scalability, the finite-sample accuracy of the noise-subspace estimator, and its robust parameter recovery under high noise and non-Gaussian perturbations. We then evaluate the end-to-end pipeline on two real-world multi-omics benchmarks (TCGA-BRCA and PBMC CITE-seq). In these real-world settings, we position our method against both joint-probabilistic competitors (e.g., PO2PLS) and modern deep nonlinear baselines, explicitly demonstrating its ability to deliver near-nominal uncertainty calibration without post-hoc repair, while maintaining competitive point-prediction accuracy and explicit factor interpretability.

\subsection{Datasets and Compared Methods}\label{sec:exp_setup}
Before presenting results, we fix the experimental setup. This subsection describes the synthetic and real-world datasets, the full catalog of compared methods, and the evaluation metrics for both point prediction and uncertainty calibration.

\subsubsection{Datasets}\label{sec:exp_datasets}
\paragraph{Synthetic data.}
We generate synthetic datasets directly from the PPLS model with true rank $r=5$ and sample size $N=2000$. Two noise regimes are considered: low noise $(\sigma_e^2,\sigma_f^2,\sigma_h^2)=(0.1,0.1,0.05)$ and high noise $(0.5,0.5,0.25)$. For parameter-recovery Monte Carlo we use $M=20$ trials at $p=q=200$ and $M=10$ trials at $p=q=500$; the synthetic prediction and calibration benchmark uses 5-fold CV in the canonical $p=q=200$ setting.

\paragraph{TCGA-BRCA.}
The real-data breast-cancer benchmark is the TCGA-BRCA multi-omics release distributed by \citet{brca}. In the prediction benchmark used here, the existing fold-wise preprocessing pipeline reduces the working feature sets to $p=q=60$ per outer fold. We evaluate all methods under matched 5-fold outer CV splits, with fold-wise preprocessing, feature screening, and hyperparameter selection fitted on training data only.

\paragraph{PBMC CITE-seq.}
We use the PBMC CITE-seq reference dataset of \citet{Hao2021} (see also \citealt{Stuart2019}), containing about $161{,}000$ cells, $q=228$ proteins, and the top $p=2000$ highly variable genes after preprocessing. RNA counts are library-size normalized and log-transformed; ADT counts are Centered Log-Ratio (CLR) normalized. Within each outer fold, we further screen RNA features to the top 500 genes using a hybrid variance/cross-view-correlation score, and then center/standardize both views using training-fold statistics only. All CITE-seq experiments use 3-fold outer CV.

\subsubsection{Compared Methods}\label{sec:compared_methods}
We position our fixed-noise solvers (SLM-Manifold and BCD-SLM) against three groups of baselines. First, within the PPLS family, we compare against classical EM/ECM and the penalty-based SLM-Interior \citep{hu2025slm}. Second, we evaluate PO2PLS \citep{PO2PLS}. Since PO2PLS captures view-specific orthogonal variation but lacks a native predictive posterior, we restrict it to point-prediction comparisons to explicitly quantify the trade-off between fuller decomposition and native predictive uncertainty. Finally, we benchmark against standard linear models (PLSR, Ridge \citep{HoerlKennard1970}) and nonlinear deep predictors (DCCA \citep{Andrew2013}, KCCA \citep{Hardoon2004}, MC-Dropout \citep{GalGhahramani2016}, Deep Ensembles \citep{Lakshminarayanan2017}), specifically contrasting our out-of-the-box probabilistic calibration against modern baselines that rely on post-hoc recalibration. We summarize the full method catalog and the technical differences among these methods in Appendix~\ref{app:method_catalog_appendix}.
The primary comparison is internal to the PPLS family: the key differences are noise-estimator quality (full-spectrum vs. noise-subspace) and orthogonality treatment (interior-point penalty vs. exact manifold retraction). Among the two fixed-noise solvers, we use BCD-SLM as the default implementation because it targets the same objective as SLM-Manifold while converging substantially faster in the synthetic benchmarks. Unless otherwise stated, synthetic parameter-recovery studies use oracle-$r$ mode, whereas rank-selection studies use data-driven mode with both V1 and V2 variants reported.

\subsubsection{Evaluation Metrics}\label{sec:exp_metrics}
For point prediction we report Mean Squared Error (MSE), Mean Absolute Error (MAE), and $R^2$. For uncertainty calibration at nominal miscoverage level $\alpha$, we use element-wise symmetric predictive intervals
\[
I_{ij}^{(\alpha)}=\Bigl[\hat\mu_j(x_i)\pm z_{1-\alpha/2}\sqrt{\hat\kappa\,[\hat\Sigma_{y\mid x}(x_i)]_{jj}}\Bigr]
\]
where $\hat\mu_j(x_i)$ is the $j$th coordinate of the predictive mean at sample $x_i$, $[\hat\Sigma_{y\mid x}(x_i)]_{jj}$ is the corresponding predictive variance, and $z_{1-\alpha/2}$ is the $(1-\alpha/2)$ quantile of the standard normal distribution. We then define fold-$k$ empirical coverage by
\[
\hat c_k(\alpha)=\frac{1}{n_k q}\sum_{i=1}^{n_k}\sum_{j=1}^{q}\mathbf 1\!\left\{y_{ij}\in I_{ij}^{(\alpha)}\right\}.
\]
Here $n_k:=|\mathcal V_k|$ is the number of validation samples in fold $k$, and $K_{\mathrm{out}}$ below denotes the total number of outer folds. We summarize calibration either by reporting the fold-wise coverages themselves or by the absolute calibration error
\[
\mathrm{ACE}(\alpha)=\left|\bar c(\alpha)-(1-\alpha)\right|,
\qquad
\bar c(\alpha)=\frac{1}{K_{\mathrm{out}}}\sum_{k=1}^{K_{\mathrm{out}}}\hat c_k(\alpha),
\]
where smaller ACE indicates closer nominal alignment. As a descriptive statistical check, we also test $H_0:p_\alpha=1-\alpha$ by a two-sided exact binomial test on pooled entry-wise coverage indicators; because outputs within the same sample are correlated, we treat these $p$-values as diagnostics rather than decision rules and emphasize effect sizes (coverage gaps / ACE). For synthetic parameter recovery we additionally report the parameter-estimation MSE, i.e.\ the mean squared error of the sign- and order-aligned estimates of $(W,C,B,\Sigma_t,\sigma_h^2)$. For the association-screening diagnostic we count detected gene--protein pairs at threshold $\alpha$, with the formal notation deferred to the appendix.

\subsection{Implementation Details}\label{sec:exp_implementation}
\paragraph{Protocol sensitivity.}
All SLM-based experiments use a shared multi-start protocol, namely $S=8$ feasible starts with one optional SVD warm start. The candidate shrinkage grids $\Gamma_{\mathrm{synth}},\Gamma_{\mathrm{BRCA}},\Gamma_{\mathrm{CITE}}$ used by adaptive predictive scaling are reported in the appendix.

\paragraph{Fairness and reproducibility.}
To ensure strictly fair comparisons, all evaluated methods are subject to identical data-leakage controls, including shared outer cross-validation folds, identical fold-wise preprocessing, and strictly training-only feature screening. In synthetic benchmarks, all PPLS-family solvers are evaluated from the exact same Monte Carlo initializations. Algorithmic stopping rules (e.g., tolerances and iteration caps) are pre-specified and strictly fixed across runs, with explicit configurations given in the appendix. Finally, to accurately reflect empirical stability, we report repeated-run variance: Monte Carlo standard deviations for synthetic data, and fold-level dispersion for the real-world benchmarks.

\paragraph{Cross-validation and evaluation protocol.}
To rigorously evaluate out-of-sample performance, we employ nested cross-validation across all benchmarks (5-fold for synthetic and TCGA-BRCA; 3-fold for the larger PBMC CITE-seq). For standard predictive evaluations, we report the oracle or pre-specified rank to isolate the optimization and calibration performance of the fixed-noise estimators. When evaluating data-driven rank selection, we leverage the inner CV splits over dataset-specific rank grids to test both the shared (V1) and adaptive (V2) noise protocols.

\paragraph{Baseline implementations and scalability.}
For the nonlinear competitors (DCCA, KCCA, MC-Dropout, Deep Ensembles), we use standard deep architectures and Nystr\"om approximations, with all hyperparameters (e.g., learning rates, regularization) tuned via inner CV. Due to the prohibitive computational cost of deep ensembles and kernel methods on large-scale omics, these baselines use a 10{,}000-cell training subsample on the PBMC CITE-seq benchmark, though still evaluated on the full test folds. In contrast, our fixed-noise PPLS pipeline scales to the full dataset on a standard local workstation; all reported experiments were run on a single Windows machine (AMD Ryzen AI 9 HX 370 CPU, 31.1\,GB RAM) without cluster or distributed execution. Complete network architectures, hyperparameter grids, and hardware specifications are deferred to the appendix.

\subsection{Main Results}\label{sec:exp_main_results}
We organize the main results from controlled synthetic studies to the two real-world multi-omics benchmarks. The synthetic experiments isolate parameter recovery, predictive calibration, the finite-sample noise bound, and rank selection; the TCGA-BRCA and PBMC CITE-seq results then evaluate end-to-end prediction and calibration against all baselines.

\subsubsection{Synthetic Parameter Recovery}\label{sec:exp_synthetic}
Table~\ref{tab:parameter_mse} reports parameter-recovery MSE for both dimensional settings. SLM-Manifold and BCD-SLM are numerically almost identical because they optimize the same fixed-noise objective with different update mechanics. This agreement suggests that the landscape is relatively benign in the regimes studied here. In high noise, the fixed-noise solvers consistently improve over EM/ECM on loading recovery $(W,C)$ and, for $p=q=200$, also on $(B,\Sigma_t,\sigma_h^2)$; for the $p=q=500$ high-noise setting, SLM-Manifold and BCD-SLM retain the advantage on $(W,C,B)$, whereas EM/ECM achieve lower MSE on $\sigma_h^2$ (Table~\ref{tab:parameter_mse}, lower block).

\input{generated/tables/tab_parameter_mse}

Table~\ref{tab:simulation_runtime_scaling} reports wall-clock runtimes. In the larger $p=q=500$ setting, both fixed-noise manifold solvers remain in the seconds-to-tens-of-seconds range, whereas SLM-Interior and EM/ECM can move into three- to five-digit seconds. Appendix Table~\ref{tab:algorithm_convergence} provides the complementary iteration-count statistics for the canonical $p=q=200$ setting, where BCD-SLM requires median 45 outer iterations versus 246 for SLM-Manifold. The two fixed-noise solvers reach the same statistical solution, but BCD-SLM converges more quickly and predictably.

\input{generated/tables/tab_simulation_runtime_scaling}

\begin{figure}[htbp]
\centering
\includegraphics[width=0.75\textwidth]{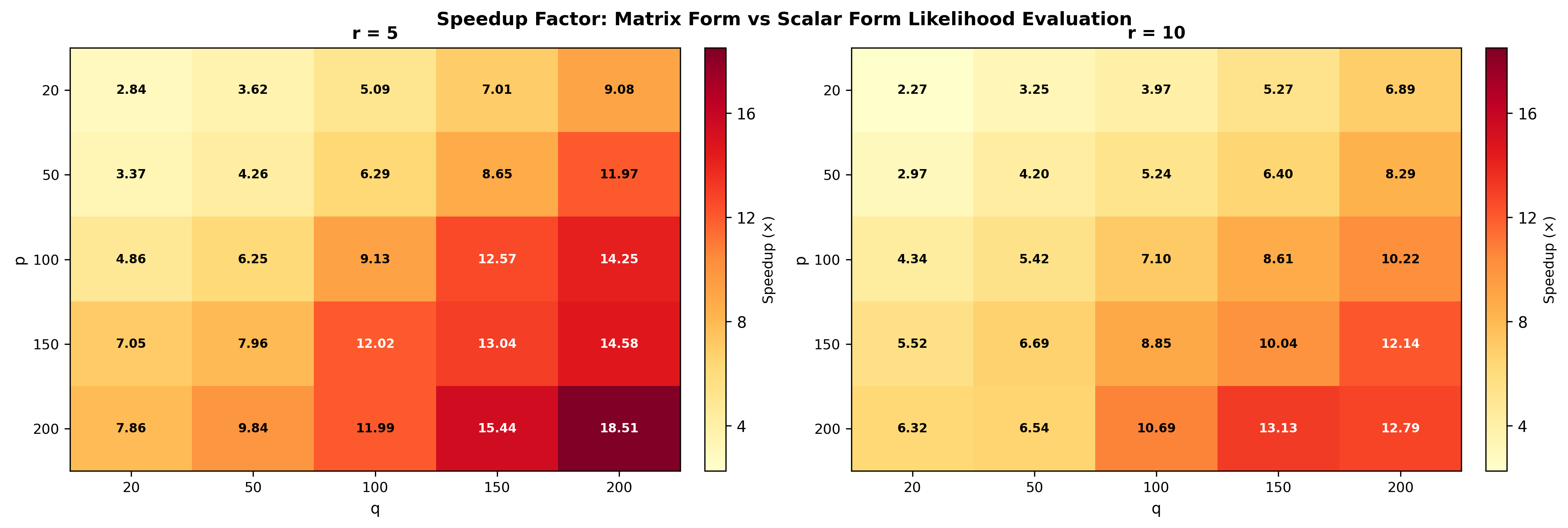}
\caption{Empirical speedup (wall-clock ratio of matrix-form to scalar-form likelihood evaluation) across a grid of $(p,q)$ (vertical axis) and $r$ (horizontal axis). Left panel: $r=5$; right panel: $r=10$. Darker cells indicate larger speedups, confirming the $O(p/r)$ acceleration from Section~\ref{sec:slm_algorithm}.}
\label{fig:speed_comparison}
\end{figure}

\IfFileExists{artifacts/simulation_convergence/convergence_curves.png}{%
\begin{figure}[t]
\centering
\includegraphics[width=0.82\textwidth]{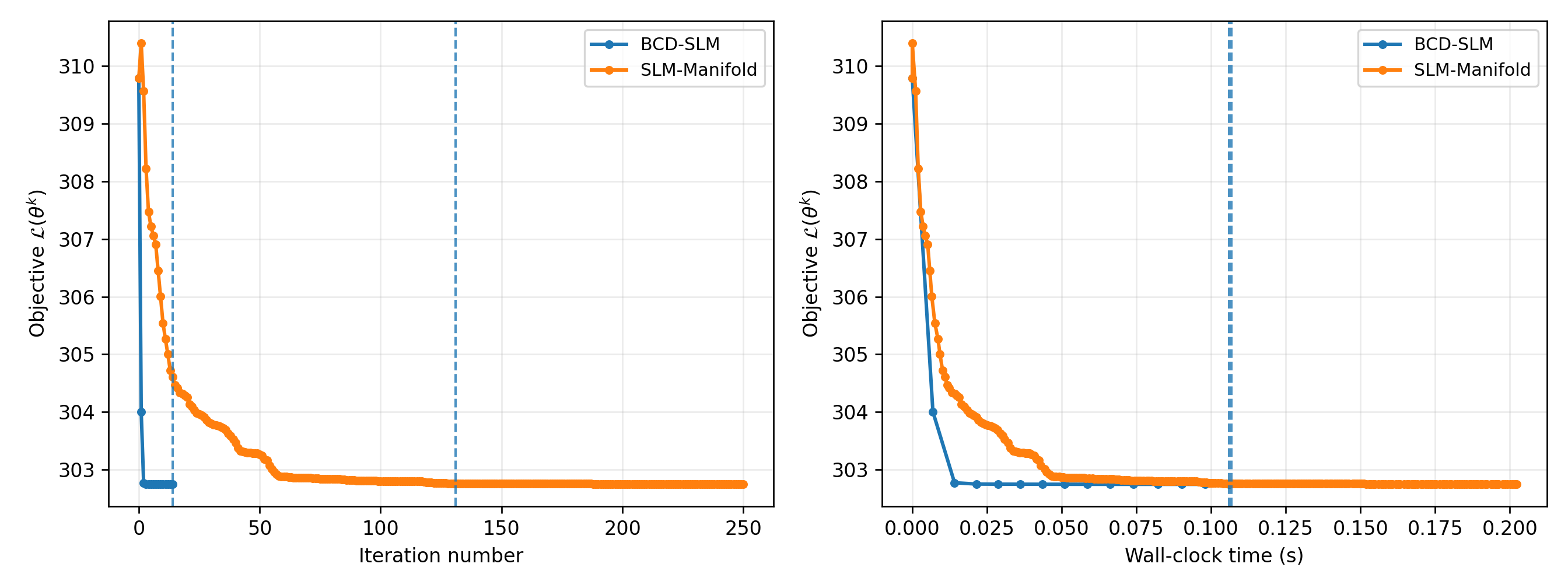}
\caption{Single-run convergence trajectories in the canonical synthetic low-noise setting ($p=q=200$, $r=5$) with matched initialization and stopping tolerance: (a) objective value versus iteration number, and (b) objective value versus wall-clock time. Dashed vertical lines mark the first hitting time of the convergence threshold for each method.}
\label{fig:convergence_curves}
\end{figure}
}{}

\subsubsection{Synthetic Prediction and Calibration}\label{sec:exp_pred_synth}
All methods achieve nearly identical point-prediction accuracy in the Gaussian synthetic benchmark: the main linear baselines are essentially tied at MSE $\approx 0.106$, with SLM-Manifold-Adaptive, EM, and PLSR all at 0.1058 and Ridge only slightly worse at 0.1079 (Appendix Table~\ref{tab:pred_synth_metrics}). Calibration is therefore the discriminating criterion; Table~\ref{tab:pred_synth_calib} reports the results.

\input{generated/tables/tab_prediction_synth_calibration}

The adaptive-uncertainty table is reported for SLM-Manifold-Adaptive; the parameter-recovery and CITE-seq tables show that BCD-SLM reaches the same fixed-noise solution up to negligible numerical differences. In that sense, the synthetic calibration advantage should be interpreted as a property of the fixed-noise PPLS framework rather than of one optimizer only. Relative to EM, the fixed-noise estimator stays systematically closer to nominal coverage levels, hence achieves smaller ACE.

\paragraph{Non-Gaussian noise robustness.}
Detailed results appear in Appendix~\ref{app:non_gaussian_noise} (Table~\ref{tab:non_gaussian_prediction}); the main conclusion is that Gaussianization is optional and Rank-INT is conservative in this benchmark.

\subsubsection{Empirical check of Theorem~\ref{thm:noise_bound}}\label{sec:exp_theorem1}
\IfFileExists{artifacts/theorem1_signal_strength/noise_error_vs_signal_strength.png}{%
\begin{figure}[t]
\centering
\includegraphics[width=0.62\textwidth]{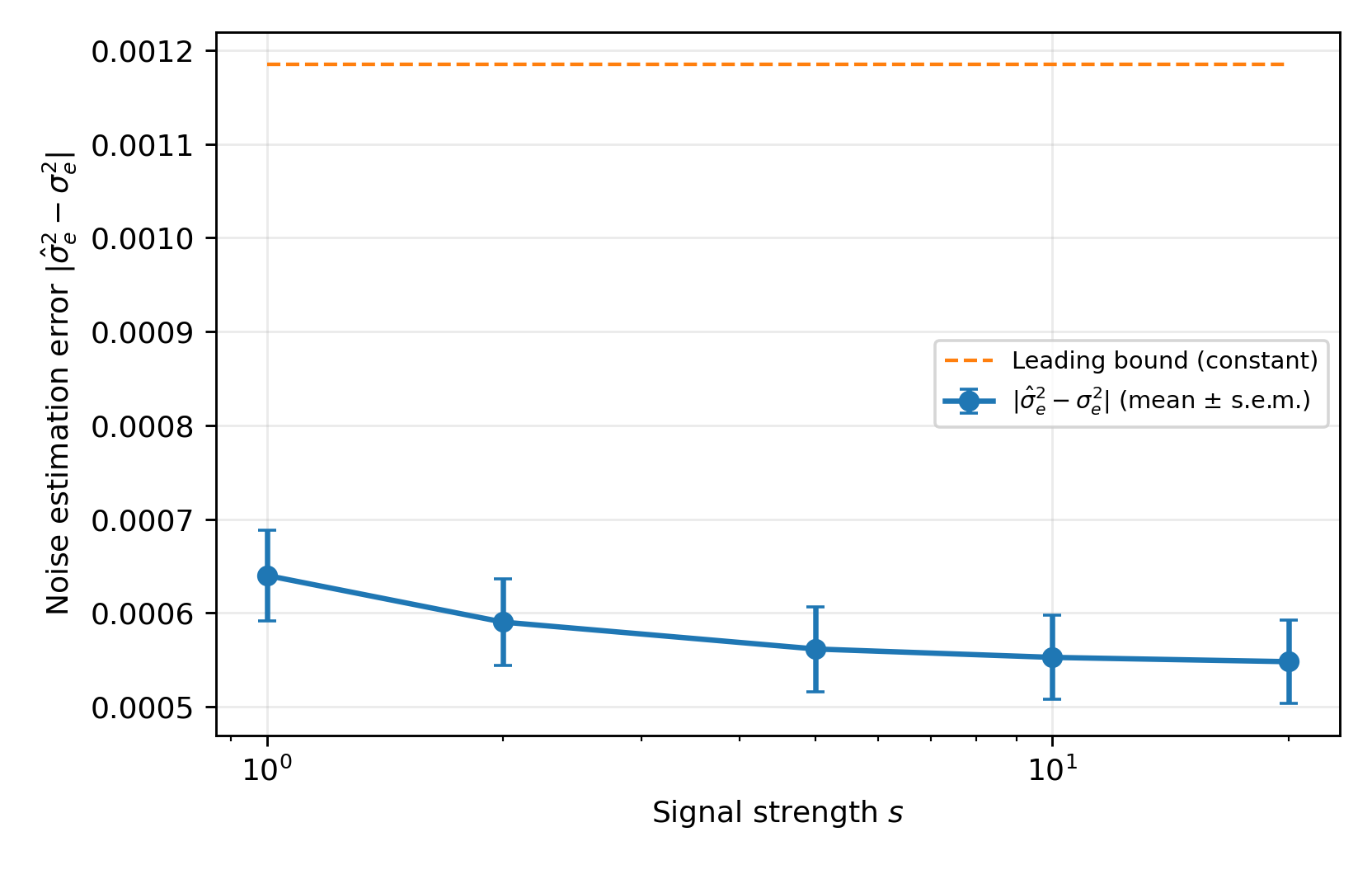}
\caption{Synthetic verification of Theorem~\ref{thm:noise_bound} under fixed noise level $\sigma_e^2=0.5$: absolute noise-estimation error $|\hat\sigma_e^2-\sigma_e^2|$ versus signal strength $s\in\{1,2,5,10,20\}$ (log-scale x-axis). Error bars are Monte Carlo standard errors over $M=80$ matched-seed runs; the dashed line is the leading bound in Eq.~\eqref{eq:noise_error_bound}.}

\label{fig:noise_signal_theorem1}
\end{figure}
}{}
The curve remains approximately flat over more than one order of magnitude in $s$, with only a mild downward drift that is consistent with the smaller eigengap-dependent second-order correction; the dominant leading term in Eq.~\eqref{eq:noise_error_bound} is signal-strength independent.

Appendix~\ref{app:hu-bias-mc} extends this comparison to wider $(p,N)$ ranges. Figure~\ref{fig:noise_ablation_vary_p} sweeps $p\in\{50,100,200,500,1000\}$ at fixed $N\!=\!2000$ and $N\in\{200,\ldots,10000\}$ at fixed $p\!=\!200$; in both regimes the noise-subspace estimator maintains consistently lower absolute error, and the gap widens with $p$ because the full-spectrum bias floor scales as $O(r/p)$ (Proposition~\ref{prop:hu-bias}).

\subsubsection{Synthetic Rank Selection}\label{sec:exp_rank}

\input{generated/tables/tab_rank_selection_summary}

Table~\ref{tab:rank_selection_summary} reports synthetic rank-selection results for eigenvalue gap, BIC, CV-NLL, and CV-MSE under data-driven mode. We report both V1 (shared conservative noise from $r_{\max}$) and V2 (rank-adaptive noise) variants. The broad picture matches the fixed-noise rank-selection analysis in Section~\ref{sec:rank_selection_fixed_noise}: in low noise, BIC and CV-NLL recover the true rank $r=5$ in all trials for both modes, while the eigenvalue-gap heuristic still occasionally under-selects and CV-MSE mildly over-selects to $\hat r=6$. In high noise, BIC and the eigenvalue-gap rule under-select because weak factors lie near the model-specific detectability threshold in Proposition~\ref{prop:weak-factor}, so their incremental likelihood gain is too small relative to complexity penalties and finite-sample Marchenko--Pastur bulk broadening blurs weak spikes. CV-MSE remains the most prediction-aligned criterion, recovering $r=5$ in 85\% of high-noise trials for both modes; its occasional $\hat r=6$ outcomes are consistent with the known mild over-selection tendency of cross-validation \citep{Shao1997}. The main V1/V2 difference appears in high-noise CV-NLL: V1 selects $r=5$ in all 20 trials, whereas V2 shifts to $\hat r=4$ in 18/20 trials, indicating that rank-adaptive noise re-estimation can make likelihood-based selection more conservative near the weak-factor boundary. Empirically, that conservativeness has little downstream effect on prediction MSE, whereas stronger under-selection removes predictive signal.

\IfFileExists{artifacts/synthetic_rank_selection/rank_vs_downstream_high_noise.png}{%
\begin{figure}[t]
\centering
\includegraphics[width=0.82\textwidth]{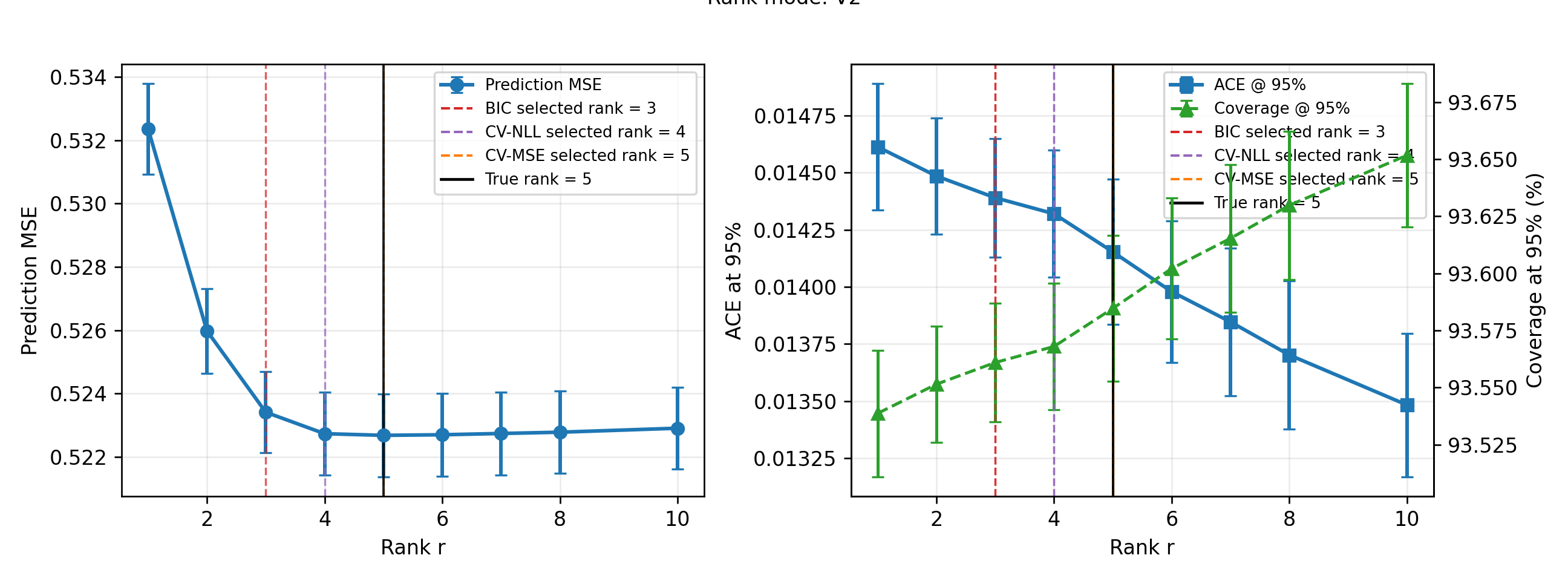}
\caption{Downstream sensitivity to rank misspecification in the synthetic high-noise setting ($p=q=200$, true $r=5$): left panel shows prediction MSE versus fitted rank, and right panel shows ACE/Coverage at 95\% nominal versus fitted rank. Dashed vertical lines mark the mode-selected ranks of BIC, CV-NLL, and CV-MSE; the solid vertical line marks the true rank.}
\label{fig:rank_downstream_high_noise}
\end{figure}
}{}

\subsubsection{TCGA-BRCA Results}\label{sec:exp_brca}

Table~\ref{tab:pred_brca_metrics} reports point-prediction results on TCGA-BRCA. SLM-Manifold-Adaptive reaches Ridge-level point-prediction accuracy with only $r=3$ components, while Table~\ref{tab:pred_brca_calib} shows that its predictive intervals remain close to nominal coverage. PO2PLS achieves MSE $=0.476$ with its auto-selected orthogonal dimensions, falling between PPLS-EM ($0.478$) and PLSR ($0.465$) but not surpassing Ridge ($0.451$) or SLM-Manifold-Adaptive ($0.450$). Combined with PO2PLS's lack of a native predictive posterior, the TCGA-BRCA results illustrate a recurring trade-off: PO2PLS spends modeling capacity on view-specific orthogonal structure, while our fixed-noise pipeline spends it on calibrated cross-view prediction. The CITE-seq results below reinforce this reading on a much larger feature space.

\input{generated/tables/tab_prediction_brca_metrics}
\input{generated/tables/tab_prediction_brca_calibration}

Small cross-method rank reversals between MSE and $R^2$ can occur because both metrics are averaged over folds with different response variances; this does not contradict the fold-wise identity $R^2=1-\mathrm{MSE}/\mathrm{Var}(y)$. At the 95\% nominal level, mean coverage is 94.87\% (ACE $0.13\%$), with only mild conservativeness at larger $\alpha$. A similar pattern also appears in the deep-baseline comparison of Table~\ref{tab:bayesian_deep_comparison}: SLM-Manifold-Adaptive attains ACE 2.71\% at MSE 0.4498, whereas the raw deep uncertainty baselines remain severely miscalibrated and need post-hoc recalibration to approach nominal coverage.

\subsubsection{CITE-seq Results}\label{sec:exp_citeseq}
Table~\ref{tab:citeseq_pred_summary} summarizes point-prediction results on PBMC CITE-seq.

\input{generated/tables/tab_citeseq_prediction_summary}

Nonlinear methods achieve lower point MSE on PBMC CITE-seq, but pay two concrete costs that the linear probabilistic pipeline avoids. First, deep baselines require an additional recalibration stage (temperature scaling \citep{Guo2017}, isotonic \citep{ZadroznyElkan2002}, or conformal \citep{VovkGammermanShafer2005}) inside each outer fold to reach usable calibration (Table~\ref{tab:bayesian_deep_comparison}); this stage consumes training data, introduces method-specific tuning, and inherits the exchangeability assumption of the chosen recalibrator. Second, nonlinear latent spaces are not directly interpretable as cross-view factors, which limits downstream biological screening. Under this positioning, SLM-Manifold-Adaptive delivers near-nominal coverage (94.74\%, 90.50\%, and 82.13\% at nominal 95\%, 90\%, and 80\%) and competitive MSE (0.2586 vs.~KCCA's 0.2462), with BCD-SLM numerically almost identical (0.2589 MSE); the 5.0\% relative MSE gap to KCCA buys model-native uncertainty, interpretable factors, and a recalibration-free deployment stack. Per-fold SLM-Manifold-Adaptive MSEs at $r=30$ are 0.2581, 0.2590, and 0.2588.

PO2PLS yields MSE $=0.475$ on CITE-seq, far above SLM-Manifold-Adaptive ($0.259$) and approaching the trivial-mean baseline ($R^2\approx 0.004$). The gap relative to TCGA-BRCA (5.5\% Ridge gap there vs.\ essentially the full-signal gap here) is consistent with two compounding factors: (i)~the auto-selected orthogonal dimensions ($r_x=r_y=5$) introduce $O((p+q)(r_x+r_y))$ extra free parameters, which scales unfavorably with the much larger $p=2000$ on CITE-seq; and (ii)~the cross-view signal in CITE-seq is more concentrated in the joint subspace, so view-specific orthogonal components consume capacity without recovering predictive signal.

Taken together with TCGA-BRCA, the two benchmarks support a consistent reading: PO2PLS is the right tool when view-specific structure itself is the inferential target, whereas the fixed-noise PPLS pipeline is the right tool when cross-view prediction with calibrated uncertainty is the target. A complementary selective-prediction analysis is reported in the appendix; its gains are modest, reflecting the near-homoscedastic structure of the linear-Gaussian model.

Figure~\ref{fig:realdata_pareto_dual} summarizes the MSE--calibration trade-off on both benchmarks; the preferred region is the left-lower corner. On TCGA-BRCA, the absolute MSE range among probabilistic methods is narrow (0.450--0.478, i.e.\ a 6.2\% relative spread); PO2PLS at 0.476 sits inside this band but is omitted from the calibration axis because it lacks a native predictive posterior, so calibration separates methods more sharply than point MSE. Triangle markers denote methods without a native 95\% ACE (N/A), including non-probabilistic baselines and CCA; BCD-SLM is shown as a point-only reference.

\IfFileExists{artifacts/prediction_brca/realdata_mse_calibration_pareto_dual.png}{%
\begin{figure}[t]
\centering
\includegraphics[width=0.92\textwidth]{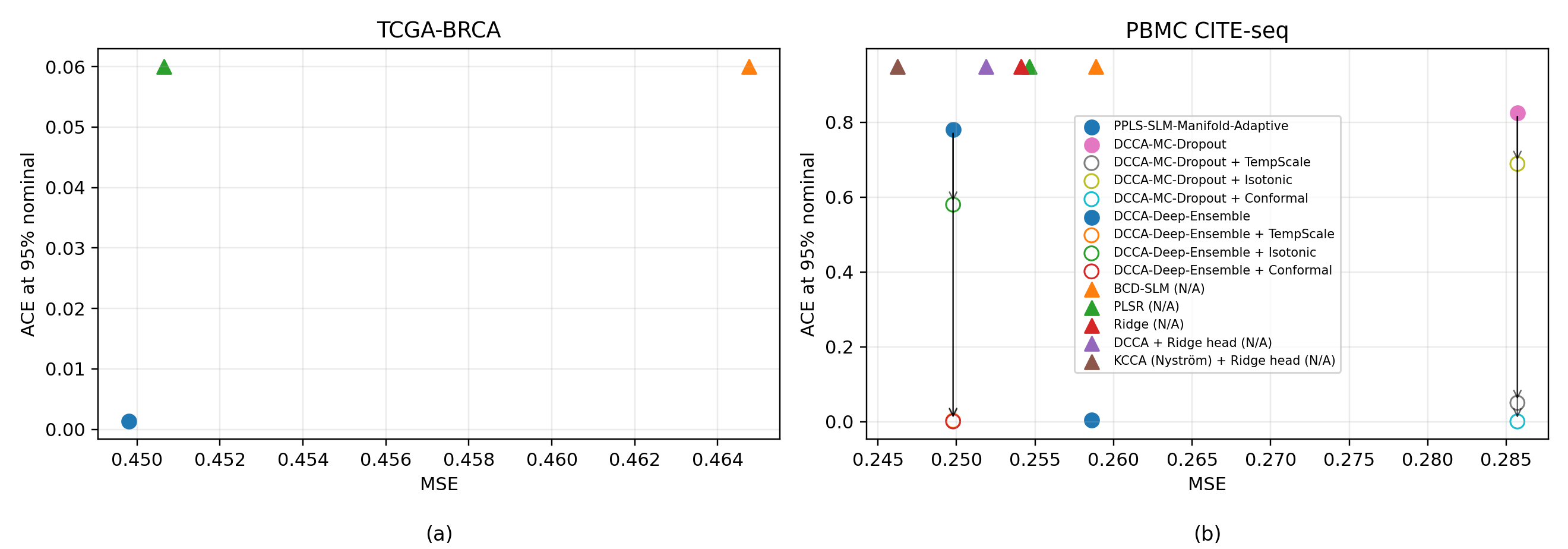}
\caption{MSE--calibration Pareto trade-offs on TCGA-BRCA and PBMC CITE-seq.}
\label{fig:realdata_pareto_dual}
\end{figure}
}{}

\subsubsection{Comparison with Bayesian Deep Baselines}\label{sec:exp_bayes_deep}
We next ask how the model-native uncertainty of the fixed-noise pipeline compares with Bayesian deep baselines, both before and after post-hoc recalibration, on the two real-data benchmarks (Table~\ref{tab:bayesian_deep_comparison}).

\IfFileExists{generated/tables/tab_bayesian_deep_comparison.tex}{\input{generated/tables/tab_bayesian_deep_comparison}}{\textit{[Bayesian deep comparison table will be auto-generated after running the corresponding experiments.]}}

Raw deep uncertainty is severely miscalibrated on both real-data benchmarks (ACE around 60--75\%), so deployment quality depends heavily on post-hoc repair. Post-hoc recalibration can numerically reduce ACE, especially conformal recalibration \citep{VovkGammermanShafer2005}, but this comes with extra assumptions and an added calibration stage inside each outer fold. In contrast, the fixed-noise PPLS pipeline outputs model-native uncertainty from a fitted conditional Gaussian law with $\kappa$-calibrated covariance. As Figure~\ref{fig:realdata_pareto_dual} shows, the linear probabilistic model does not win the absolute point-prediction frontier, yet it achieves a much cleaner accuracy--calibration trade-off without requiring an external recalibration module.

For conformal recalibration in particular, the guarantee is marginal coverage under exchangeability rather than per-sample conditional calibration \citep{VovkGammermanShafer2005,ShaferVovk2008}. In biological datasets with batch effects or emerging cell states, exchangeability can be fragile, so very small ACE values obtained only after a dedicated recalibration split should be interpreted cautiously. The recalibration stage also consumes training data, adds method-specific tuning choices, and lengthens the pipeline.

\subsection{Ablation Studies}\label{sec:exp_ablation}
Our ablation analysis covers three aspects of the pipeline: sensitivity to model selection through rank selection, robustness to non-Gaussian noise through Gaussianization, and the value of the recovered factors for downstream association screening. The first two are integrated into the main synthetic study above; we report the association-screening diagnostic next.

\paragraph{Association-screening diagnostic on TCGA-BRCA.}
For the association-screening metric, a gene--protein pair is counted as detected at threshold $\alpha$ if both the gene-to-score and protein-to-score correlations are significant at level $\alpha$; Appendix~\ref{app:assoc_notation} gives the formal notation. Table~\ref{tab:Npairs} summarizes detected-pair counts on TCGA-BRCA. At stringent thresholds, the fixed-noise manifold solution detects more pairs than EM (for example, 109138 vs.~104608 at $p<10^{-7}$), which is consistent with the better recovery of cross-view coupling parameters seen in the synthetic study.

\input{generated/tables/tab_detected_pairs}

\paragraph{Summary of experimental findings.}
Across the synthetic and real-data benchmarks, the fixed-noise PPLS framework delivers near-nominal calibration without post-hoc repair, remains practical at $p=q=500$, and stabilizes parameter recovery under high noise. The synthetic studies also show robustness to moderate non-Gaussianity and rank-selection behavior consistent with the fixed-noise analysis. On PBMC CITE-seq, nonlinear methods achieve better point accuracy, but the linear pipeline provides calibrated uncertainty, interpretable factors, and a simpler deployment story.

%% file: generated/tables/tab_parameter_mse.tex
\begin{sidewaystable}
\centering
\setlength{\tabcolsep}{2pt}
\renewcommand{\arraystretch}{1.2}
\scriptsize
\caption{Parameter estimation MSE ($\times 10^{2}$, in squared parameter units), with mean and standard deviation (SD) shown in separate columns over Monte Carlo trials, under low/high noise ($r=5$, $N=2000$). Upper block: $p=q=200$, $M=20$; lower block: $p=q=500$, $M=10$.}
\label{tab:parameter_mse}
\begin{tabular*}{0.93\textheight}{@{\extracolsep{\fill}}l*{20}{c}@{}}
\toprule
& \multicolumn{10}{c}{\textbf{Low noise}} & \multicolumn{10}{c}{\textbf{High noise}} \\
\cmidrule(lr){2-11}\cmidrule(l){12-21}
& \multicolumn{2}{c}{$\mathrm{MSE}_W$} & \multicolumn{2}{c}{$\mathrm{MSE}_C$} & \multicolumn{2}{c}{$\mathrm{MSE}_B$} & \multicolumn{2}{c}{$\mathrm{MSE}_{\Sigma_t}$} & \multicolumn{2}{c}{$\mathrm{MSE}_{\sigma_h^2}$} & \multicolumn{2}{c}{$\mathrm{MSE}_W$} & \multicolumn{2}{c}{$\mathrm{MSE}_C$} & \multicolumn{2}{c}{$\mathrm{MSE}_B$} & \multicolumn{2}{c}{$\mathrm{MSE}_{\Sigma_t}$} & \multicolumn{2}{c}{$\mathrm{MSE}_{\sigma_h^2}$} \\
\cmidrule(lr){2-3}\cmidrule(lr){4-5}\cmidrule(lr){6-7}\cmidrule(lr){8-9}\cmidrule(lr){10-11}\cmidrule(lr){12-13}\cmidrule(lr){14-15}\cmidrule(lr){16-17}\cmidrule(lr){18-19}\cmidrule(lr){20-21}
\textbf{Method} & Mean & SD & Mean & SD & Mean & SD & Mean & SD & Mean & SD & Mean & SD & Mean & SD & Mean & SD & Mean & SD & Mean & SD \\
\midrule
\multicolumn{21}{l}{\textbf{Upper block: $p=q=200$, $M=20$}} \\
SLM-Manifold & 0.01 & 0.00 & 0.02 & 0.00 & 0.04 & 0.03 & 0.06 & 0.04 & 0.00 & 0.00 & 0.08 & 0.01 & 0.06 & 0.01 & 0.37 & 0.25 & 0.84 & 0.23 & 0.26 & 0.21 \\
BCD-SLM & 0.01 & 0.00 & 0.02 & 0.00 & 0.04 & 0.03 & 0.06 & 0.04 & 0.00 & 0.00 & 0.08 & 0.01 & 0.06 & 0.01 & 0.37 & 0.25 & 0.84 & 0.23 & 0.26 & 0.21 \\
SLM-Interior & 0.23 & 0.08 & 0.22 & 0.09 & 2.52 & 1.27 & 1.72 & 0.58 & 0.43 & 0.35 & 0.22 & 0.07 & 0.17 & 0.07 & 1.64 & 1.21 & 1.17 & 0.27 & 0.45 & 0.33 \\
SLM-Oracle & 0.01 & 0.00 & 0.02 & 0.00 & 0.04 & 0.03 & 0.06 & 0.04 & 0.00 & 0.00 & 0.08 & 0.01 & 0.06 & 0.01 & 0.36 & 0.26 & 0.81 & 0.23 & 0.22 & 0.21 \\
EM & 0.02 & 0.01 & 0.03 & 0.01 & 0.06 & 0.05 & 0.06 & 0.04 & 0.01 & 0.01 & 0.09 & 0.01 & 0.06 & 0.01 & 9.89 & 2.39 & 5.47 & 1.32 & 5.00 & 1.02 \\
ECM & 0.02 & 0.01 & 0.03 & 0.01 & 0.05 & 0.05 & 0.07 & 0.04 & 0.01 & 0.01 & 0.09 & 0.01 & 0.06 & 0.01 & 6.19 & 4.13 & 2.60 & 1.52 & 3.52 & 2.20 \\
\midrule
\multicolumn{21}{l}{\textbf{Lower block: $p=q=500$, $M=10$}} \\
SLM-Manifold & 0.01 & 0.00 & 0.01 & 0.00 & 0.22 & 0.06 & 0.30 & 0.08 & 0.30 & 0.04 & 0.08 & 0.01 & 0.06 & 0.02 & 8.17 & 1.40 & 11.10 & 0.71 & 17.22 & 1.38 \\
BCD-SLM & 0.01 & 0.00 & 0.01 & 0.00 & 0.22 & 0.06 & 0.30 & 0.08 & 0.30 & 0.04 & 0.08 & 0.01 & 0.06 & 0.02 & 8.09 & 1.33 & 11.10 & 0.70 & 17.19 & 1.39 \\
SLM-Interior & 0.06 & 0.03 & 0.10 & 0.02 & 5.06 & 0.74 & 4.70 & 0.86 & 4.00 & 1.35 & 0.11 & 0.01 & 0.10 & 0.02 & 5.93 & 0.82 & 9.87 & 1.14 & 17.80 & 3.24 \\
EM & 0.01 & 0.00 & 0.01 & 0.00 & 0.17 & 0.11 & 0.14 & 0.05 & 0.00 & 0.00 & 0.08 & 0.01 & 0.07 & 0.02 & 11.81 & 1.02 & 11.01 & 2.33 & 5.31 & 0.06 \\
ECM & 0.01 & 0.00 & 0.02 & 0.00 & 0.19 & 0.11 & 0.15 & 0.06 & 0.00 & 0.01 & 0.08 & 0.01 & 0.06 & 0.01 & 10.55 & 1.23 & 9.33 & 1.62 & 5.16 & 0.05 \\
\bottomrule
\end{tabular*}
\end{sidewaystable}

%% file: generated/tables/tab_simulation_runtime_scaling.tex
\begin{table*}[t]
\centering
\setlength{\tabcolsep}{6pt}
\renewcommand{\arraystretch}{1.25}
\small
\caption{Total runtime per Monte Carlo trial across the completed simulation dimensions; the mean and standard deviation (SD) over trials are shown in separate columns (all values in seconds). SLM-Oracle is diagnostic only.}
\label{tab:simulation_runtime_scaling}
\begin{tabular*}{\textwidth}{@{\extracolsep{\fill}}lcccccccc@{}}
\toprule
\textbf{Method} & \multicolumn{4}{c}{\textbf{$p=q=200$, $M=20$}} & \multicolumn{4}{c}{\textbf{$p=q=500$, $M=10$}} \\
\cmidrule(lr){2-5}\cmidrule(lr){6-9}
 & \multicolumn{2}{c}{Low noise} & \multicolumn{2}{c}{High noise} & \multicolumn{2}{c}{Low noise} & \multicolumn{2}{c}{High noise} \\
\cmidrule(lr){2-3}\cmidrule(lr){4-5}\cmidrule(lr){6-7}\cmidrule(lr){8-9}
 & Mean & SD & Mean & SD & Mean & SD & Mean & SD \\
\midrule
SLM-Manifold & 7.82 & 3.95 & 6.49 & 2.71 & 12.66 & 2.60 & 33.81 & 33.28 \\
BCD-SLM & 13.44 & 7.18 & 5.78 & 2.65 & 31.72 & 7.86 & 23.13 & 4.24 \\
SLM-Interior & 785.18 & 150.81 & 606.21 & 216.68 & 5406.66 & 2321.61 & 5674.77 & 863.44 \\
SLM-Oracle & 7.55 & 4.13 & 6.55 & 3.33 & 22.31 & 18.55 & 19.82 & 1.76 \\
EM & 168.88 & 174.66 & 194.61 & 172.79 & 11209.53 & 11189.86 & 1239.95 & 411.90 \\
ECM & 238.37 & 156.71 & 280.10 & 174.49 & 19758.33 & 18436.18 & 3885.59 & 1125.53 \\
\bottomrule
\end{tabular*}
\end{table*}

%% file: generated/tables/tab_prediction_synth_calibration.tex
\begin{table*}[t]
\centering
\setlength{\tabcolsep}{6pt}
\renewcommand{\arraystretch}{1.25}
\small
\caption{Synthetic calibration of predictive intervals (5-fold CV), measured by element-wise empirical coverage $\hat c_k(\alpha)$; for each method the fold mean and fold standard deviation (SD) are shown in separate columns (in \%).}
\label{tab:pred_synth_calib}
\begin{tabular}{cccccc}
\toprule
& & \multicolumn{2}{c}{SLM-Manifold-Adaptive} & \multicolumn{2}{c}{EM} \\
\cmidrule(lr){3-4}\cmidrule(lr){5-6}
$\alpha$ & Expected & Mean & SD & Mean & SD \\
\midrule
0.05 & 95.00 & 95.10 & 0.12 & 94.53 & 0.14 \\ 
0.10 & 90.00 & 90.46 & 0.19 & 89.62 & 0.20 \\ 
0.15 & 85.00 & 85.81 & 0.21 & 84.79 & 0.20 \\ 
0.20 & 80.00 & 81.05 & 0.24 & 79.91 & 0.25 \\ 
0.25 & 75.00 & 76.23 & 0.31 & 75.03 & 0.31 \\ 
\bottomrule
\end{tabular}
\end{table*}

%% file: generated/tables/tab_rank_selection_summary.tex
\begin{table*}[t]
\centering
\setlength{\tabcolsep}{6pt}
\renewcommand{\arraystretch}{1.25}
\footnotesize
\caption{Synthetic rank-selection performance by mode (V1: M=20, V2: M=20) with $p=q=200$, true $r=5$, $N=2000$. Selected ranks report the empirical distribution of $\hat r$; hit rate is the fraction of trials with $\hat r=r$.}
\label{tab:rank_selection_summary}
\begin{tabular}{llcccc}
\toprule
& & \multicolumn{2}{c}{\textbf{Low noise}} & \multicolumn{2}{c}{\textbf{High noise}} \\
\cmidrule(lr){3-4}\cmidrule(lr){5-6}
\textbf{Mode} & \textbf{Criterion} & \textbf{Selected $\hat r$} & \textbf{Hit rate} & \textbf{Selected $\hat r$} & \textbf{Hit rate} \\
\midrule
V1 (shared $r_{\max}$ noise) & Eigenvalue gap & \makecell[l]{$\hat{r}=4$: 6\\$\hat{r}=5$: 14} & 70\% & \makecell[l]{$\hat{r}=2$: 3\\$\hat{r}=3$: 16\\$\hat{r}=4$: 1} & 0\% \\
 & BIC & \makecell[l]{$\hat{r}=5$: 20} & 100\% & \makecell[l]{$\hat{r}=3$: 20} & 0\% \\
 & CV-NLL & \makecell[l]{$\hat{r}=5$: 20} & 100\% & \makecell[l]{$\hat{r}=5$: 20} & 100\% \\
 & CV-MSE & \makecell[l]{$\hat{r}=5$: 15\\$\hat{r}=6$: 5} & 75\% & \makecell[l]{$\hat{r}=5$: 17\\$\hat{r}=6$: 3} & 85\% \\
\midrule
V2 (rank-adaptive noise) & Eigenvalue gap & \makecell[l]{$\hat{r}=4$: 6\\$\hat{r}=5$: 14} & 70\% & \makecell[l]{$\hat{r}=2$: 3\\$\hat{r}=3$: 16\\$\hat{r}=4$: 1} & 0\% \\
 & BIC & \makecell[l]{$\hat{r}=5$: 20} & 100\% & \makecell[l]{$\hat{r}=3$: 20} & 0\% \\
 & CV-NLL & \makecell[l]{$\hat{r}=5$: 20} & 100\% & \makecell[l]{$\hat{r}=4$: 18\\$\hat{r}=5$: 2} & 10\% \\
 & CV-MSE & \makecell[l]{$\hat{r}=5$: 16\\$\hat{r}=6$: 4} & 80\% & \makecell[l]{$\hat{r}=5$: 17\\$\hat{r}=6$: 3} & 85\% \\
\bottomrule
\end{tabular}
\end{table*}

%% file: generated/tables/tab_prediction_brca_metrics.tex
\begin{table*}[t]
\centering
\setlength{\tabcolsep}{6pt}
\renewcommand{\arraystretch}{1.25}
\small
\caption{TCGA-BRCA point-prediction accuracy (5-fold CV). For latent-dimension methods, $r^*$ minimises CV-MSE.}
\label{tab:pred_brca_metrics}
\begin{tabular}{lccccccc}
\toprule
& & \multicolumn{2}{c}{\textbf{MSE}} & \multicolumn{2}{c}{\textbf{MAE}} & \multicolumn{2}{c}{\textbf{$R^2$}} \\
\cmidrule(lr){3-4}\cmidrule(lr){5-6}\cmidrule(lr){7-8}
\textbf{Method} & $r$ & Mean & SD & Mean & SD & Mean & SD \\
\midrule
SLM-Manifold-Adaptive & 3 & 0.4498 & 0.0499 & 0.4613 & 0.0211 & 0.1724 & 0.0212 \\ 
EM & 8 & 0.4782 & 0.0496 & 0.4735 & 0.0220 & 0.1459 & 0.0186 \\ 
PO2PLS & 8 & 0.4764 & 0.0490 & 0.4731 & 0.0222 & 0.1460 & 0.0189 \\ 
PLSR & 8 & 0.4648 & 0.0540 & 0.4678 & 0.0228 & 0.1584 & 0.0256 \\ 
Ridge & -- & 0.4506 & 0.0509 & 0.4604 & 0.0218 & 0.1774 & 0.0206 \\ 
\bottomrule
\end{tabular}
\end{table*}

%% file: generated/tables/tab_prediction_brca_calibration.tex
\begin{table*}[t]
\centering
\setlength{\tabcolsep}{6pt}
\renewcommand{\arraystretch}{1.25}
\footnotesize
\caption{TCGA-BRCA calibration of SLM-Manifold-Adaptive predictive intervals, measured by per-fold element-wise empirical coverage $\hat c_k(\alpha)$ (in \%).}
\label{tab:pred_brca_calib}
\begin{tabular}{cccccccc}
\toprule
$\alpha$ & Expected & Fold 1 & Fold 2 & Fold 3 & Fold 4 & Fold 5 & Mean \\
\midrule
0.05 & 95.00 & 94.92 & 95.14 & 94.04 & 95.57 & 94.70 & 94.87 \\ 
0.10 & 90.00 & 91.62 & 91.43 & 90.32 & 92.35 & 91.35 & 91.41 \\ 
0.15 & 85.00 & 88.00 & 87.80 & 86.84 & 89.20 & 87.39 & 87.85 \\ 
0.20 & 80.00 & 83.90 & 83.94 & 82.92 & 85.37 & 83.91 & 84.01 \\ 
0.25 & 75.00 & 80.02 & 80.06 & 79.03 & 81.31 & 80.21 & 80.13 \\ 
\bottomrule
\end{tabular}
\end{table*}

%% file: generated/tables/tab_citeseq_prediction_summary.tex
\begin{table*}[t]
\centering
\setlength{\tabcolsep}{6pt}
\renewcommand{\arraystretch}{1.25}
\small
\caption{Protein imputation on PBMC CITE-seq (3-fold CV). Reported metrics are fold means; we omit standard deviations because only three outer folds are used and fold-wise variance estimates are unstable. KCCA and DCCA are the strongest pure point predictors, while the fixed-noise PPLS variants remain competitive and provide calibrated uncertainty.}
\label{tab:citeseq_pred_summary}
\begin{tabular}{lcccc}
\toprule
\textbf{Method} & \textbf{$r$/dim} & \textbf{MSE} $\downarrow$ & \textbf{MAE} $\downarrow$ & \textbf{$R^2$} $\uparrow$ \\
\midrule
SLM-Manifold-Adaptive (ours) & 30 & 0.2586 & 0.3901 & 0.3363 \\
BCD-SLM (ours) & 30 & 0.2589 & 0.3903 & 0.3359 \\
PLSR & 30 & 0.2546 & 0.3876 & 0.3434 \\
Ridge & -- & 0.2541 & 0.3871 & 0.3429 \\
DCCA + Ridge head & 20 & 0.2519 & 0.3827 & 0.3480 \\
KCCA (Nystr\"om) + Ridge head & 30 & 0.2462 & 0.3806 & 0.3574 \\
PO2PLS & 5 & 0.4749 & 0.5180 & 0.0041 \\
\bottomrule
\end{tabular}
\end{table*}

%% file: generated/tables/tab_bayesian_deep_comparison.tex
\begin{table*}[t]
\centering
\setlength{\tabcolsep}{2pt}
\renewcommand{\arraystretch}{1.2}
\scriptsize
\caption{Comparison with Bayesian deep uncertainty baselines on real-data benchmarks, including post-hoc recalibration variants (temperature scaling, isotonic recalibration, and split conformal). Coverage entries are element-wise empirical coverages; ACE is mean absolute calibration error across nominal levels $\alpha\in\{0.05,0.10,0.15,0.20,0.25\}$. Runtime is total wall-clock time over outer folds (including recalibration fitting time).}
\label{tab:bayesian_deep_comparison}
\begin{tabular*}{\textwidth}{@{\extracolsep{\fill}}llcccccccc@{}}
\toprule
\textbf{Dataset} & \textbf{Method} & \textbf{MSE} & \makecell{\textbf{95\%}\\\textbf{coverage}} & \makecell{\textbf{90\%}\\\textbf{coverage}} & \makecell{\textbf{85\%}\\\textbf{coverage}} & \makecell{\textbf{80\%}\\\textbf{coverage}} & \makecell{\textbf{75\%}\\\textbf{coverage}} & \textbf{ACE} & \makecell{\textbf{Runtime}\\\textbf{(s)}} \\
\midrule
TCGA-BRCA & \shortstack[l]{SLM-Manifold-\\Adaptive} & 0.4498 & 94.87\% & 91.41\% & 87.85\% & 84.01\% & 80.13\% & 2.71\% & 41.5 \\
TCGA-BRCA & \shortstack[l]{DCCA-MC-Dropout\\(raw)} & 0.4786 & 16.72\% & 14.09\% & 12.33\% & 10.93\% & 9.81\% & 72.23\% & 15.9 \\
TCGA-BRCA & \shortstack[l]{DCCA-MC-Dropout\\+ TempScale} & 0.4786 & 94.27\% & 91.26\% & 88.36\% & 85.30\% & 82.06\% & 3.54\% & 15.9 \\
TCGA-BRCA & \shortstack[l]{DCCA-MC-Dropout\\+ Isotonic} & 0.4786 & 39.79\% & 39.79\% & 39.79\% & 39.79\% & 39.79\% & 45.21\% & 15.9 \\
TCGA-BRCA & \shortstack[l]{DCCA-MC-Dropout\\+ Conformal} & 0.4786 & 94.92\% & 89.83\% & 84.76\% & 79.91\% & 74.99\% & 0.12\% & 15.9 \\
TCGA-BRCA & \shortstack[l]{DCCA-Deep-Ensemble\\(raw)} & 0.4652 & 30.43\% & 26.12\% & 23.04\% & 20.67\% & 18.74\% & 61.20\% & 10.9 \\
TCGA-BRCA & \shortstack[l]{DCCA-Deep-Ensemble\\+ TempScale} & 0.4652 & 94.17\% & 91.65\% & 89.17\% & 86.74\% & 84.16\% & 4.51\% & 10.9 \\
TCGA-BRCA & \shortstack[l]{DCCA-Deep-Ensemble\\+ Isotonic} & 0.4652 & 62.25\% & 62.25\% & 62.25\% & 62.25\% & 62.25\% & 22.75\% & 10.9 \\
TCGA-BRCA & \shortstack[l]{DCCA-Deep-Ensemble\\+ Conformal} & 0.4652 & 93.99\% & 88.42\% & 83.14\% & 77.44\% & 72.42\% & 1.92\% & 10.9 \\
\midrule
CITE-seq & \shortstack[l]{SLM-Manifold-\\Adaptive} & 0.2586 & 94.74\% & 90.50\% & 86.31\% & 82.13\% & 77.90\% & 1.42\% & 613.3 \\
CITE-seq & \shortstack[l]{DCCA-MC-Dropout\\(raw)} & 0.3045 & 12.48\% & 10.71\% & 9.51\% & 8.55\% & 7.75\% & 75.20\% & 931.3 \\
CITE-seq & \shortstack[l]{DCCA-MC-Dropout\\+ TempScale} & 0.3045 & 89.99\% & 88.18\% & 86.64\% & 85.17\% & 83.70\% & 4.47\% & 934.1 \\
CITE-seq & \shortstack[l]{DCCA-MC-Dropout\\+ Isotonic} & 0.3045 & 26.11\% & 26.11\% & 26.11\% & 26.11\% & 26.11\% & 58.89\% & 932.7 \\
CITE-seq & \shortstack[l]{DCCA-MC-Dropout\\+ Conformal} & 0.3045 & 95.03\% & 90.08\% & 85.13\% & 80.15\% & 75.17\% & 0.11\% & 932.3 \\
CITE-seq & \shortstack[l]{DCCA-Deep-Ensemble\\(raw)} & 0.2656 & 17.00\% & 14.44\% & 12.73\% & 11.39\% & 10.26\% & 71.83\% & 613.7 \\
CITE-seq & \shortstack[l]{DCCA-Deep-Ensemble\\+ TempScale} & 0.2656 & 95.20\% & 93.50\% & 91.95\% & 90.42\% & 88.84\% & 6.98\% & 615.2 \\
CITE-seq & \shortstack[l]{DCCA-Deep-Ensemble\\+ Isotonic} & 0.2656 & 37.01\% & 37.01\% & 37.01\% & 37.01\% & 37.01\% & 47.99\% & 614.8 \\
CITE-seq & \shortstack[l]{DCCA-Deep-Ensemble\\+ Conformal} & 0.2656 & 94.95\% & 89.92\% & 84.85\% & 79.79\% & 74.71\% & 0.16\% & 614.6 \\
\bottomrule
\end{tabular*}
\end{table*}

%% file: generated/tables/tab_detected_pairs.tex
\begin{table*}[t]
\centering
\setlength{\tabcolsep}{6pt}
\renewcommand{\arraystretch}{1.25}
\small
\caption{Number of detected gene--protein pairs by SLM-Manifold and EM under different $p$-value thresholds.}
\label{tab:Npairs}
\begin{tabular}{lcccc}
\toprule
\textbf{Method} & \multicolumn{4}{c}{\textbf{$p$-value threshold}} \\
\cmidrule(lr){2-5}
 & $10^{-7}$ & $10^{-6}$ & $10^{-5}$ & $10^{-4}$ \\
\midrule
SLM-Manifold & 109138 & 115608 & 122199 & 126600 \\
EM & 104608 & 111453 & 117940 & 124362 \\
\midrule
Overlap & 95638 & 103812 & 112499 & 120541 \\
\bottomrule
\end{tabular}
\end{table*}

%% file: sub/sec7_discussion.tex
\section{Discussion}\label{sec:discussion}

This section synthesizes the empirical and methodological trade-offs that are only partially visible from per-table comparisons.

\paragraph{When to use BCD-SLM vs. SLM-Manifold.}
Both solvers optimize the same fixed-noise objective and typically reach nearly identical statistical solutions in our experiments. In practice, BCD-SLM is the default when runtime and robustness are primary concerns, because its componentwise closed-form updates reduce per-iteration uncertainty and improve wall-clock predictability. SLM-Manifold remains useful as a conceptually clean full-manifold baseline and as a generic template when one later modifies the scalar blocks beyond closed-form solvability.

\paragraph{Failure mode near weak-factor boundaries.}
The fixed-noise strategy is most reliable when latent factors are sufficiently separated from the noise bulk. Near the weak-factor detectability boundary, under-selection can occur because incremental likelihood gains become comparable to complexity penalties and finite-sample spectral broadening. A concrete instance is the high-noise synthetic setting with true $r=5$: there, the weakest factor sits close to the Marchenko--Pastur bulk edge, and likelihood-based criteria such as BIC select $\hat r=3$ rather than $5$, whereas the prediction-aligned CV-MSE still recovers $\hat r=5$ in most trials. In such regimes, conservative rank-grid design and reporting sensitivity across adjacent ranks are advisable.

\paragraph{Calibration--interpretability trade-off vs. nonlinear baselines.}
On TCGA-BRCA and PBMC CITE-seq, nonlinear baselines can improve point MSE, especially in the more nonlinear CITE-seq regime, but usually require an additional recalibration stage to obtain usable uncertainty. The fixed-noise PPLS pipeline targets the opposite corner: slightly weaker point-optimality in some settings, but model-native calibration, explicit orthogonality-constrained factors, and a simpler uncertainty deployment path.

\paragraph{Joint factors vs.\ view-specific orthogonal decomposition.}
Across both real-data benchmarks, PO2PLS underperforms the fixed-noise pipeline on cross-view prediction despite using a strictly richer parameterization. This is not evidence against orthogonal decomposition as a modeling tool, since PO2PLS remains the natural choice when the inferential question is \emph{what is view-specific}, but it does indicate that, under the identifiable PPLS joint structure, the cross-view predictive signal is already captured by the joint loading pair $(W, C)$. Adding view-specific orthogonal components increases parameter count without recovering additional predictive signal, and the cost grows with the larger view dimension (most visibly on CITE-seq with $p=2000$). Practically, the two methods are complementary: PO2PLS for view-specific interpretation, fixed-noise PPLS for calibrated cross-view prediction.

\paragraph{On regimes where EM/ECM can be competitive.}
Our synthetic results also show that EM/ECM can outperform fixed-noise solvers on selected parameters (notably $\sigma_h^2$ in some high-dimensional high-noise settings). This indicates that fixed-noise decoupling is not uniformly dominant for every parameter block, and motivates hybrid schemes that retain exact manifold feasibility while adaptively re-coupling selected noise components when diagnostics indicate underfitting in the residual channel.

%% file: sub/sec8_conclusion.tex
\section{Conclusion}\label{sec:conclusion}

We showed that direct optimization of the PPLS observed-data likelihood on the product Stiefel manifold, with spectrally pre-estimated noise variances, yields a practical linear probabilistic pipeline with model-native uncertainty calibration. In the $p=q=500$ high-noise setting, SLM-Manifold and BCD-SLM retain the advantage on $(W,C,B)$, whereas EM/ECM achieve lower MSE on $\sigma_h^2$ (Table~\ref{tab:parameter_mse}, lower block). On real data, the method achieves near-nominal coverage without an external post-hoc calibration stack; deep baselines can approach calibration targets only after additional recalibration stages. This shifts the practical advantage from ``best single calibration number'' to deployment simplicity, interpretability, and reduced calibration fragility. Relative to PO2PLS, which extends the joint PPLS model with view-specific orthogonal components, our pipeline forgoes explicit view-specific decomposition in exchange for native predictive calibration; the empirical comparison indicates these are complementary modeling targets rather than competing ones. Computationally, the manifold solvers remain practical at $p=q=500$ with seconds-to-tens-of-seconds runtime. We also extended the theory to non-Gaussian observation noise via Gaussianization, showing that the spectral noise estimator's leading finite-sample rate is preserved up to an additive excess-kurtosis bias term.

Several open problems remain. We list them in order of increasing scope.

\textbf{OP1 (Global landscape for general PPLS).}
Proposition~\ref{prop:pcca-benign} establishes a benign landscape on the PCCA
submanifold. Extending this to the full PPLS parameter space, with $B\ne I_r$
and $\sigma^2_h > 0$, would yield a deterministic single-start convergence
guarantee and eliminate the reliance on the multi-start safeguard in
Section~\ref{sec:landscape_multistart}.

\textbf{OP2 (Heteroscedastic noise).}
The fixed-noise protocol extends cleanly only to isotropic noise. Coordinate-
wise variances $\sigma^2_{e,j}, \sigma^2_{f,k}$ introduce $O(p+q)$ additional
free parameters and break the scalar decomposition in
Theorem~\ref{thm:theoremA}. A natural starting point is block-diagonal noise
with fixed group sizes; see Section~\ref{sec:gaussianization_extension},
Remark~\ref{rem:noise_model_relaxation} for why Gaussianization is an interim
solution but not a full substitute.

\textbf{OP3 (Non-Gaussian likelihoods).}
Theorem~\ref{thm:noise_bound_subgaussian} handles sub-Gaussian noise at the
estimation level, but the model-level likelihood still assumes Gaussianity. A
native exponential-family PPLS (e.g., negative binomial for count omics) would
remove the need for pre-transformation in Section~\ref{sec:gaussianization_extension}
and is compatible with the scalar decomposition through variance-stabilizing
reparameterizations.

\textbf{OP4 ($r$-agnostic noise-subspace estimation).}
Our estimator in Eq.~\eqref{eq:sigma_e_hat} trades the $r$-agnostic property of full-spectrum averaging for consistency and a signal-independent leading rate. It remains open to design a unified estimator that is both consistent and $r$-agnostic, for example one coupled to Marchenko--Pastur spike detection.

%% file: sub/appendix.tex

\part*{Appendix}
\addcontentsline{toc}{section}{Appendix}

\noindent
This appendix collects the material that supports, but is not essential to follow, the main text. It is organized into four parts. Appendix~\ref{app:proofs} gives the supplementary proofs for the results stated in the main text, including the block-matrix identity behind the scalar likelihood, the Wald consistency and asymptotic-normality arguments, and the optimization-convergence verification. Appendix~\ref{app:hu_comparison} provides the detailed derivations specific to the fixed-noise scalar form, namely the cubic and closed-form component updates, the Davis--Kahan and sub-Gaussian spectral arguments, the closed-form Fisher blocks, and the full PCCA benign-landscape proof. Appendix~\ref{app:exp_supp} reports additional experimental results, including supplementary optimization and prediction tables, the method catalog, the PCCA and PPCA verifications, the bias-floor and weak-factor studies, the non-Gaussian noise experiments, and the association-screening notation. Appendix~\ref{app:scalar_details} records the implementation and reproducibility details, covering the Riemannian optimization specifics, initialization, and asset licenses.

\medskip
\noindent\textbf{Contents of the appendix.}
\begin{itemize}[leftmargin=*,nosep]
\item Appendix~\ref{app:proofs}: Supplementary proofs.
\item Appendix~\ref{app:hu_comparison}: Additional derivations for the fixed-noise scalar form.
\item Appendix~\ref{app:exp_supp}: Additional experimental results.
\item Appendix~\ref{app:scalar_details}: Implementation and reproducibility details.
\end{itemize}

\medskip

\section{Supplementary Proofs}\label{app:proofs}

\noindent\textbf{Notation.}\;
We use standard asymptotic notation throughout: $a_N\lesssim b_N$ means $a_N\le C b_N$ for some absolute constant $C>0$; $a_N\asymp b_N$ means $c b_N\le a_N\le C b_N$ for absolute constants $0<c<C$; $a_N\ll b_N$ means $a_N/b_N\to0$; $O_P(\cdot)$ denotes the usual stochastic boundedness (in probability). The symbol $C>0$ denotes a generic absolute constant that may take different values at different occurrences; its dependency on problem parameters $(p,r,N)$ is stated when first introduced in each proof. The notation $a_N\lesssim b_N$ is used uniformly in place of $a_N\le C\cdot b_N$ throughout the appendix; the implied constant's dependency is indicated by a subscript when relevant (e.g., $\lesssim_K$ denotes a constant depending on $K$).

\subsection{Block-matrix identity used in the scalar likelihood}\label{app:block_matrix_identity}
\begin{lemma}[Determinant and inverse of a block matrix in the PPLS model; \citet{hu2025slm}]\label{lem:rank_n}
Let $n \geq 1$ and let $D_n$ be an $n \times n$ diagonal matrix with positive diagonal entries $\sigma_1,\ldots,\sigma_n$. Let $m>n$ and let $A\in\mathbb{R}^{m\times n}$ have orthonormal columns, i.e., $A^\top A=I_n$. Let $k>0$. Then
\begin{equation}\label{RankNDet}
\det(AD_nA^\top + kI_m) = k^{m-n} \prod_{i=1}^n (k + \sigma_i),
\end{equation}
\begin{equation}\label{RankNInverse}
(AD_nA^\top + kI_m)^{-1} = \frac{1}{k} \left(I_m - A\left(I_n + kD_n^{-1}\right)^{-1}A^\top\right).
\end{equation}
\end{lemma}

\subsection{Wald consistency details}\label{proof:thm:slm_consistency}
\begin{proof}
We follow a standard Wald consistency argument. Write $\Sigma_0:=\Sigma(\theta_0)$ for the population covariance at the true parameter, define the population objective by $\mathcal{L}_\infty(\theta):=\ln\det\Sigma(\theta)+\mathrm{tr}(\Sigma_0\Sigma(\theta)^{-1})$, let $K_0\subset\overline\Omega$ be a compact sublevel set containing $\theta_0$ and, eventually, $\hat\theta_N$, and let $C$ denote a deterministic positive constant that may change from line to line.

\textbf{Step 1 (Uniform convergence).}
Since $\mathcal{L}(\theta) - \mathcal{L}_\infty(\theta) = \mathrm{tr}((S-\Sigma_0)\Sigma(\theta)^{-1})$ and $\|\Sigma(\theta)^{-1}\|_F$ is uniformly bounded on a compact sublevel set $K_0$, Cauchy--Schwarz gives
\[
\sup_{\theta \in K_0} |\mathcal{L}(\theta) - \mathcal{L}_\infty(\theta)| \leq C\,\|S-\Sigma_0\|_F \xrightarrow{\mathrm{a.s.}} 0.
\]

\textbf{Step 2 (Unique minimizer).}
Applying the KL-based lower bound with $S$ replaced by $\Sigma_0$ yields $\mathcal{L}_\infty(\theta) \ge (p+q) + \ln\det\Sigma_0$, with equality iff $\Sigma(\theta)=\Sigma_0$. This follows from $\mathrm{KL}(\mathcal{N}(0,\Sigma_0)\|\mathcal{N}(0,\Sigma(\theta))) = \tfrac{1}{2}[\mathrm{tr}(\Sigma(\theta)^{-1}\Sigma_0) - (p+q) + \ln(\det\Sigma(\theta)/\det\Sigma_0)] \ge 0$, which upon rearrangement gives $\ln\det\Sigma(\theta)+\mathrm{tr}(\Sigma_0\Sigma(\theta)^{-1}) \ge (p+q)+\ln\det\Sigma_0$, with equality iff $\Sigma(\theta)=\Sigma_0$. By identifiability, this implies $\theta=\theta_0$.

\textbf{Step 3 (Consistency).}
For any $\varepsilon>0$, define $K_\varepsilon = \{\theta \in K_0 : \|\theta-\theta_0\| \ge \varepsilon\}$ and $\eta = \inf_{K_\varepsilon} \mathcal{L}_\infty(\theta) - \mathcal{L}_\infty(\theta_0) > 0$. Since $\hat{\theta}_N$ minimizes $\mathcal{L}$,
\[
\eta \le \mathcal{L}_\infty(\hat{\theta}_N) - \mathcal{L}_\infty(\theta_0) \le 2\sup_{\theta \in K_0}|\mathcal{L}(\theta)-\mathcal{L}_\infty(\theta)|,
\]
so $\Pr(\hat{\theta}_N \in K_\varepsilon) \to 0$.
\end{proof}

\subsection{Asymptotic normality proof details}\label{proof:thm:asymptotic_normality}
\begin{proof}
\textbf{Chart construction.} We work in local coordinates on $\mathcal{M}_r=\mathrm{St}(p,r)\times\mathrm{St}(q,r)\times\mathbb{R}_{++}^{2r+1}$. For $W\in\mathrm{St}(p,r)$, we use exponential coordinates: choose a smooth section $\sigma$ of the orthonormal frame bundle and parameterize a neighborhood of $W_0$ by $\eta_W\in\mathbb{R}^{(p-r)r}$ via $W=\mathrm{Exp}_{W_0}(\xi_W)$, where $\xi_W=\sum_k\eta_{W,k}E_k$ is expressed in the chosen basis $\{E_k\}$ of $T_{W_0}\mathrm{St}(p,r)$ and $\mathrm{Exp}$ is the Riemannian exponential map on the Stiefel manifold \citep[Chapter~3]{Boumal2023}. Similarly for $C\in\mathrm{St}(q,r)$ with coordinates $\eta_C\in\mathbb{R}^{(q-r)r}$. For the positive coordinates $(\theta_{t_i}^2,b_i,\sigma_h^2)$, we use log-parameterization: $\eta_{\theta,i}=\ln\theta_{t_i}^2$, $\eta_{b,i}=\ln b_i$, $\eta_{\sigma_h}=\ln\sigma_h^2$. The total dimension is $d=(p-r)r+(q-r)r+2r+1$, matching the BIC parameter count in Section~\ref{sec:rank_selection_fixed_noise}. Let $\phi$ denote the combined chart map, with local parameter $\eta=\phi(\theta)\in\mathbb{R}^d$ and $\eta_0:=\phi(\theta_0)$ for the true parameter.

We write $\hat\eta_N:=\phi(\hat\theta_N)$ for the estimator in this chart, and $\tilde\eta_N$ for an intermediate point on the segment joining $\eta_0$ and $\hat\eta_N$.

\textbf{Step 1 (Score CLT).} Let $\mathcal{L}_N(\eta)$ denote the empirical objective in local coordinates. By independence and finite second moments of Gaussian score terms,
\[
\sqrt{N}\,\nabla_\eta \mathcal{L}_N(\eta_0)\xrightarrow{d}\mathcal{N}(0,\mathcal{I}(\theta_0)),
\]
where $\mathcal{I}(\theta_0)$ is the Fisher information matrix in the local coordinates defined by the chart $\phi$.

\textbf{Step 2 (Hessian convergence).} The Hessian $\nabla_\eta^2\mathcal{L}_N(\eta)$ converges uniformly in probability to $\nabla_\eta^2\mathcal{L}_\infty(\eta)$ on a neighborhood of $\eta_0$, because $\mathcal{L}_N$ is a smooth function of the sufficient statistics $S$ which converge uniformly on compact parameter sets, and the chart $\phi$ is smooth. By assumption, $\nabla_\eta^2\mathcal{L}_\infty(\eta_0)=\mathcal{I}(\theta_0)$ is nonsingular.

\textbf{Step 3 (Taylor expansion).} By Corollary~\ref{cor:slm_consistency}, $\hat\eta_N:=\phi(\hat\theta_N)\xrightarrow{P}\eta_0$. Expanding the score at $\hat\eta_N$ around $\eta_0$ gives
\[
0=\nabla_\eta\mathcal{L}_N(\hat\eta_N)=\nabla_\eta\mathcal{L}_N(\eta_0)+\nabla_\eta^2\mathcal{L}_N(\tilde\eta_N)(\hat\eta_N-\eta_0),
\]
where $\tilde\eta_N$ lies between $\eta_0$ and $\hat\eta_N$. The uniform convergence of the Hessian from Step~2 ensures $\nabla_\eta^2\mathcal{L}_N(\tilde\eta_N)\xrightarrow{P}\mathcal{I}(\theta_0)$, and rearranging and applying Slutsky's theorem,
\[
\sqrt{N}(\hat\eta_N-\eta_0)\xrightarrow{d}\mathcal{N}(0,\mathcal{I}(\theta_0)^{-1}).
\]
Mapping back through the chart diffeomorphism yields the stated limit for $\hat\theta_N$ \citep[Theorem~5.39]{vanderVaart,NeweyMcFadden1994}.
\end{proof}

\subsection{Verification details for optimization convergence}\label{proof:prop:riemann_convergence}
\begin{proof}
The standard convergence theorem~\cite[Theorem~4.3.1]{Absil2008} requires: (R1) objective $C^1$ on manifold, (R2) iterates in a compact sublevel set, and (R3) Lipschitz Riemannian gradient on that set.

\emph{(R1)} By Theorem~\ref{thm:theoremA}, $\mathcal{L}(\theta)$ is composed of smooth scalar terms with strictly positive denominators under fixed-noise interior constraints, hence $\mathcal{L}\in C^\infty$.

\emph{(R2)} Stiefel factors are compact; coercivity from Theorem~\ref{thm:slm_guarantees} bounds positive coordinates away from infinity within sublevel sets.

\emph{(R3)} On any compact subset, the projected Euclidean gradient in Eq.~\eqref{eq:euclidean_gradients} is a smooth map on a compact domain, hence Lipschitz.

Therefore the required theorem applies.
\end{proof}

\section{Additional derivations for the fixed-noise scalar form}\label{app:hu_comparison}
For reference, Table~\ref{tab:vs-hu-app} summarizes the design choices of the fixed-noise framework and the guarantees established for each, contrasting our noise-subspace estimator and exact-retraction optimization with the full-spectrum, penalty-based alternatives where relevant.
\begin{table}[h]
\centering
\caption{Summary of the fixed-noise framework's design choices and guarantees.}
\label{tab:vs-hu-app}
\footnotesize
\renewcommand{\arraystretch}{1.2}
\begin{tabular}{@{}lll@{}}
\toprule
Component & Full-spectrum / penalty baseline & This framework \\
\midrule
Noise handling & Fixed-noise protocol & Fixed-noise protocol (same) \\
Scalar likelihood derivation & Standard expansion & Theorem~\ref{thm:theoremA} \\
Noise estimation & Full spectrum $\frac{1}{p}\sum_i\lambda_i$ & Noise subspace $\frac{1}{p-r}\sum_{i>r}\lambda_i$ (Eq.~\ref{eq:sigma_e_hat}) \\
Rank dependence of noise est. & $r$-agnostic & requires an input $r$ (mitigated by \S\ref{sec:rank_selection_fixed_noise}) \\
Consistency of noise est. & Inconsistent (Prop.~\ref{prop:hu-bias}) & Consistent (Cor.~\ref{cor:consistency}) \\
Finite-sample rate & Signal-dependent & Signal-independent (Thm.~\ref{thm:noise_bound}) \\
Minimax optimality & Not established & Established (Thm.~\ref{thm:minimax-lb}) \\
Orthogonality handling & Interior-point penalty & Exact manifold retraction (\S\ref{sec:optimization_algorithms}) \\
Sub-Gaussian extension & Not considered & Theorem~\ref{thm:noise_bound_subgaussian} \\
Non-Gaussian observations & Not considered & Gaussianization layer (\S\ref{sec:gaussianization_extension}) \\
Inference (asymp.\ normality) & Not established & Theorem~\ref{thm:asymptotic_normality}, Prop.~\ref{prop:fisher-block} \\
Landscape guarantee & None & PCCA benign (Prop.~\ref{prop:pcca-benign}) \\
\bottomrule
\end{tabular}
\end{table}

\subsection{Proof details for Proposition~\ref{prop:bcd_b_cubic}}\label{app:proof_b_cubic}
Write $s:=\theta_{t_i}^2$, $\alpha:=\sigma_e^2$, $\beta:=\sigma_f^2$, and $\gamma:=\sigma_h^2$. With
\[
D(b)=(\beta+\gamma)(s+\alpha)+b^2s\alpha,
\quad
N(b)=(\beta+\gamma)sQ_x(i)+\gamma(s+\alpha)Q_y(i)+b^2s\alpha Q_y(i)+bsQ_{xy}(i),
\]
we have
\[
\frac{\partial \ell_i}{\partial b}=\frac{D'(b)(D(b)+N(b))-N'(b)D(b)}{D(b)^2},
\]
where $D'(b)=2bs\alpha$ and $N'(b)=2bs\alpha Q_y(i)+sQ_{xy}(i)$. Setting the numerator to zero and expanding:
\begin{align*}
& D'(b)(D(b)+N(b))-N'(b)D(b) \\
&= 2bs\alpha\bigl[(\beta+\gamma)(s+\alpha)+b^2s\alpha + (\beta+\gamma)sQ_x(i)+\gamma(s+\alpha)Q_y(i)+b^2s\alpha Q_y(i)+bsQ_{xy}(i)\bigr] \\
&\quad - \bigl[2bs\alpha Q_y(i)+sQ_{xy}(i)\bigr]\bigl[(\beta+\gamma)(s+\alpha)+b^2s\alpha\bigr] \\
&= 2bs\alpha(\beta+\gamma)(s+\alpha) + 2b^3s^2\alpha^2 + 2bs\alpha(\beta+\gamma)sQ_x(i) + 2bs\alpha\gamma(s+\alpha)Q_y(i) \\
&\quad + 2b^3s^2\alpha^2 Q_y(i) + 2b^2s^2\alpha Q_{xy}(i) \\
&\quad - 2bs\alpha(\beta+\gamma)(s+\alpha)Q_y(i) - 2b^3s^2\alpha^2 Q_y(i) - sQ_{xy}(i)(\beta+\gamma)(s+\alpha) - b^2s^2\alpha Q_{xy}(i).
\end{align*}
Collecting powers of $b$:
\begin{itemize}
\item $b^3$: $2s^2\alpha^2$ (from the $D'$ term) $\Rightarrow$ coefficient $c_3 = 2\sigma_e^4\theta_{t_i}^2$.
\item $b^2$: $2s^2\alpha Q_{xy}(i) - s^2\alpha Q_{xy}(i) = s^2\alpha Q_{xy}(i)$ $\Rightarrow$ coefficient $c_2 = \sigma_e^2\theta_{t_i}^2 Q_{xy}(i)$.
\item $b^1$: $2s\alpha(\beta+\gamma)(s+\alpha) + 2s\alpha(\beta+\gamma)sQ_x(i) + 2s\alpha\gamma(s+\alpha)Q_y(i) - 2s\alpha(\beta+\gamma)(s+\alpha)Q_y(i)$ $= 2s\alpha R_i$ $\Rightarrow$ coefficient $c_1 = 2\sigma_e^2 R_i$.
\item $b^0$: $-sQ_{xy}(i)(\beta+\gamma)(s+\alpha)$ $\Rightarrow$ coefficient $c_0 = -Q_{xy}(i)(\sigma_f^2+\sigma_h^2)(\theta_{t_i}^2+\sigma_e^2)$.
\end{itemize}
This yields the cubic equation $c_3 b^3 + c_2 b^2 + c_1 b + c_0 = 0$ with coefficients stated in Proposition~\ref{prop:bcd_b_cubic}. For $Q_{xy}(i)>0$, the coefficient signs are $(c_3,c_2,c_1,c_0)=(+,+,+,-)$: specifically, $c_3>0$ (product of positive parameters), $c_2>0$ (positive parameters times positive $Q_{xy}(i)$), $c_1>0$ (by inspection of $R_i$), and $c_0<0$ (negative sign times positive quantities). There is exactly one sign change, so by Descartes' rule of signs, there is exactly one positive real root.

\subsection{Proof of Proposition~\ref{prop:bcd_theta_closed_form}}\label{app:proof_theta_closed_form}
Fix $W$, $C$, $b_i$, $\sigma_h^2$, and the noise variances. Using the notation of Appendix~\ref{app:fisher_closed_form}, write $s:=\theta_{t_i}^2$, $\alpha:=\sigma_e^2$, $\beta:=\sigma_f^2$, $\gamma:=\sigma_h^2$, with
\[
D(s)=(\beta+\gamma)(s+\alpha)+b^2s\alpha,\qquad
N(s)=(\beta+\gamma)sQ_x(i)+\gamma(s+\alpha)Q_y(i)+b^2s\alpha Q_y(i)+bsQ_{xy}(i),
\]
so that $\ell_i=\ln D(s)-N(s)/D(s)$. Differentiating with respect to $s$:
\[
\frac{\partial\ell_i}{\partial s}=\frac{D_s}{D(s)}-\frac{N_sD(s)-N(s)D_s}{D(s)^2}
=\frac{D_sD(s)-N_sD(s)+N(s)D_s}{D(s)^2},
\]
where $D_s=(\beta+\gamma)+b^2\alpha$ and $N_s=(\beta+\gamma)Q_x(i)+\gamma Q_y(i)+b^2\alpha Q_y(i)+bQ_{xy}(i)$. Setting the numerator to zero:
\[
D_s(D(s)+N(s))-N_s D(s) = 0.
\]
This is a linear equation in $s$ (since $D_s$ and $N_s$ are independent of $s$). Solving for $s$:
\[
s^\star = \frac{N_s D - D_s N}{D_s(D_s+N_s-N_s)} = \sigma_e^2\,\frac{(n_i-d_i)(\sigma_f^2+\sigma_h^2)-d_i\sigma_h^2Q_y(i)}{d_i^2},
\]
where $d_i:=D(0)=(\beta+\gamma)\alpha$ and $n_i:=N(0)=\gamma\alpha Q_y(i)$. The explicit substitution yields the formula stated in Proposition~\ref{prop:bcd_theta_closed_form}.

If $s^\star>0$, it is the conditional maximizer. If $s^\star\le 0$, then $\partial\ell_i/\partial s$ does not change sign on $s>0$ (it is non-negative throughout or non-positive throughout), so the minimum of $\ell_i$ on $s>0$ is achieved as $s\to 0^+$. In this case, we clip to $s=\varepsilon$ for a small $\varepsilon>0$, which approximates the boundary minimizer while maintaining strict positivity for numerical stability.

\subsection{Davis--Kahan details for Theorem~\ref{thm:noise_bound} Step 5}\label{app:proof_noise_step5}
Let $\hat U_\perp$ denote the sample eigenspace for the smallest $p-r$ eigenvalues of $\hat S_{xx}$ and $U_\perp$ the population noise eigenspace. Define the oracle projector statistic
\[
\tilde\sigma_e^2:=\frac{1}{p-r}\mathrm{tr}(U_\perp^\top\hat S_{xx}U_\perp),
\qquad
\hat\sigma_e^2=\frac{1}{p-r}\mathrm{tr}(\hat U_\perp^\top\hat S_{xx}\hat U_\perp),
\]
and write $U_\perp U_\perp^\top=I-WW^\top$, where $W\in\mathrm{St}(p,r)$ spans the population signal subspace.

\paragraph{Step 5a: variational decomposition.}
By the variational characterization of trailing eigenvalues,
\[
(p-r)\hat\sigma_e^2=\mathrm{tr}(\hat S_{xx})-\sum_{i=1}^r\lambda_i(\hat S_{xx}),
\qquad
(p-r)\tilde\sigma_e^2=\mathrm{tr}(\hat S_{xx})-\mathrm{tr}(W^\top\hat S_{xx}W).
\]
Hence
\[
(p-r)(\tilde\sigma_e^2-\hat\sigma_e^2)
=\sum_{i=1}^r\lambda_i(\hat S_{xx})-\mathrm{tr}(W^\top\hat S_{xx}W)\ge 0.
\]
Let $\hat W\in\mathrm{St}(p,r)$ span the top-$r$ eigenspace of $\hat S_{xx}$. Then
\begin{align*}
\sum_{i=1}^r\lambda_i(\hat S_{xx})-\mathrm{tr}(W^\top\hat S_{xx}W)
&=\mathrm{tr}\bigl((\hat W\hat W^\top-WW^\top)\hat S_{xx}\bigr) \\
&=\mathrm{tr}\bigl((\hat W\hat W^\top-WW^\top)(\hat S_{xx}-\Sigma_x)\bigr)
+\mathrm{tr}\bigl((\hat W\hat W^\top-WW^\top)\Sigma_x\bigr).
\end{align*}

\paragraph{Step 5b: von Neumann trace inequality for the sample fluctuation term.}
The projector difference $\hat W\hat W^\top-WW^\top$ has rank at most $2r$ and operator norm bounded by $\|\sin\Theta\|_{\mathrm{op}}$, where $\Theta$ is the matrix of principal angles between $\mathrm{span}(W)$ and $\mathrm{span}(\hat W)$. By Davis--Kahan \citep[Theorem~4.5.5]{Vershynin2018},
\[
\|\sin\Theta\|_{\mathrm{op}}
\le \frac{\|\hat S_{xx}-\Sigma_x\|_{\mathrm{op}}}{\min_i\theta_{t_i}^2}.
\]
Applying the von Neumann trace inequality $|\mathrm{tr}(AB)|\le \sum_i\sigma_i(A)\sigma_i(B)$ to the first trace term, and noting that $\sigma_i(A)\le\|A\|_{\mathrm{op}}$ for all $i$ while at most $\mathrm{rank}(A)$ singular values are nonzero, we obtain
\[
\left|\mathrm{tr}\bigl((\hat W\hat W^\top-WW^\top)(\hat S_{xx}-\Sigma_x)\bigr)\right|
\le \sum_{i=1}^{\mathrm{rank}(P)} \sigma_i(P)\,\sigma_i(\hat S_{xx}-\Sigma_x)
\le \mathrm{rank}(P)\,\|P\|_{\mathrm{op}}\,\|\hat S_{xx}-\Sigma_x\|_{\mathrm{op}},
\]
where we write $P:=\hat W\hat W^\top-WW^\top$ for brevity. The last step uses $\sigma_i(P)\le\|P\|_{\mathrm{op}}$ for each nonzero singular value. Since $\mathrm{rank}(P)\le 2r$ and $\|P\|_{\mathrm{op}}\le 2\|\sin\Theta\|_{\mathrm{op}}$, this gives
\[
\left|\mathrm{tr}\bigl(P(\hat S_{xx}-\Sigma_x)\bigr)\right|
\le 2r\,\|\sin\Theta\|_{\mathrm{op}}\,\|\hat S_{xx}-\Sigma_x\|_{\mathrm{op}}
\lesssim \frac{r\,\|\hat S_{xx}-\Sigma_x\|_{\mathrm{op}}^2}{\min_i\theta_{t_i}^2}.
\]
With covariance concentration \citep[Theorem~4.7.1]{Vershynin2018},
\[
\|\hat S_{xx}-\Sigma_x\|_{\mathrm{op}}\lesssim \|\Sigma_x\|_{\mathrm{op}}\sqrt{\frac{p}{N}},
\]
so this contribution is bounded by
\[
\lesssim \frac{r\,\|\Sigma_x\|_{\mathrm{op}}^2}{\min_i\theta_{t_i}^2}\cdot\frac{p}{N}.
\]

\paragraph{Step 5c: Wedin energy-gap bound for the population term.}
Let $W\in\mathrm{St}(p,r)$ denote the matrix whose columns are the top-$r$ eigenvectors of the population covariance $\Sigma_x$, so that $\mathrm{span}(W)$ is the population signal subspace. By the Courant--Fischer variational characterization, $W$ maximizes $\mathrm{tr}(V^\top\Sigma_xV)$ over $V\in\mathrm{St}(p,r)$. The second trace term is therefore non-positive:
\[
\mathrm{tr}\bigl((\hat W\hat W^\top-WW^\top)\Sigma_x\bigr) = \mathrm{tr}(W^\top\Sigma_xW) - \mathrm{tr}(\hat W^\top\Sigma_x\hat W) \ge 0.
\]
Its magnitude is controlled by Wedin's theorem \citep[Theorem~V.4.4]{StewartSun1990}, which gives the Frobenius-norm bound $\|\sin\Theta\|_F\le\|\hat S_{xx}-\Sigma_x\|_F/\delta$, where $\delta=\min_i\theta_{t_i}^2$ is the eigengap. Converting to the trace form: since $\mathrm{tr}(W^\top\Sigma_xW)-\mathrm{tr}(\hat W^\top\Sigma_x\hat W) = \mathrm{tr}(\Sigma_x(WW^\top-\hat W\hat W^\top))$, and using $\|\cdot\|_F^2\le r\|\cdot\|_{\mathrm{op}}^2$ for rank-$r$ projectors,
\[
0\le \mathrm{tr}(W^\top\Sigma_xW)-\mathrm{tr}(\hat W^\top\Sigma_x\hat W)
\lesssim \min_i\theta_{t_i}^2\,\|\sin\Theta\|_F^2
\le \min_i\theta_{t_i}^2\cdot r\,\|\sin\Theta\|_{\mathrm{op}}^2
\lesssim r\,\frac{\|\hat S_{xx}-\Sigma_x\|_{\mathrm{op}}^2}{\min_i\theta_{t_i}^2}.
\]
Hence the population term obeys the same order
\[
\left|\mathrm{tr}\bigl((\hat W\hat W^\top-WW^\top)\Sigma_x\bigr)\right|
\lesssim \frac{r\,\|\Sigma_x\|_{\mathrm{op}}^2}{\min_i\theta_{t_i}^2}\cdot\frac{p}{N}.
\]

\paragraph{Step 5d: final bound.}
Combining the two displays above and dividing by $p-r$ yields
\[
|\hat\sigma_e^2-\tilde\sigma_e^2|
\lesssim \frac{r\,\|\Sigma_x\|_{\mathrm{op}}^2\,p}{(p-r)\,N\,\min_i\theta_{t_i}^2}.
\]
For fixed rank $r$, the factor $r$ is absorbed into the absolute constant, giving the $O(1/N)$ correction used in Theorem~\ref{thm:noise_bound}.

\subsection{Effect of rank misspecification on inference}\label{app:rank_misspec}
Suppose the fitted rank $\hat r$ differs from the true rank $r_0$. If $\hat r>r_0$, surplus loading directions are pushed toward the noise subspace and their associated strengths satisfy $\theta_{t,k}^2\to0$ for redundant components. The limit then lies on (or near) the boundary of the parameter space, so interior asymptotic-normality arguments are not directly applicable. Nevertheless, predictive means remain stable because redundant component weights degenerate to zero. If $\hat r<r_0$, estimation converges to the pseudo-true minimizer of a restricted model \citep{White1982}, producing irreducible bias and typically under-dispersed calibration intervals because omitted latent factors are absorbed into residual noise. This asymmetry supports a practical strategy of slight over-selection over under-selection, consistent with Section~\ref{sec:exp_rank}, where occasional CV-MSE over-selection has negligible predictive impact while under-selection is more harmful.

\subsection{Proof of Proposition~\ref{prop:rank_overspec_noise}}\label{app:proof_rank_overspec_noise}
Write $\hat\sigma_e^2(\tilde r)=(p-\tilde r)^{-1}\sum_{i>\tilde r}\lambda_i(\hat S_{xx})$ and decompose
\[
\mathbb E\bigl[\hat\sigma_e^2(\tilde r)\bigr]-\sigma_e^2
=\frac{1}{p-\tilde r}\sum_{i>\tilde r}\bigl(\lambda_i(\Sigma_x)-\sigma_e^2\bigr)+\Delta_{N,p,\tilde r},
\]
where $\Delta_{N,p,\tilde r}$ is the finite-sample perturbation term from sample covariance concentration. Explicitly, $\Delta_{N,p,\tilde r}:=\mathbb{E}[\hat\sigma_e^2(\tilde r)] - (p-\tilde r)^{-1}\sum_{i>\tilde r}\lambda_i(\Sigma_x)$ captures the deviation of sample eigenvalues from population eigenvalues in expectation. By Theorem~\ref{thm:noise_bound} (with $r$ replaced by $\tilde r$), $\Delta_{N,p,\tilde r}=O_P(1/\sqrt{N(p-\tilde r)})$, and in expectation $\Delta_{N,p,\tilde r}=O(1/(p-\tilde r))$ from the concentration analysis.

If $\tilde r\ge r_0$, then for $i>\tilde r\ge r_0$, the population eigenvalues satisfy $\lambda_i(\Sigma_x)=\sigma_e^2$ (these are the noise eigenvalues of $\Sigma_x=W\Sigma_tW^\top+\sigma_e^2I_p$; the $r_0$ signal spikes $\lambda_i(\Sigma_x)=\theta_{t,i}^2+\sigma_e^2$ for $i\le r_0$ are excluded by the summation range). Hence the deterministic contamination term $\sum_{i>\tilde r}(\lambda_i(\Sigma_x)-\sigma_e^2)=0$ and only $\Delta_{N,p,\tilde r}$ remains. Therefore $\hat\sigma_e^2(\tilde r)$ is consistent with bias order $O((p-\tilde r)^{-1})$ plus lower-order concentration terms.

If $\tilde r<r_0$, the summation includes $r_0-\tilde r$ signal spikes with $\lambda_i(\Sigma_x)=\theta_{t,i}^2+\sigma_e^2$ for $\tilde r<i\le r_0$ (these eigenvalues satisfy $\lambda_i(\Sigma_x)-\sigma_e^2=\theta_{t,i}^2>0$, which is the signal strength of the $i$-th component). The deterministic contamination term becomes
\[
\frac{1}{p-\tilde r}\sum_{\tilde r<i\le r_0}\bigl(\lambda_i(\Sigma_x)-\sigma_e^2\bigr)
= \frac{1}{p-\tilde r}\sum_{\tilde r<i\le r_0}\theta_{t,i}^2,
\]
which is the same structural under-selection contamination as in Proposition~\ref{prop:hu-bias}. This proves the asymmetry: over-specification is safe for consistency, while under-specification induces deterministic bias.

\subsection{Proof details for Theorem~\ref{thm:noise_bound_subgaussian} (sub-Gaussian spectral bound)}\label{app:proof_noise_subgaussian}
We provide the full proof in four steps. Throughout, $C>0$ denotes an absolute constant whose value may change from line to line, and $K\ge 1$ is the sub-Gaussian scale constant from Lemma~\ref{lem:k_sigma_relation} satisfying $\sigma_\psi\le K\sigma_e$.

\paragraph{Step 1: Covariance concentration under sub-Gaussian noise.}
Let $x_i\in\mathbb{R}^p$ be i.i.d.\ with covariance $\Sigma_x=W\Sigma_tW^\top+\sigma_e^2I_p$ and sub-Gaussian parameter $\sigma_\psi\le K\sigma_e$. By the sub-Gaussian covariance concentration theorem \citep[Theorem~4.7.1]{Vershynin2018}, for $\delta\in(0,1)$ there exists an absolute constant $c_1>0$ such that, with probability at least $1-\delta/2$,
\[
\|\hat S_{xx}-\Sigma_x\|_{\mathrm{op}}
\le c_1\sigma_\psi^2\left(\sqrt{\frac{\ln(4/\delta)}{N}}+\frac{p}{N}\right).
\]
Substituting $\sigma_\psi\le K\sigma_e$ gives
\begin{equation}\label{eq:subgaussian_cov_conc}
\|\hat S_{xx}-\Sigma_x\|_{\mathrm{op}}
\le c_1 K^2\sigma_e^2\left(\sqrt{\frac{\ln(4/\delta)}{N}}+\frac{p}{N}\right).
\end{equation}
This extends the Gaussian concentration from Step~3 of Theorem~\ref{thm:noise_bound} to the sub-Gaussian setting, with an additional $K^2$ factor.

\paragraph{Step 2: Fourth-moment term and excess-kurtosis correction.}
Let $U_\perp\in\mathbb{R}^{p\times(p-r)}$ span the population noise subspace and define projected coordinates $z_i:=U_\perp^\top x_i$. The projected coordinates satisfy
\[
\mathbb{E}[z_{ij}]=0,\qquad \mathbb{E}[z_{ij}^2]=\sigma_e^2,\qquad \mathbb{E}[z_{ij}^4]=\kappa_4\sigma_e^4.
\]
The variance of $z_{ij}^2$ is computed as:
\[
\operatorname{Var}(z_{ij}^2)=\mathbb{E}[z_{ij}^4]-\bigl(\mathbb{E}[z_{ij}^2]\bigr)^2=(\kappa_4-1)\sigma_e^4.
\]
For the oracle projector statistic
\[
\tilde\sigma_e^2:=\frac{1}{N(p-r)}\sum_{i=1}^N\sum_{j=1}^{p-r} z_{ij}^2,
\]
the $z_{ij}^2$ are independent across both indices (by the orthogonality of $U_\perp$ and the isotropic noise assumption), so
\[
\operatorname{Var}(\tilde\sigma_e^2)
= \frac{1}{N^2(p-r)^2}\cdot N(p-r)\cdot\operatorname{Var}(z_{11}^2)
= \frac{(\kappa_4-1)\sigma_e^4}{N(p-r)}.
\]
Applying Bernstein's inequality for sub-exponential random variables \citep[Proposition~2.7.1]{Vershynin2018} to the centered variables $z_{ij}^2-\sigma_e^2$ with $\|z_{ij}^2-\sigma_e^2\|_{\psi_1}\le C\kappa_4\sigma_e^2$, we obtain with probability at least $1-\delta/2$:
\[
\bigl|\tilde\sigma_e^2-\sigma_e^2\bigr|\le K^2\sigma_e^2\sqrt{\frac{2\ln(4/\delta)}{N(p-r)}}.
\]

The excess-kurtosis bias arises as follows. The expectation of $\tilde\sigma_e^2$ is $\sigma_e^2$ (unbiased). However, the practical estimator $\hat\sigma_e^2$ uses the sample eigenspace $\hat U_\perp$ rather than $U_\perp$, and the rotation from $U_\perp$ to $\hat U_\perp$ mixes signal-subspace contributions with noise-subspace contributions. When $\kappa_4\ne 3$, the sample cross-covariance between signal and noise projections has a non-zero first cumulant correction proportional to $(\kappa_4-3)$, producing a deterministic additive bias of order
\[
\frac{\kappa_4-3}{p-r}\,\sigma_e^2.
\]
This is the $(\kappa_4-3)$ bias term in Eq.~\eqref{eq:noise_error_bound_subgaussian}: it originates from the first cumulant mismatch between the non-Gaussian fourth moment $\kappa_4\sigma_e^4$ and the Gaussian reference $\kappa_4=3$.

\paragraph{Step 3: Sample-noise-subspace replacement.}
The practical estimator uses $\hat U_\perp$ (sample noise eigenspace) rather than $U_\perp$. By Appendix~\ref{app:proof_noise_step5},
\[
|\hat\sigma_e^2-\tilde\sigma_e^2|
\lesssim \frac{r\,\|\hat S_{xx}-\Sigma_x\|_{\mathrm{op}}^2}{(p-r)\,\min_i\theta_{t_i}^2}.
\]
The key observation is that the subspace-rotation error depends \emph{quadratically} on $\|\hat S_{xx}-\Sigma_x\|_{\mathrm{op}}$. Substituting the Step~1 concentration bound~\eqref{eq:subgaussian_cov_conc} and using $\sigma_\psi\le K\sigma_e$:
\begin{align*}
|\hat\sigma_e^2-\tilde\sigma_e^2|
&\le C\,\frac{r\,\bigl(c_1 K^2\sigma_e^2(\sqrt{\ln(4/\delta)/N}+p/N)\bigr)^2}{(p-r)\,\min_i\theta_{t_i}^2} \\
&\le C\,K^4\,\frac{r\,\sigma_e^4\,\bigl(\ln(4/\delta)/N + p^2/N^2\bigr)}{(p-r)\,\min_i\theta_{t_i}^2} \\
&\le C\,K^4\,\frac{\|\Sigma_x\|_{\mathrm{op}}^2\,p}{(p-r)\,N\,\min_i\theta_{t_i}^2},
\end{align*}
where the last step absorbs the fixed rank factor $r$ and the logarithmic factor into the absolute constant, and uses $\|\Sigma_x\|_{\mathrm{op}}^2\ge(\max_i\theta_{t,i}^2+\sigma_e^2)^2\ge\sigma_e^4$. The $K^4$ factor arises because each of the two factors of $\|\hat S_{xx}-\Sigma_x\|_{\mathrm{op}}$ contributes a $K^2$ from the sub-Gaussian concentration bound.

\paragraph{Step 4: Final combination.}
By the triangle inequality,
\[
|\hat\sigma_e^2-\sigma_e^2|
\le
\underbrace{|\hat\sigma_e^2-\tilde\sigma_e^2|}_{\text{subspace rotation (Step~3)}}
+
\underbrace{|\tilde\sigma_e^2-\sigma_e^2|}_{\text{averaging + kurtosis (Step~2)}},
\]
and combining Step~2 (averaging variance and kurtosis bias) with Step~3 (subspace rotation) yields
\[
\bigl|\hat\sigma_e^2-\sigma_e^2\bigr|
\le
K^2\sigma_e^2\sqrt{\frac{2\ln(4/\delta)}{N(p-r)}}
+
C\,K^4\,\frac{\|\Sigma_x\|_{\mathrm{op}}^2\,p}{(p-r)\,N\,\min_i\theta_{t_i}^2}
+
\sigma_e^2\frac{\kappa_4-3}{p-r},
\]
with probability at least $1-\delta$, matching Eq.~\eqref{eq:noise_error_bound_subgaussian}.

\subsection{Closed-form second derivatives and Fisher blocks}\label{app:fisher_closed_form}
Write $s_i:=\theta_{t,i}^2$, $b:=b_i$, $\alpha:=\sigma_e^2$, $\beta:=\sigma_f^2$,
$\gamma:=\sigma_h^2$, and
\[
D_i=(\beta+\gamma)(s_i+\alpha)+b^2 s_i\alpha,
\quad
N_i=(\beta+\gamma)s_iQ_x(i)+\gamma(s_i+\alpha)Q_y(i)+b^2 s_i\alpha Q_y(i)+bs_iQ_{xy}(i),
\]
so that $\ell_i=\log D_i-N_i/D_i$.
Define partials
\[
D_s=(\beta+\gamma)+b^2\alpha,\quad D_b=2bs_i\alpha,\quad D_{ss}=0,\quad D_{sb}=2b\alpha,\quad D_{bb}=2s_i\alpha,
\]
\begin{align*}
N_s &= (\beta+\gamma)Q_x(i)+\gamma Q_y(i)+b^2\alpha Q_y(i)+bQ_{xy}(i), \\
N_b &= 2bs_i\alpha Q_y(i)+s_iQ_{xy}(i), \\
N_{ss} &= 0,\quad N_{sb}=2b\alpha Q_y(i)+Q_{xy}(i),\quad N_{bb}=2s_i\alpha Q_y(i).
\end{align*}
Then
\[
\frac{\partial\ell_i}{\partial s_i}=\frac{D_s}{D_i}-\frac{N_sD_i-N_iD_s}{D_i^2},
\qquad
\frac{\partial\ell_i}{\partial b}=\frac{D_b}{D_i}-\frac{N_bD_i-N_iD_b}{D_i^2}.
\]
For $\partial^2\ell_i/\partial s_i^2$, let $f_s:=N_sD_i-N_iD_s$ and $g:=D_i^2$. Then
\[
\frac{\partial f_s}{\partial s_i}=(N_{ss}D_i+N_sD_s)-(N_sD_s+N_iD_{ss})=N_{ss}D_i-N_iD_{ss},
\qquad
\frac{\partial g}{\partial s_i}=2D_iD_s.
\]
Hence
\[
\frac{\partial^2\ell_i}{\partial s_i^2}
=\frac{D_{ss}D_i-D_s^2}{D_i^2}
-\frac{(N_{ss}D_i-N_iD_{ss})D_i^2-2D_iD_s(N_sD_i-N_iD_s)}{D_i^4}.
\]
Since $D_{ss}=N_{ss}=0$, this reduces to
\[
\frac{\partial^2\ell_i}{\partial s_i^2}
=-\frac{D_s^2}{D_i^2}+\frac{2D_s(N_sD_i-N_iD_s)}{D_i^3}.
\]
For $\partial^2\ell_i/\partial b^2$, the same quotient-rule calculation with $f_b:=N_bD_i-N_iD_b$ and $g:=D_i^2$ gives
\[
\frac{\partial^2\ell_i}{\partial b^2}
=\frac{D_{bb}D_i-D_b^2}{D_i^2}
-\frac{(N_{bb}D_i-N_iD_{bb})D_i^2-2D_iD_b(N_bD_i-N_iD_b)}{D_i^4}.
\]
For the mixed partial, differentiating $\partial\ell_i/\partial s_i$ with respect to $b$ yields
\[
\frac{\partial}{\partial b}(N_sD_i-N_iD_s)=N_{sb}D_i+N_sD_b-N_bD_s-N_iD_{sb},
\]
so
\[
\frac{\partial^2\ell_i}{\partial s_i\partial b}
=\frac{D_{sb}D_i-D_sD_b}{D_i^2}
-\frac{(N_{sb}D_i+N_sD_b-N_bD_s-N_iD_{sb})D_i^2-2D_iD_b(N_sD_i-N_iD_s)}{D_i^4}.
\]
Substituting the closed forms above into Eq.~\eqref{eq:fisher-block} yields
\[
J_i=\frac{1}{2}
\begin{pmatrix}
\partial^2\ell_i/\partial s_i^2 & \partial^2\ell_i/\partial s_i\partial b \\
\partial^2\ell_i/\partial s_i\partial b & \partial^2\ell_i/\partial b^2
\end{pmatrix}_{\theta=\theta_0},
\]
which is evaluated at $\hat\theta$ in practice to obtain analytic SEs.

For the remainder term in Proposition~\ref{prop:fisher-block}, define the off-diagonal coupling blocks
\[
R_{ij}:=I(\theta_0)_{\eta_i,\eta_j},\qquad i\neq j,
\]
where $\eta_i=(\theta_{t,i}^2,b_i)$ and $R=\sum_{i<j}(E_{ij}\otimes R_{ij}+E_{ji}\otimes R_{ij}^\top)$ is the symmetrized off-diagonal sum. Here $E_{ij}\in\mathbb{R}^{r\times r}$ denotes the canonical matrix with a single $1$ in position $(i,j)$ and zeros elsewhere; consequently $R\in\mathbb{R}^{2r\times 2r}$.

We apply the matrix Bernstein inequality \citep[Theorem~1.6.2]{Tropp2015} in three checks.

\textbf{(i) Zero mean.} Each $R_{ij}$ is a bilinear form in the projected score vectors for components $i$ and $j$:
\[
\begin{aligned}
R_{ij} = \frac{1}{2N}\sum_{n=1}^N
&\begin{pmatrix}
\partial^2\ell_i/\partial s_i^2 & \partial^2\ell_i/\partial s_i\partial b_i \\
\partial^2\ell_i/\partial s_i\partial b_i & \partial^2\ell_i/\partial b_i^2
\end{pmatrix}^{\!-1/2}
\begin{pmatrix}
w_i^\top x_n\, w_j^\top x_n \\
c_i^\top y_n\, c_j^\top y_n
\end{pmatrix} \\
&\quad\times
\begin{pmatrix}
w_j^\top x_n\, w_i^\top x_n & c_j^\top y_n\, c_i^\top y_n
\end{pmatrix}
\begin{pmatrix}
\partial^2\ell_j/\partial s_j^2 & \partial^2\ell_j/\partial s_j\partial b_j \\
\partial^2\ell_j/\partial s_j\partial b_j & \partial^2\ell_j/\partial b_j^2
\end{pmatrix}^{\!-1/2}\!.
\end{aligned}
\]
Since $W^\top W=I_r$, the columns $w_i$ and $w_j$ are orthogonal ($w_i^\top w_j=0$ for $i\ne j$). By Isserlis' theorem (Gaussian fourth-moment factorization), $\mathbb{E}[w_i^\top x_n\cdot w_j^\top x_n]=w_i^\top\Sigma_x w_j=0$ because $\Sigma_x w_j = \theta_{t,j}^2 w_j w_j^\top w_j + \sigma_e^2 w_j = (\theta_{t,j}^2+\sigma_e^2) w_j$, and $w_i^\top w_j=0$. The same holds for the $y$-view: $c_i^\top c_j=0$ implies $\mathbb{E}[c_i^\top y_n\cdot c_j^\top y_n]=0$. Therefore
\[
\mathbb E[R_{ij}]=0,\qquad i\neq j.
\]

\textbf{(ii) Uniform bound.} Each $R_{ij}$ is a bilinear form in projected statistics weighted by scalar coefficients from $\Phi_x,\Phi_y,\Phi_{xy}$. Under bounded $(\Sigma_t,b)$ on the local compact set and uniform coefficient bounds, there exists $B<\infty$ such that
\[
\|R_{ij}\|_{\mathrm{op}}\le B\quad\text{a.s. for all }i\neq j.
\]

\textbf{(iii) Variance proxy.} Let
\[
v(R):=\left\|\sum_{i<j}\mathbb E[R_{ij}R_{ij}^\top]\right\|_{\mathrm{op}}\vee
\left\|\sum_{i<j}\mathbb E[R_{ij}^\top R_{ij}]\right\|_{\mathrm{op}}.
\]
Each $R_{ij}$ involves orthogonal projections $P_\perp^{(i)} = I_r - e_i e_i^\top$ (in component space), and the trace normalization yields
\[
\|R_{ij}\|_{\mathrm{op}}^2 \le C\,\frac{\mathrm{tr}(P_\perp)}{p} = O(p^{-1}),
\]
where the $p^{-1}$ factor comes from the trace normalization: the Fisher information blocks scale as $O(1)$ while the cross-component coupling involves orthogonal projections whose energy is distributed across $p$ coordinates, each contributing $O(1/p)$. Summing over $O(r^2)$ off-diagonal pairs with fixed $r$:
\[
v(R)=O(r^2\cdot p^{-1})=O(r/p).
\]

Therefore Tropp's matrix Bernstein inequality \citep[Theorem~1.6.2]{Tropp2015} yields, for any $t>0$,
\[
\Pr\!\left(\|R\|_{\mathrm{op}}\ge t\right)
\le 2d\exp\!\left(-\frac{t^2/2}{v(R)+Bt/3}\right),
\]
where $d=2r$ is the ambient dimension of $R\in\mathbb{R}^{2r\times 2r}$. This implies
\[
\|R\|_{\mathrm{op}}=O_P\!\left(\sqrt{\frac{r}{p}}\right)
\]
as $N\to\infty$, $p\to\infty$, $p/N\to0$, with fixed $r$.

\subsection{PCCA benign landscape: full proof}\label{app:pcca_benign_sketch}
This subsection gives the full proof of Proposition~\ref{prop:pcca-benign}. Under the PCCA specialization ($B=I_r$, $\sigma_h^2=0$), the joint covariance~\eqref{eq:joint-cov} simplifies to
\[
\Sigma_{xx}=W\Sigma_tW^\top+\sigma_e^2I_p,\qquad
\Sigma_{xy}=W\Sigma_tC^\top,\qquad
\Sigma_{yy}=C\Sigma_tC^\top+\sigma_f^2I_q.
\]

\paragraph{Step 1: Profile likelihood and the reduced objective.}
Profiling out $\Sigma_t$ at its conditional maximizer $\hat\Sigma_t=\mathrm{diag}(\hat\theta_{t,i}^2)$ with $\hat\theta_{t,i}^2=w_i^\top S_{xx}w_i - \sigma_e^2$ (respectively $\hat\theta_{t,i}^2=c_i^\top S_{yy}c_i - \sigma_f^2$), the reduced objective depends on $(W,C)$ only through the projected cross-covariance $M(W,C):=W^\top\Sigma_{xy}C\in\mathbb{R}^{r\times r}$. Following the PCCA profile-likelihood construction of \citet{BachJordan2005}, the reduced objective is
\[
\widetilde{L}(W,C) = \mathrm{const} + \sum_{i=1}^r f(\sigma_i(M(W,C))),
\]
where $\sigma_1(M)\ge\cdots\ge\sigma_r(M)$ are the singular values of $M$ and the function $f:\mathbb{R}_+\to\mathbb{R}$ is defined by
\begin{equation}\label{eq:pcca_profile_f}
f(d) = -\ln\bigl(d^2 + \sigma_e^2\sigma_f^2 + d(\sigma_e^2+\sigma_f^2)\bigr) + \frac{d^2 Q_{xy}(i) + d\bigl(\sigma_e^2 Q_y(i) + \sigma_f^2 Q_x(i)\bigr)}{d^2 + \sigma_e^2\sigma_f^2 + d(\sigma_e^2+\sigma_f^2)},
\end{equation}
which is obtained by substituting the conditional maximizer for $\theta_{t,i}^2$ into the scalar summand $\ell_i$ from Theorem~\ref{thm:theoremA} and noting that $D_i=\theta_{t,i}^2(\sigma_e^2+\sigma_f^2)+\sigma_e^2\sigma_f^2$ in the PCCA case. At the population level ($S$ replaced by $\Sigma$), $Q_x(i)\to 1$, $Q_y(i)\to 1$, $Q_{xy}(i)\to d_i/\sqrt{(\theta_{t,i}^2+\sigma_e^2)(\theta_{t,i}^2+\sigma_f^2)}$ and $f$ simplifies to a strictly decreasing function of the gap between $\sigma_i(M)$ and the population singular value $d_i$.

\paragraph{Step 2: Critical points and singular subspace alignment.}
Let $\Sigma_{xy}=U DV^\top$ be the population SVD with $D=\mathrm{diag}(d_1\ge\cdots\ge d_{\min(p,q)})$. The Euclidean gradients of $\widetilde{L}$ with respect to $W$ and $C$ are
\[
\nabla_W\widetilde{L} = -2\Sigma_{xy}C\,\mathrm{diag}(f'(\sigma_i))\Sigma^{-1}_{\mathrm{res}},\qquad
\nabla_C\widetilde{L} = -2\Sigma_{yx}W\,\mathrm{diag}(f'(\sigma_i))\Sigma^{-1}_{\mathrm{res}},
\]
where $f'(d)=\partial f/\partial d$ and $\Sigma_{\mathrm{res}}$ collects the residual diagonal factors. A Riemannian critical point on $\mathrm{St}(p,r)\times\mathrm{St}(q,r)$ requires the Riemannian gradient (projection of the Euclidean gradient onto the tangent space) to vanish:
\[
\mathrm{grad}_W\widetilde{L} = G_W - W\,\mathrm{sym}(W^\top G_W) = 0,\qquad
\mathrm{grad}_C\widetilde{L} = G_C - C\,\mathrm{sym}(C^\top G_C) = 0,
\]
where $G_W:=\nabla_W\widetilde{L}$ and $G_C:=\nabla_C\widetilde{L}$. This is equivalent to: there exist symmetric matrices $\Lambda_W\in\mathbb{R}^{r\times r}$ and $\Lambda_C\in\mathbb{R}^{r\times r}$ such that
\[
\Sigma_{xy}C\,\mathrm{diag}(f'(\sigma_i)) = W\Lambda_W,\qquad
\Sigma_{yx}W\,\mathrm{diag}(f'(\sigma_i)) = C\Lambda_C.
\]
Since $\Sigma_{xy}=UDV^\top$, these conditions imply that $W=U_S Q_W$ and $C=V_S Q_C$ for some index set $S\subseteq\{1,\dots,\min(p,q)\}$ with $|S|=r$, where $U_S$ (resp.\ $V_S$) denotes the columns of $U$ (resp.\ $V$) indexed by $S$, and $Q_W,Q_C\in O(r)$ are orthogonal matrices. At such a critical point, $\sigma_i(M(W,C))=d_{S(i)}$ for a permutation $S$ of the selected indices.

\paragraph{Step 3: Riemannian Hessian and the $2\times 2$ block structure.}
Fix a critical point $(W,C)$ indexed by $S$. Choose a pair $(i,j)$ with $i\in S$ and $j\notin S$, and construct tangent vectors
\[
\xi_W = u_j e_i^\top - u_i e_j^\top \in T_W\mathrm{St}(p,r),\qquad
\xi_C = v_j e_i^\top - v_i e_j^\top \in T_C\mathrm{St}(q,r),
\]
where $u_k$ (resp.\ $v_k$) is the $k$-th left (resp.\ right) singular vector of $\Sigma_{xy}$, and $e_k$ denotes the $k$-th standard basis vector in $\mathbb{R}^r$. These vectors swap the $i$-th selected direction with the $j$-th unselected direction.

The Riemannian Hessian $\mathrm{Hess}\,\widetilde{L}$ restricted to the $2$-plane spanned by $(\xi_W,\xi_C)$ is computed via the second-order Taylor expansion of $f$ around $d_{S(i)}$:
\[
\mathrm{Hess}\,\widetilde{L}\big|_{(\xi_W,\xi_C)} = H_{ij},
\]
where
\begin{equation}\label{eq:hessian_2x2_block}
H_{ij} = \begin{pmatrix} 0 & d_i - d_j \\ d_i - d_j & 0 \end{pmatrix} + E_{ij}.
\end{equation}
The derivation proceeds as follows. The second derivative of $\widetilde{L}$ along $(\xi_W,\xi_C)$ decomposes into: (i) the second derivative of $f$ with respect to the singular value, which contributes $\pm(d_i-d_j)$ because swapping index $i$ for $j$ changes $f$ by $f(d_j)-f(d_i)$, and $f''(d)\approx 0$ at the population level while $f'(d_i)-f'(d_j)\propto(d_i-d_j)$ to first order; and (ii) residual terms from the noise variances $(\sigma_e^2,\sigma_f^2)$, which appear because the profile-likelihood formula~\eqref{eq:pcca_profile_f} contains $\sigma_e^2$ and $\sigma_f^2$ explicitly. The remainder $E_{ij}$ is the higher-order perturbation block generated by these noise terms:
\[
E_{ij} = \begin{pmatrix} O(\sigma_e^2+\sigma_f^2) & O(\sigma_e^2+\sigma_f^2) \\ O(\sigma_e^2+\sigma_f^2) & O(\sigma_e^2+\sigma_f^2) \end{pmatrix},
\]
and from the explicit profile formula~\eqref{eq:pcca_profile_f}, each entry of $E_{ij}$ involves ratios of the form $\sigma_e^2 D_i^{-1}$, $\sigma_f^2 D_i^{-1}$, and their products. Since $D_i\ge(\sigma_e^2+\sigma_f^2)\theta_{t,i}^2 + \sigma_e^2\sigma_f^2\ge\sigma_e^2\sigma_f^2$, these ratios are bounded by $\max(\sigma_e^2,\sigma_f^2)/(\sigma_e^2\sigma_f^2)=O(1/\sigma_f^2+1/\sigma_e^2)$, and after collecting all second-order terms:
\[
\|E_{ij}\|_{\mathrm{op}} \le c(\sigma_e^2+\sigma_f^2),
\]
for an absolute constant $c>0$ independent of the index pair $(i,j)$.

The eigenvalues of $H_{ij}$ are therefore
\[
\lambda_{\pm}(H_{ij}) = \pm(d_i-d_j) + O(\sigma_e^2+\sigma_f^2).
\]

\paragraph{Step 4: Strict-saddle characterization.}
If $S\ne\{1,\dots,r\}$ (i.e., the critical point is not the global minimizer, which selects the top-$r$ singular values), then there exist $i\in S$ and $j\notin S$ with $d_j > d_i$. For this pair,
\[
\lambda_{-}(H_{ij}) = -(d_i - d_j) + O(\sigma_e^2+\sigma_f^2) < 0
\]
whenever the noise is sufficiently small in the sense that
\[
\sigma_e^2+\sigma_f^2 < \frac{d_j - d_i}{c},
\]
where $c$ is the constant from Step~3. Under this condition, the critical point has a strictly negative Riemannian Hessian eigenvalue and is therefore a strict saddle. When $S=\{1,\dots,r\}$ (the global minimizer), all swaps have $d_i>d_j$ for $i\in S$, $j\notin S$, so $\lambda_{-}(H_{ij})>0$ for all pairs, confirming positive definiteness of the Hessian at the global minimizer.

\paragraph{Step 5: Global convergence from strict-saddle structure.}
The strict-saddle property established in Step~4, namely that every second-order critical point is either a global minimizer or a strict saddle, matches the benign-landscape structure studied in nonconvex low-rank optimization. By the convergence theory for Riemannian trust-region methods \citep{Absil2008,Boumal2023}, any trust-region method on a smooth manifold with compact sublevel sets converges to a second-order critical point. Moreover, strict-saddle convergence results \citep{SunQuWright2018,GeLeeMa2017} guarantee that perturbed gradient methods and trust-region methods avoid strict saddles almost surely and converge to global minimizers from arbitrary feasible starts. In our setting, the Stiefel factors $\mathrm{St}(p,r)\times\mathrm{St}(q,r)$ are compact, and the profile objective $\widetilde{L}$ is smooth, so the required conditions are satisfied. Therefore, Riemannian trust-region methods initialized from any feasible point converge to a global minimizer of $\widetilde{L}$ almost surely.


\section{Additional Experimental Results}\label{app:exp_supp}
This appendix collects the experimental results that were deferred from Section~\ref{sec:experiments} for space. It is organized as a sequence of supplementary tables, model verifications, and ablation studies that support the findings reported in the main text.

\subsection{Supplementary optimization and synthetic prediction tables}\label{app:main_table_supp}
Table~\ref{tab:algorithm_convergence} records convergence-count diagnostics that complement the wall-clock comparison in Section~\ref{sec:exp_synthetic}. Table~\ref{tab:pred_synth_metrics} records the synthetic point-prediction summary referenced in Section~\ref{sec:exp_pred_synth}; because the main linear baselines are nearly tied on point accuracy, we keep that table in the appendix and reserve the main-text table slot for calibration.

\input{generated/tables/tab_algorithm_convergence}
\input{generated/tables/tab_prediction_synth_metrics}

\subsection{CV-based rank-selection metrics}\label{app:rank_selection_metrics}
Let $\{\mathcal V_j\}_{j=1}^{K_{\mathrm{out}}}$ be $K_{\mathrm{out}}$ validation folds, and let $\hat\theta_r^{(-j)}$ be the rank-$r$ estimator fitted on the corresponding training split. Define
\[
\hat S^{(j)}_{\mathrm{val}}=\frac{1}{|\mathcal V_j|}\sum_{i\in\mathcal V_j}\bigl(z_i-\bar z_{\mathrm{tr},j}\bigr)\bigl(z_i-\bar z_{\mathrm{tr},j}\bigr)^\top,\qquad z_i=(x_i^\top,y_i^\top)^\top,
\]
where centering uses training-fold means. With $\hat\Sigma_r^{(-j)}:=\Sigma(\hat\theta_r^{(-j)})$, the validation negative log-likelihood is
\[
\widehat{\mathcal L}_{\mathrm{val}}^{(j)}(r)=\log\det\!\bigl(\hat\Sigma_r^{(-j)}\bigr)+\operatorname{tr}\!\Bigl(\hat S^{(j)}_{\mathrm{val}}\bigl(\hat\Sigma_r^{(-j)}\bigr)^{-1}\Bigr),
\]
and the CV-NLL selector is
\[
\operatorname{CV\text{-}NLL}(r)=\frac{1}{K_{\mathrm{out}}}\sum_{j=1}^{K_{\mathrm{out}}} \widehat{\mathcal L}_{\mathrm{val}}^{(j)}(r),
\qquad
\hat r_{\mathrm{CV\text{-}NLL}}=\arg\min_{r\in\mathcal R}\operatorname{CV\text{-}NLL}(r).
\]
For CV-MSE, with conditional mean predictor $\hat\mu_r^{(-j)}(\cdot)$ from Eq.~\eqref{eq:pred_mean} under $\hat\theta_r^{(-j)}$, we use
\[
\operatorname{CV\text{-}MSE}(r)=\frac{1}{K_{\mathrm{out}}}\sum_{j=1}^{K_{\mathrm{out}}}\frac{1}{|\mathcal V_j|\,q}\sum_{i\in\mathcal V_j}\left\|y_i-\hat\mu_r^{(-j)}(x_i)\right\|_2^2,
\qquad
\hat r_{\mathrm{CV\text{-}MSE}}=\arg\min_{r\in\mathcal R}\operatorname{CV\text{-}MSE}(r).
\]

\subsection{Method catalog and technical differences}\label{app:method_catalog_appendix}
Table~\ref{tab:method_catalog_merged} consolidates all compared methods along three axes: the modeling category, whether the method produces native uncertainty quantification, and the dominant computational cost together with how noise and orthogonality are handled. Three contrasts are worth drawing out. First, on uncertainty: the PPLS-family methods (EM/ECM, SLM-Interior, and our SLM-Manifold and BCD-SLM) produce native predictive uncertainty from a probabilistic model, whereas the deep nonlinear baselines obtain usable calibration only through a post-hoc recalibration stage, and the purely predictive baselines (PLSR, Ridge, DCCA/KCCA point predictors) provide none by default. Second, on noise and orthogonality: the classical and interior-point PPLS solvers estimate noise jointly and enforce orthogonality indirectly, while our solvers fix the noise spectrally and keep exact manifold feasibility at every iterate. Third, on cost: the joint and kernelized methods scale as $O((p+q)^3)$ or worse per iteration, whereas the fixed-noise solvers run at $O(rp^2+rq^2)$ through covariance projection, which is the source of their scalability on the larger benchmarks. PO2PLS is the closest joint-probabilistic point of comparison; it differs primarily in adding view-specific orthogonal components and in not exposing a native predictive posterior, which is why it appears only in point-prediction comparisons.
\begin{table*}[t]\scriptsize
\centering
\caption{Method catalog and technical differences.}
\label{tab:method_catalog_merged}
\renewcommand{\arraystretch}{1.35}
\begin{tabularx}{\textwidth}{@{}>{\raggedright\arraybackslash}p{2.2cm}>{\raggedright\arraybackslash}p{2.1cm}>{\raggedright\arraybackslash}p{2.9cm}>{\raggedright\arraybackslash}p{2.2cm}>{\raggedright\arraybackslash}X@{}}
\toprule
\textbf{Method} & \textbf{Category} & \textbf{UQ / brief description} & \textbf{Noise \& orthogonality} & \textbf{Dominant cost / comment}\\
\midrule
PLSR & Traditional & No native UQ; linear latent-regression baseline. & N/A; SVD-based. & $O(rp^2)$ baseline for point prediction.\\
Ridge & Traditional & No native UQ; strong linear baseline. & N/A. & $O(p^3)$ penalized regression benchmark.\\
EM / ECM & PPLS & Native UQ from probabilistic model. & Joint noise estimation; indirect orthogonality handling. & Classical PPLS fitting; expensive in high dimensions.\\
PO2PLS & PPLS & Joint + view-specific orthogonal latent variables; no native posterior over $Y$ given $X$. & EM-style joint noise estimation; orthogonality enforced indirectly. & $O((p+q)^3)$ per iteration; designed for decomposition rather than calibrated prediction.\\
SLM-Interior & PPLS & Native UQ from scalar likelihood. & Joint noise estimation; interior-point penalty. & Scalar objective with penalty iterations.\\
SLM-Manifold & Ours & Native UQ; full Riemannian optimization. & Spectral fixed-noise; exact manifold feasibility. & $O(rp^2+rq^2)$ exact-retraction solver.\\
BCD-SLM & Ours & Native UQ; accelerated solver on same objective. & Spectral fixed-noise; exact manifold feasibility. & $O(rp^2+rq^2)$ with componentwise closed-form updates.\\
SLM-Oracle & Diagnostic & Native UQ; synthetic-only oracle-noise diagnostic. & Oracle noise; exact manifold feasibility. & Used only to benchmark fixed-noise gap to oracle.\\
DCCA / KCCA + Ridge & Deep/nonlinear & No native calibrated UQ in default form \citep{Andrew2013,Hardoon2004,HoerlKennard1970}. & N/A. & Nonlinear representation focused on point prediction.\\
MC-Dropout / Deep Ensemble & Deep/nonlinear & Yes (post-hoc in this paper) \citep{GalGhahramani2016,Lakshminarayanan2017}. & N/A. & Require extra recalibration stage for deployment-level calibration.\\
\bottomrule
\end{tabularx}
\smallskip
\noindent\textit{Note:} PO2PLS is the closest joint-probabilistic competitor to our fixed-noise pipeline; it differs primarily in adding view-specific orthogonal components and in not providing a native predictive posterior.
\end{table*}

\subsection{PCCA specialization simulation}\label{app:pcca_sim}
\input{generated/tables/tab_pcca_simulation}

We evaluate the PCCA specialization ($B=I_r$, $\sigma_h^2=0$) at $p=q=20$, $r=3$, $N=500$ to empirically ground the benign-landscape theory of Proposition~\ref{prop:pcca-benign}. In this boundary regime, classical EM achieves slightly lower loading MSEs, while both fixed-noise solvers strictly match EM on the latent signal variance $\Sigma_t$. This behavior aligns with our theoretical expectations: absent the internal signal-noise coupling of the full PPLS model, EM's joint imputation is well matched to the PCCA sub-model. The experiment therefore serves as a sanity check, confirming that the decoupled fixed-noise solvers remain competitive even in a regime that explicitly favors classical joint updates.

\subsection{PPCA noise variance verification}
The spectral estimator coincides with the PPCA noise MLE in the single-view setting.
\input{generated/tables/tab_ppca_verification}

\subsection{Numerical confirmation of the full-spectrum bias floor}\label{app:hu-bias-mc}
Table~\ref{tab:hu_bias_mc} confirms Eq.~\eqref{eq:hu-bias} by Monte Carlo under
$p=50$, $r=5$, $\sigma_e^2=0.5$, and
$\theta_t^2=(2,1.8,1.6,1.4,1.2)$, where the theoretical bias is $8/50=0.160$.
The full-spectrum estimator bias remains near $0.160$ as $N$ grows, whereas the
noise-subspace estimator is centered near zero with standard deviation matching
the predicted $1/\sqrt{N(p-r)}$ scale.

\begin{table}[t]
\centering
\caption{Bias-floor verification for full-spectrum vs. noise-subspace estimator.}
\label{tab:hu_bias_mc}
\renewcommand{\arraystretch}{1.2}
\small
\begin{tabular}{lccc}
\toprule
Estimator / metric & $N=500$ & $N=2000$ & $N=10000$ \\
\midrule
$\mathrm{Bias}(\widehat\sigma^2_{\mathrm{full}})$ & 0.159 & 0.160 & 0.160 \\
$\mathrm{Bias}(\widehat\sigma^2_{e})$ & $-0.002$ & $-0.002$ & $-0.001$ \\
$\mathrm{SD}(\widehat\sigma^2_{e})$ & 0.0048 & 0.0024 & 0.0011 \\
\bottomrule
\end{tabular}
\end{table}

Figure~\ref{fig:noise_ablation_vary_p} extends the bias-floor analysis to a wider range of dimensions and sample sizes. Panel~(a) fixes $N\!=\!2000$ and varies $p\in\{50,100,200,500,1000\}$; panel~(b) fixes $p\!=\!200$ and varies $N\in\{200,500,1000,2000,5000,10000\}$ ($M\!=\!50$ Monte Carlo trials, $\sigma_e^2\!=\!0.5$). Error bars show Monte Carlo standard errors. In both sweeps the noise-subspace estimator (blue) achieves consistently lower absolute error than the full-spectrum estimator (orange), and the gap widens as $p$ increases because the full-spectrum bias floor scales as $O(r/p)$ (Proposition~\ref{prop:hu-bias}), whereas the noise-subspace estimator remains essentially unbiased.

\begin{figure}[t]
\centering
\includegraphics[width=0.92\textwidth]{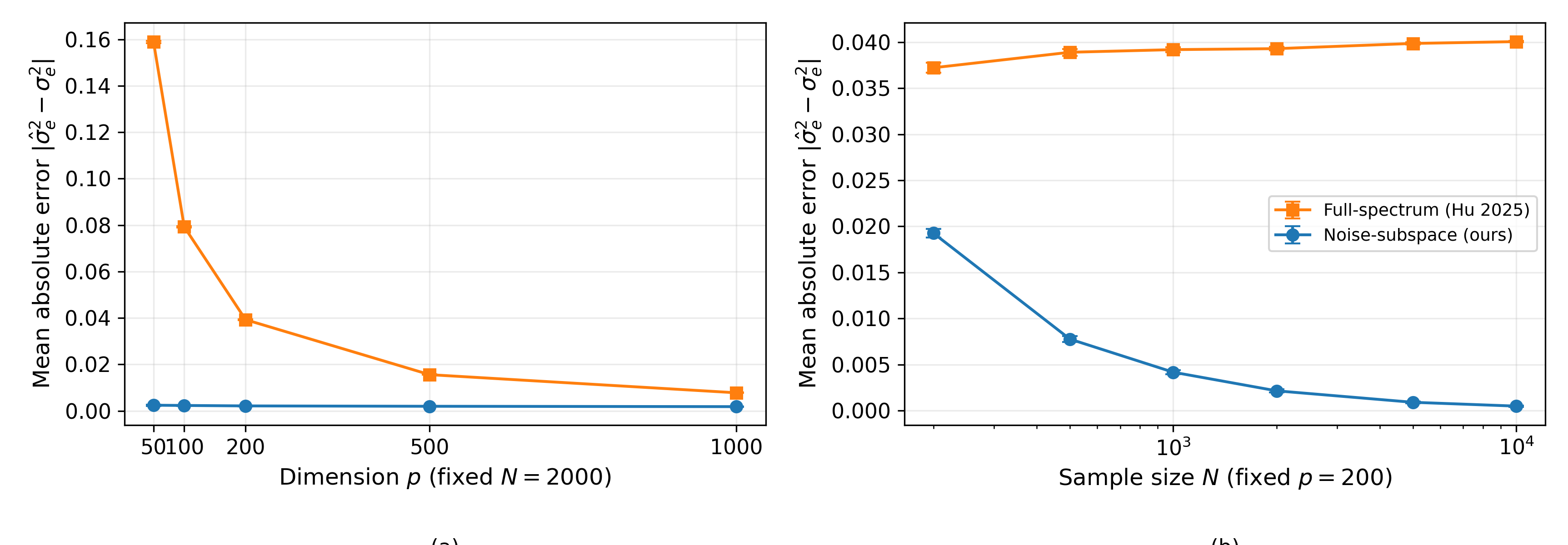}
\caption{Noise-estimation ablation under fixed $\sigma_e^2=0.5$.}
\label{fig:noise_ablation_vary_p}
\end{figure}

\subsection{Weak-factor attribution in $p=q=500$ high noise}\label{app:weak-factor-attribution}
We run the planned ablation under matched seeds on the completed high-dimensional
suite ($p=q=500$, $r=5$, $N=2000$, $M=10$). For each trial and each method,
we identify the weakest recovered factor by
$k^*=\arg\min_k \widehat{\theta}_{t,k}^2$, set
$\widehat{\theta}_{t,k^*}^2\leftarrow 0$ and $\widehat b_{k^*}\leftarrow 0$,
and then re-optimize only $\sigma_h^2$ by one-dimensional profile likelihood
with all other coordinates fixed. We then recompute
$\mathrm{MSE}_{\Sigma_t}$ and $\mathrm{MSE}_{\sigma_h^2}$ against the same ground truth.

\begin{table}[t]
\centering
\caption{Weak-factor ablation under high noise ($p=q=500$, $r=5$, $N=2000$, $M=10$): before/after and gap.}
\label{tab:weak_factor_ablation}
\renewcommand{\arraystretch}{1.2}
\small
\begin{tabular}{lccc}
\toprule
Method & $\mathrm{MSE}_{\Sigma_t}$ (before $\to$ after) & $\Delta\mathrm{MSE}_{\Sigma_t}$ & $\Delta\mathrm{MSE}_{\sigma_h^2}$ \\
\midrule
SLM-Manifold & $0.1110 \to 0.0915$ & $-0.0195\pm 0.0036$ & $+22.3899\pm 0.0138$ \\
BCD-SLM & $0.1110 \to 0.0915$ & $-0.0195\pm 0.0036$ & $+22.3902\pm 0.0139$ \\
\bottomrule
\end{tabular}
\end{table}

The same directional pattern appears for both methods: removing the weakest
factor slightly improves diagonal $\Sigma_t$ fit but causes a dramatic increase
in $\sigma_h^2$ error, indicating that the near-threshold component still carries
shared cross-view structure that would otherwise be absorbed into residual
latent noise.

\subsection{Non-Gaussian noise experiments}\label{app:non_gaussian_noise}
To empirically validate the Gaussianization framework in Section~\ref{sec:gaussianization_extension}, we evaluate the PPLS pipeline under four non-Gaussian observation-noise distributions: $t(5)$, $t(10)$, mixture Gaussian, and centered Poisson. We use the synthetic setting $p=q=200$, $r=5$, $N=2000$, high-noise $(\sigma_e^2,\sigma_f^2,\sigma_h^2)=(0.5,0.5,0.25)$, and $M=20$ trials. For each noise type, we compare three preprocessing strategies before BCD-SLM fitting: Raw (None), log-transform, and rank-based inverse normal transform (Rank-INT). Table~\ref{tab:non_gaussian_prediction} reports prediction MSE and calibration coverage at nominal levels 95\%, 90\%, and 80\%. The 'Gaussian (ref)' row uses the Gaussian benchmark from Section~\ref{sec:exp_pred_synth} under the same nominal high-noise tuple; it serves as an anchor and is not a paired control for the non-Gaussian runs.

\input{generated/tables/tab_non_gaussian_prediction}

Main observations are: (i) Raw (None) is already reasonably calibrated in most settings (typically around 93--95\% at nominal 95\%); (ii) Rank-INT tends to over-cover (e.g., up to 98.72\%), which is consistent with rank-based tail compression inflating residual-variance estimates; (iii) this error direction is conservative (wider intervals) and is often preferable to under-coverage in biomedical screening; and (iv) log-transform under-covers substantially in this benchmark. Overall, Gaussianization is optional.

\paragraph{Bias of the Gaussianization layer.} A coordinate-wise monotone map chosen to normalise marginals does not commute with the additive-noise decomposition: when the raw observation is $x_{\mathrm{raw}}=\mathrm{signal}+e$ with non-Gaussian additive noise (for instance uniform or Laplace $e$), the transformed variable $\psi(x_{\mathrm{raw}})$ is in general \emph{not} of the form $\mathrm{signal}'+\mathrm{Gaussian}$. The transform mixes signal and noise nonlinearly, so the exact isotropic-Gaussian PPLS likelihood is misspecified on the transformed scale, and the resulting estimator is, strictly speaking, \emph{biased} for the raw-scale parameters. We therefore interpret the fitted object as a Gaussian quasi-likelihood estimator \citep{White1982}: under standard regularity it remains consistent for the covariance-identified latent structure of the \emph{transformed} data and retains valid predictive uncertainty in transformed space, but it should not be read as an unbiased estimator of the original-scale loadings and noise variances. This is consistent with the empirical behaviour above: the rank-based inverse normal transform systematically over-covers (up to $98.72\%$ at nominal $95\%$), reflecting transform-induced tail compression that inflates the residual-variance estimate --- a conservative, but genuinely biased, behaviour. We accordingly recommend the raw (untransformed) fit when the marginals are approximately symmetric, and reserve Gaussianization for clearly skewed or heavy-tailed data, treating its outputs as quasi-likelihood quantities.

\subsection{CITE-seq selective prediction curves}\label{app:citeseq_selective_prediction}
Figure~\ref{fig:app_citeseq_selective_prediction} reports selective prediction curves on CITE-seq. SLM-Manifold-Adaptive shows nearly flat MSE across retention ratios, reflecting the near-homoscedastic structure of the PPLS conditional covariance: under the model, predictive variance is dominated by the global residual terms ($\sigma_f^2 I_q$ and $\sigma_h^2 I_r$) rather than input-dependent signal variance. DCCA-Deep-Ensemble exhibits a more pronounced decline, suggesting that ensemble disagreement captures input-dependent prediction difficulty that the linear model does not. Extending PPLS to heteroscedastic noise (e.g., input-dependent $\sigma_f^2(x)$) is a direction for future work.

\IfFileExists{artifacts/citeseq_protein_imputation/citeseq_selective_prediction_curve.png}{%
\begin{figure}[t]
\centering
\includegraphics[width=0.72\textwidth]{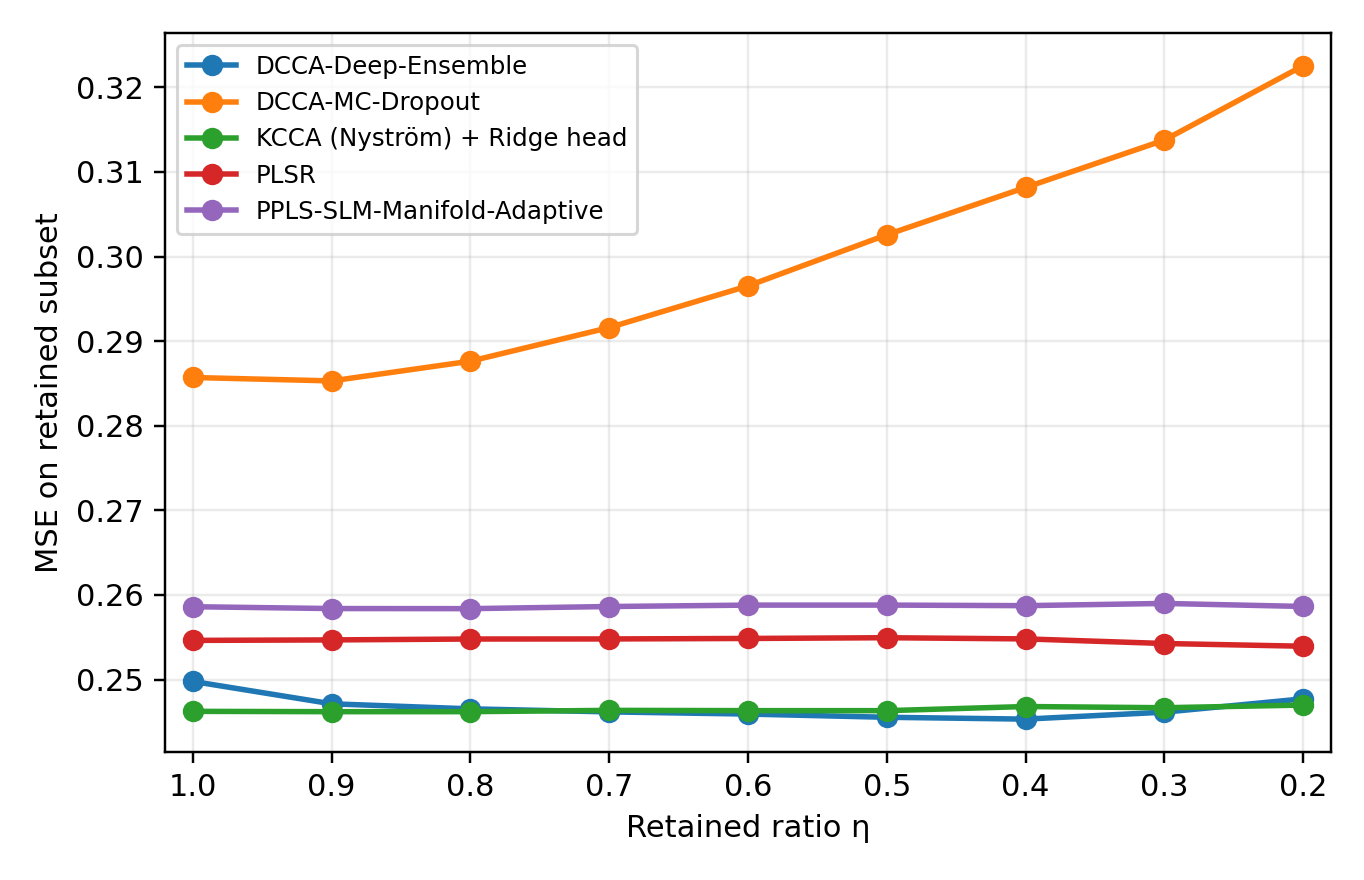}
\caption{Selective prediction on CITE-seq: retained-subset MSE versus retained ratio $\eta$ (supplementary).}
\label{fig:app_citeseq_selective_prediction}
\end{figure}
}{}

\subsection{Additional loading recovery figures}
Figures~\ref{fig:loading_C_all} and \ref{fig:loading_W} provide supplementary loading visualizations for Section~\ref{sec:exp_synthetic}, complementing the numerical MSE comparisons.

\begin{figure}[t]
\centering
\includegraphics[width=0.32\textwidth]{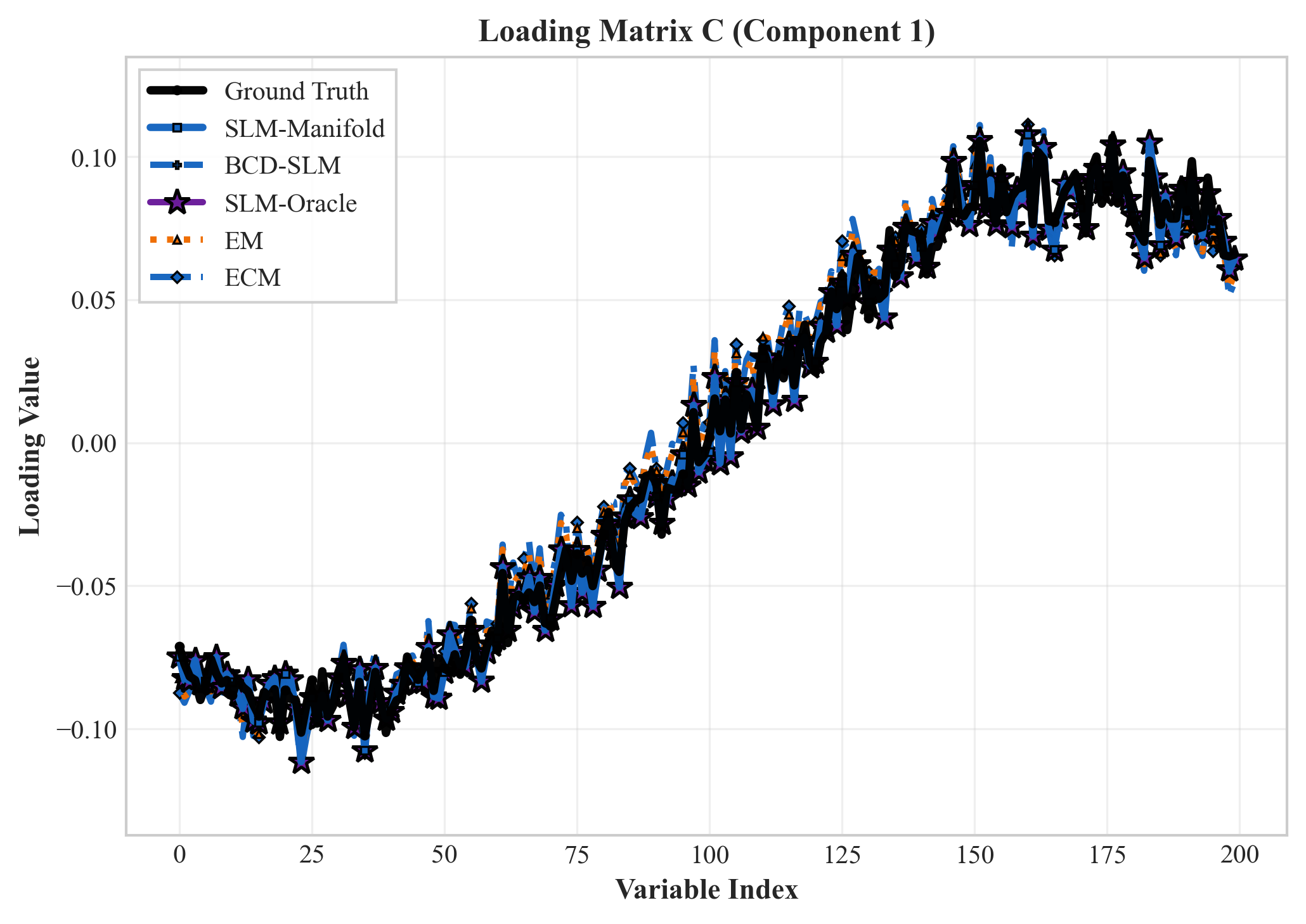}
\hfill
\includegraphics[width=0.32\textwidth]{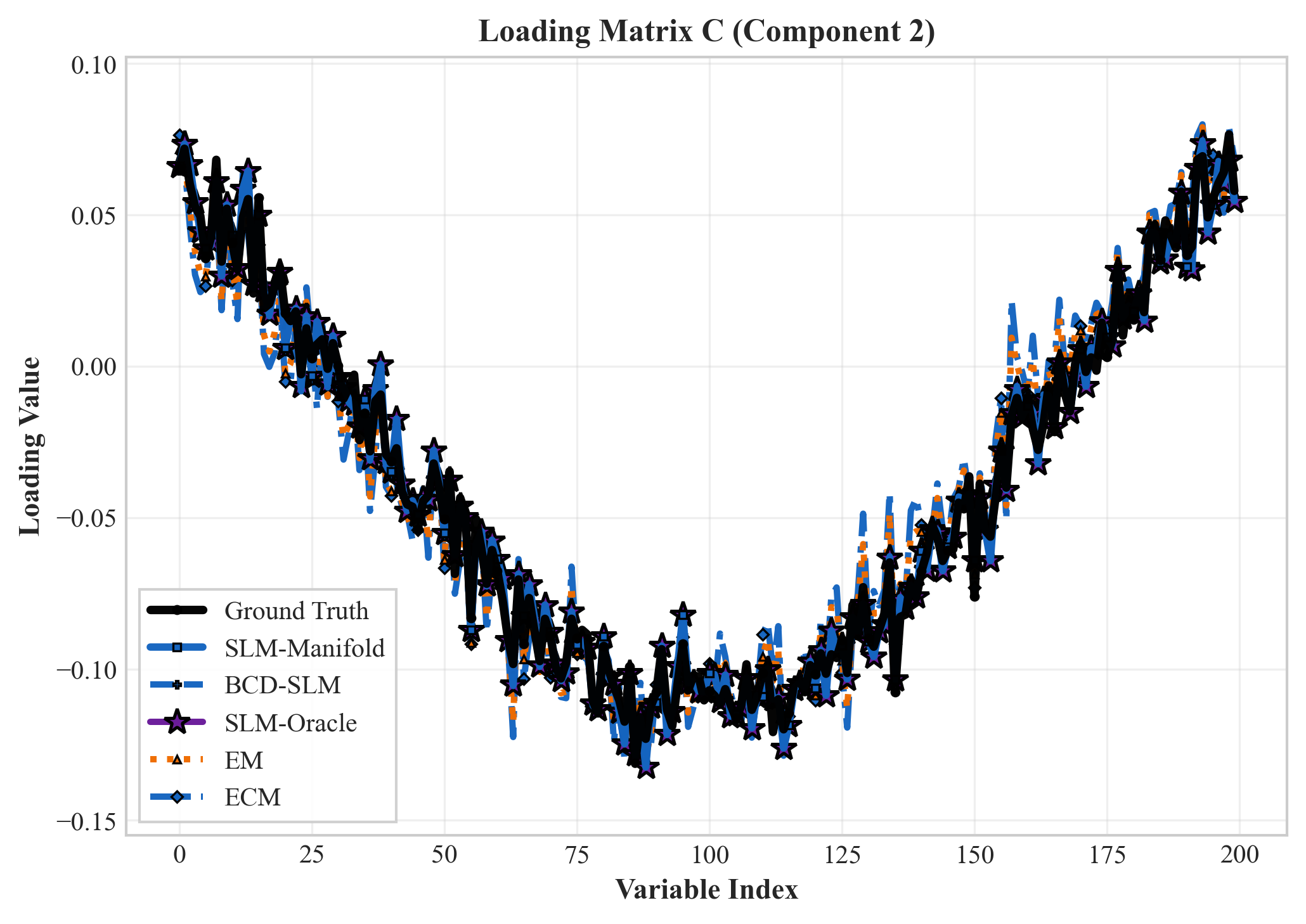}
\hfill
\includegraphics[width=0.32\textwidth]{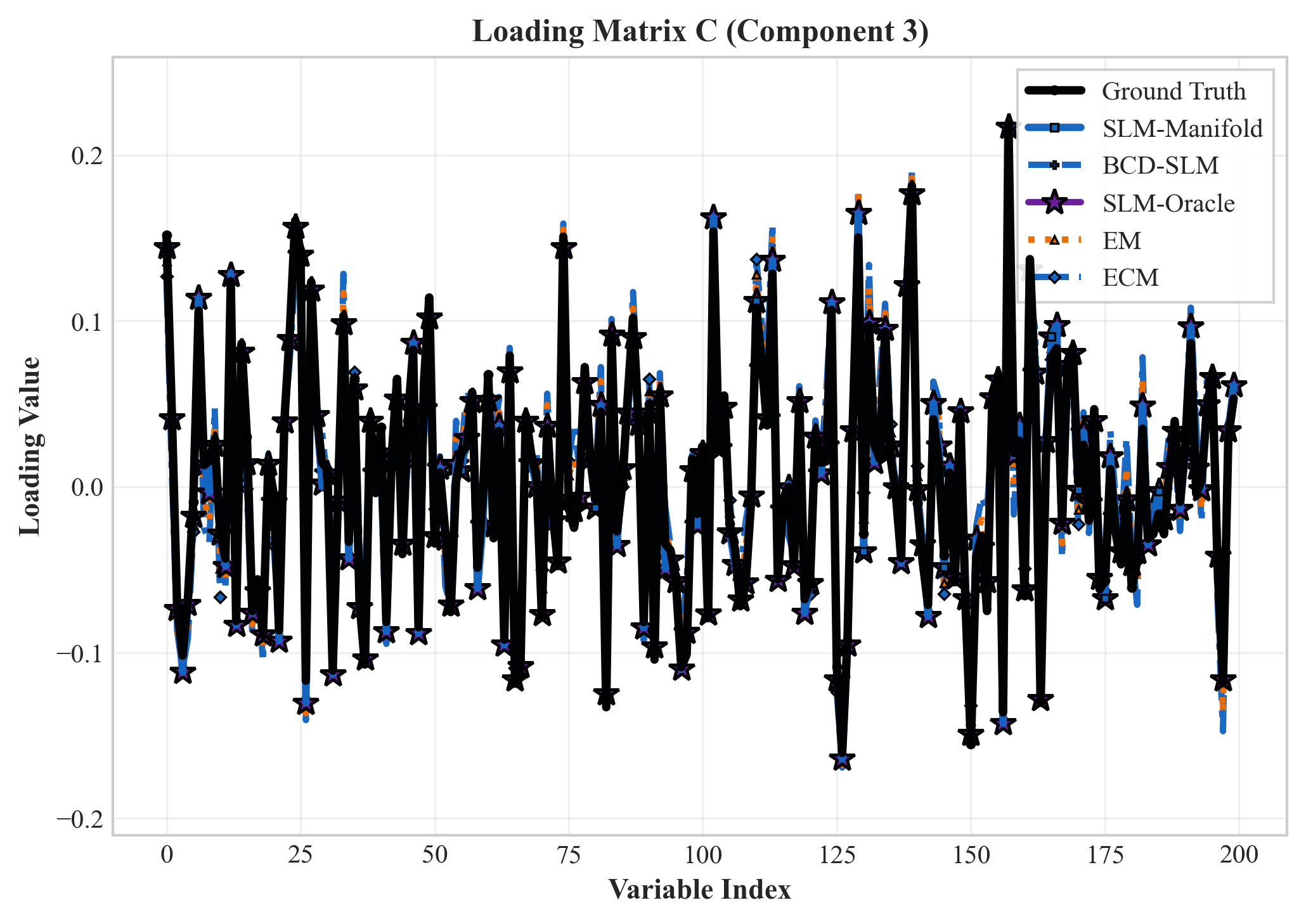}
\caption{Recovery plots for $C$ (components 1--3) in the simulation setting of Section~\ref{sec:exp_synthetic}.}
\label{fig:loading_C_all}
\end{figure}

\begin{figure}[t]
\centering
\includegraphics[width=0.32\textwidth]{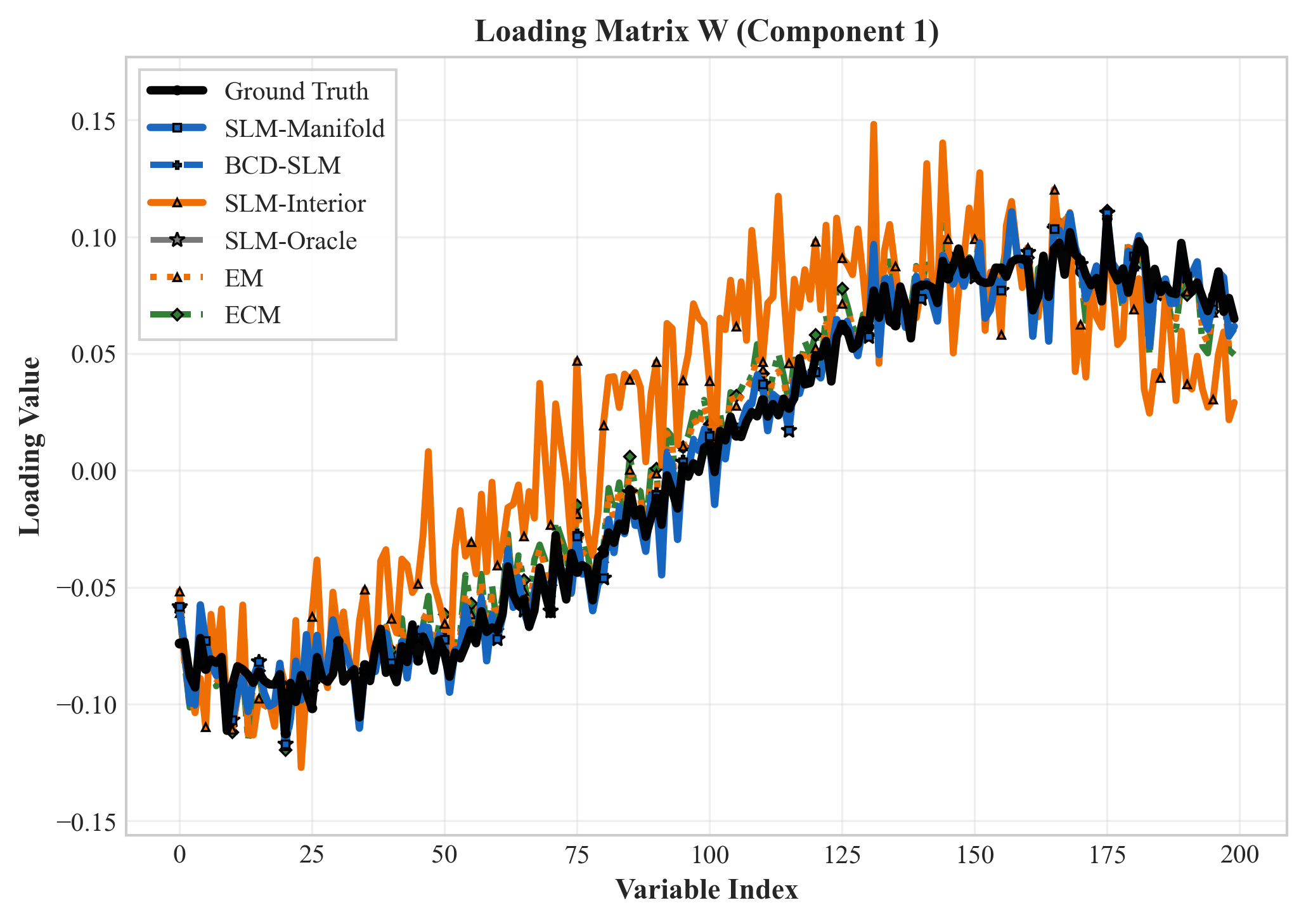}
\hfill
\includegraphics[width=0.32\textwidth]{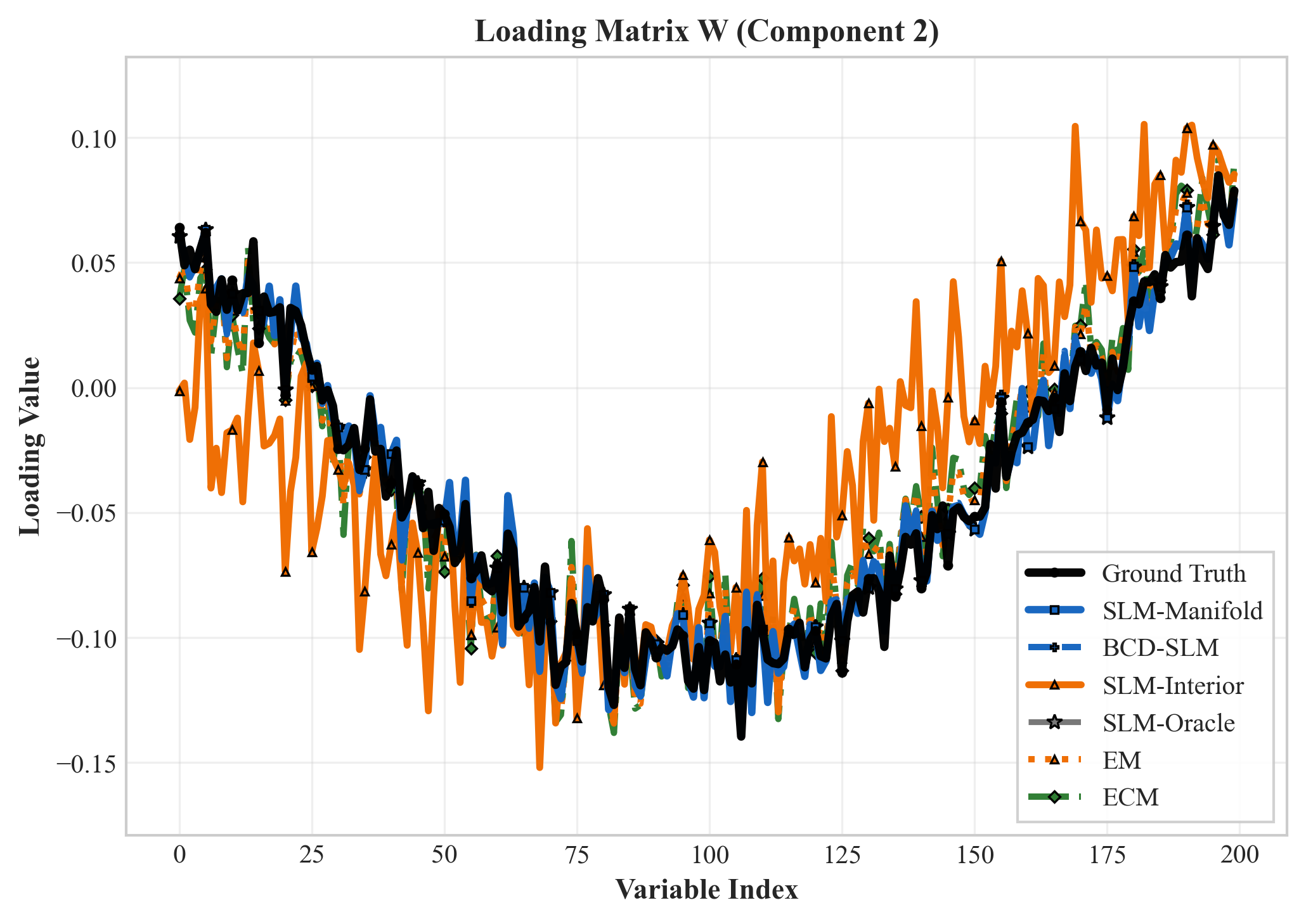}
\hfill
\includegraphics[width=0.32\textwidth]{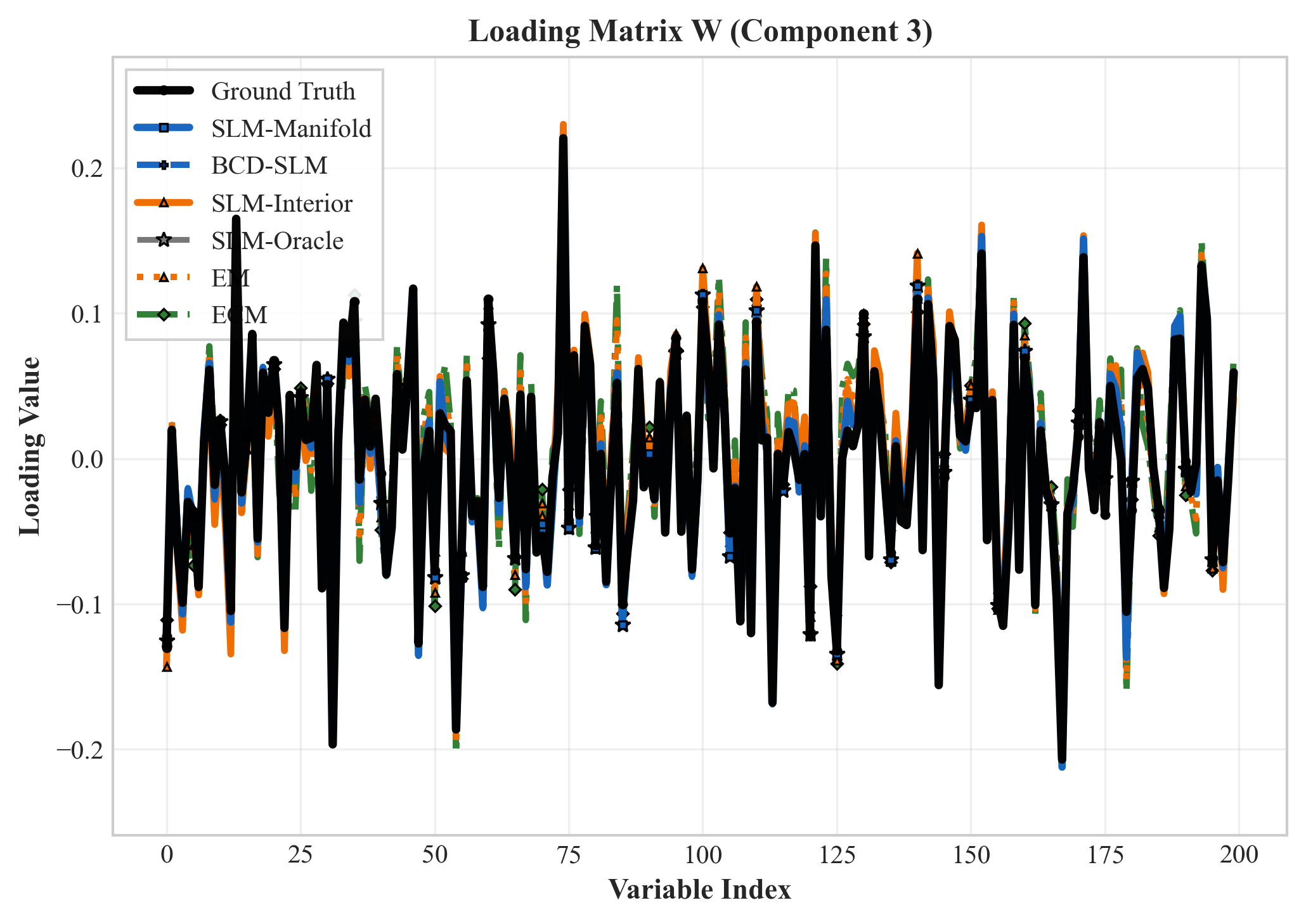}
\caption{Recovery plots for $W$ (components 1--3) in the simulation setting of Section~\ref{sec:exp_synthetic}.}
\label{fig:loading_W}
\end{figure}

\subsection{Association-screening notation}\label{app:assoc_notation}
Given centered data matrices $X\in\mathbb{R}^{N\times p}$ and $Y\in\mathbb{R}^{N\times q}$ and fitted loadings $(\hat W,\hat C)$, define
\(\hat T:=X\hat W\), \(\hat U:=Y\hat C\),
\(\hat\rho^{(x)}_{j,i}:=\operatorname{corr}(X_{:,j},\hat T_{:,i})\), and
\(\hat\rho^{(y)}_{\ell,i}:=\operatorname{corr}(Y_{:,\ell},\hat U_{:,i})\).
Each $p$-value $p(\hat\rho)$ is the two-sided $p$-value of the standard Pearson correlation test based on the corresponding correlation estimate under a normal approximation.
Detected pairs at threshold $\alpha$ are
\[
\mathcal{P}_i(\alpha):=\{(j,\ell): p(\hat\rho^{(x)}_{j,i})<\alpha\ \wedge\ p(\hat\rho^{(y)}_{\ell,i})<\alpha\},
\quad
\mathcal{P}(\alpha):=\cup_{i=1}^r\mathcal{P}_i(\alpha).
\]

\subsection{Broader impacts and responsible use}\label{app:broader_impacts}
The method improves calibrated uncertainty and latent-factor recovery for multi-view analysis, which can support more reliable downstream screening in exploratory biomedical studies. However, calibrated intervals from observational data do not imply causal validity, and model misspecification can still produce overconfident conclusions in subgroup analyses. We therefore recommend external cohort validation and domain-expert review before any high-stakes clinical use.


\section{Implementation and Reproducibility Details}\label{app:scalar_details}
This appendix records the implementation and reproducibility details behind the reported experiments, covering the Riemannian optimization settings, the initialization and warm-start procedure, and the licenses of the datasets and software used.

\subsection{Implementation details of Riemannian optimization}\label{app:implementation_details}
\paragraph{QR retraction with sign consistency.}
SLM-Manifold uses Riemannian conjugate gradient with Armijo line search and sign-consistent thin-QR retraction.

\paragraph{Identifiability-preserving ordering.}
After convergence of each start, components are reordered so that $\theta_{t,\pi(1)}^2 b_{\pi(1)}\ge\cdots\ge\theta_{t,\pi(r)}^2 b_{\pi(r)}$.

\paragraph{Multi-start and warm start.}
Each start samples $W,C$ via random QR and initializes $\Sigma_t=B=I_r$, $\sigma_h^2=0.01$. One optional warm start uses top singular vectors of training cross-covariance.

\paragraph{Adaptive shrinkage grid.}
The nested-CV candidate sets are fixed before evaluation:
\[
\Gamma_{\mathrm{synth}}=\{0.5,0.7,0.85,0.95,1.0,1.05,1.15,1.3,1.5,2.0\},
\]
\[
\Gamma_{\mathrm{BRCA}}=\{0.3,0.5,0.7,0.85,1.0,1.15,1.3,1.5,2.0,3.0\},
\]
\[
\Gamma_{\mathrm{CITE}}=\{0.1,0.2,0.3,0.5,0.7,1.0,1.5,2.0,3.0,5.0,10.0\}.
\]

\subsection{Initialization details}\label{app:init_algo}

Default initialization uses feasible random-QR starts for $(W,C)$, diagonal defaults for $(\Sigma_t,B,\sigma_h^2)$, and spectral pre-estimates for $(\sigma_e^2,\sigma_f^2)$. For one optional warm start, let $S_{xy}^{\mathrm{tr}} = n_{\mathrm{tr}}^{-1}X_{\mathrm{tr}}^\top Y_{\mathrm{tr}} = U\Lambda V^\top$, set $W^{(0)}=U_{[:,1:r]}$, $C^{(0)}=V_{[:,1:r]}$, align component signs via cross-covariance of projected scores, initialize diagonal parameters from projected moments, and enforce the identifiability ordering.

\subsection{Licenses for existing assets}\label{app:asset_licenses}
\paragraph{Datasets.}
\begin{itemize}
    \item TCGA-BRCA Multi-Omics \citep{brca}: Kaggle distribution under CC BY-NC-SA 4.0.
    \item PBMC CITE-seq reference \citep{Hao2021}: public SeuratData/10x resources for academic research.
\end{itemize}

\paragraph{Software.}
\begin{itemize}
    \item Pymanopt \citep{Pymanopt}: BSD 3-Clause.
    \item Manopt \citep{BoumalManopt}: GPL v3 (reference only).
    \item NumPy/SciPy/scikit-learn: BSD-style licenses.
\end{itemize}

%% file: generated/tables/tab_algorithm_convergence.tex
\begin{table*}[t]
\centering
\setlength{\tabcolsep}{6pt}
\renewcommand{\arraystretch}{1.25}
\footnotesize
\caption{Convergence statistics across $M=20$ Monte Carlo trials in the canonical low-noise setting ($p=q=200$, $r=5$, $N=2000$).}
\label{tab:algorithm_convergence}
\begin{tabular}{lcccccc}
\toprule
\textbf{Method} & $\mathbb{E}[I]$ & $\sqrt{\text{Var}[I]}$ & $\min I$ & $\max I$ & $\text{Median}[I]$ & Success \\
\midrule
SLM-Manifold  & 243.9  & 17.4  & 209  & 277  & 246  & 100\%  \\
BCD-SLM  & 45.7  & 10.5  & 23  & 67  & 45  & 100\%  \\
SLM-Interior  & 77.6  & 15.9  & 63  & 135  & 73  & 100\%  \\
SLM-Oracle  & 252.0  & 15.0  & 216  & 277  & 254  & 100\%  \\
EM  & 99.7  & 15.6  & 80  & 140  & 96  & 100\%  \\
ECM  & 93.8  & 22.8  & 60  & 148  & 91  & 100\%  \\
\bottomrule
\end{tabular}
\end{table*}

%% file: generated/tables/tab_prediction_synth_metrics.tex
\begin{table*}[t]
\centering
\setlength{\tabcolsep}{6pt}
\renewcommand{\arraystretch}{1.25}
\small
\caption{Synthetic point-prediction accuracy (5-fold CV); each metric is reported as separate mean and standard-deviation (SD) columns over folds.}
\label{tab:pred_synth_metrics}
\begin{tabular}{lcccccc}
\toprule
& \multicolumn{2}{c}{\textbf{MSE}} & \multicolumn{2}{c}{\textbf{MAE}} & \multicolumn{2}{c}{\textbf{$R^2$}} \\
\cmidrule(lr){2-3}\cmidrule(lr){4-5}\cmidrule(lr){6-7}
\textbf{Method} & Mean & SD & Mean & SD & Mean & SD \\
\midrule
SLM-Manifold-Adaptive & 0.1058 & 0.0008 & 0.2593 & 0.0011 & 0.2749 & 0.0107 \\ 
EM & 0.1058 & 0.0008 & 0.2593 & 0.0011 & 0.2749 & 0.0106 \\ 
PLSR & 0.1058 & 0.0008 & 0.2593 & 0.0011 & 0.2748 & 0.0105 \\ 
Ridge & 0.1079 & 0.0008 & 0.2620 & 0.0011 & 0.2617 & 0.0106 \\ 
\bottomrule
\end{tabular}
\end{table*}

%% file: generated/tables/tab_pcca_simulation.tex
\begin{table*}[t]
\centering
\setlength{\tabcolsep}{6pt}
\renewcommand{\arraystretch}{1.25}
\small
\caption{PCCA specialization parameter estimation MSE ($\times 10^{2}$), with mean and standard deviation (SD) reported in separate columns over $M=20$ trials ($p=q=20$, $r=3$, $N=500$; $B=I_r$, $\sigma_h^2=0$).}
\label{tab:pcca_parameter_mse}
\begin{tabular}{lcccccc}
\toprule
& \multicolumn{2}{c}{$\mathrm{MSE}_W$} & \multicolumn{2}{c}{$\mathrm{MSE}_C$} & \multicolumn{2}{c}{$\mathrm{MSE}_{\Sigma_t}$} \\
\cmidrule(lr){2-3}\cmidrule(lr){4-5}\cmidrule(lr){6-7}
\textbf{Method} & Mean & SD & Mean & SD & Mean & SD \\
\midrule
BCD-SLM & 0.08 & 0.09 & 0.08 & 0.08 & 0.70 & 0.52 \\
EM & 0.04 & 0.04 & 0.04 & 0.04 & 0.70 & 0.55 \\
\bottomrule
\end{tabular}
\end{table*}

%% file: generated/tables/tab_ppca_verification.tex
\begin{table*}[t]
\centering
\setlength{\tabcolsep}{6pt}
\renewcommand{\arraystretch}{1.25}
\small
\caption{PPCA noise variance estimation verification ($M=20$ trials, $p=20$, $r=3$, $N=500$, $\sigma_e^2=0.100000$).}
\label{tab:ppca_noise_verification}
\begin{tabular}{lcc}
\toprule
Estimator & Mean $\hat{\sigma}_e^2$ & Mean $\lvert\hat{\sigma}_e^2-\sigma_e^2\rvert$ \\
\midrule
Spectral (Theorem 4) & 0.098801 & 0.001450 \\
Tipping \& Bishop MLE & 0.098801 & 0.001450 \\
\bottomrule
\end{tabular}
\end{table*}

%% file: generated/tables/tab_non_gaussian_prediction.tex
\begin{table*}[t]
\centering
\setlength{\tabcolsep}{4pt}
\renewcommand{\arraystretch}{1.25}
\scriptsize
\caption{Prediction MSE and calibration coverage under non-Gaussian noise ($p=q=200$, $r=5$, $N=2000$, high-noise setting, $M=20$ trials). The final Gaussian (ref) row is taken from the separate Gaussian synthetic prediction benchmark and is reported as a reference anchor.}
\label{tab:non_gaussian_prediction}
\begin{tabular*}{\textwidth}{@{\extracolsep{\fill}}llcccccccc@{}}
\toprule
& & \multicolumn{2}{c}{\textbf{Pred MSE}} & \multicolumn{2}{c}{\textbf{95\% Cov}} & \multicolumn{2}{c}{\textbf{90\% Cov}} & \multicolumn{2}{c}{\textbf{80\% Cov}} \\
\cmidrule(lr){3-4}\cmidrule(lr){5-6}\cmidrule(lr){7-8}\cmidrule(lr){9-10}
\textbf{Noise Type} & \textbf{Preprocess} & Mean & SD & Mean & SD & Mean & SD & Mean & SD \\
\midrule
$t(5)$ & None & 0.5239 & 0.00215 & 93.74\% & 0.03\% & 89.76\% & 0.04\% & 81.79\% & 0.07\% \\
$t(5)$ & Log & 0.5262 & 0.002106 & 81.43\% & 0.04\% & 74.44\% & 0.05\% & 63.51\% & 0.06\% \\
$t(5)$ & Rank-INT & 0.5275 & 0.002273 & 97.71\% & 0.02\% & 95.77\% & 0.03\% & 91.00\% & 0.05\% \\
$t(10)$ & None & 0.5238 & 0.001614 & 93.44\% & 0.03\% & 88.62\% & 0.05\% & 79.26\% & 0.06\% \\
$t(10)$ & Log & 0.5275 & 0.001691 & 79.98\% & 0.04\% & 72.30\% & 0.05\% & 60.79\% & 0.07\% \\
$t(10)$ & Rank-INT & 0.5273 & 0.001592 & 98.15\% & 0.03\% & 95.94\% & 0.04\% & 90.20\% & 0.06\% \\
Mixture & None & 0.5244 & 0.002577 & 95.01\% & 0.02\% & 93.79\% & 0.02\% & 90.38\% & 0.04\% \\
Mixture & Log & 0.5222 & 0.002601 & 87.90\% & 0.03\% & 82.59\% & 0.05\% & 72.70\% & 0.05\% \\
Mixture & Rank-INT & 0.5368 & 0.002412 & 96.38\% & 0.02\% & 95.51\% & 0.03\% & 93.83\% & 0.04\% \\
Poisson & None & 0.5241 & 0.001076 & 93.74\% & 0.02\% & 88.13\% & 0.03\% & 77.45\% & 0.04\% \\
Poisson & Log & 0.5285 & 0.001104 & 78.97\% & 0.03\% & 70.74\% & 0.05\% & 58.59\% & 0.07\% \\
Poisson & Rank-INT & 0.5275 & 0.001146 & 98.72\% & 0.02\% & 96.51\% & 0.03\% & 90.00\% & 0.06\% \\
\midrule
Gaussian (ref) & None & 0.1058 & 0.0007973 & 95.10\% & 0.12\% & 90.46\% & 0.19\% & 81.05\% & 0.24\% \\
\bottomrule
\end{tabular*}
\end{table*}